\documentclass{article}

\usepackage{microtype}
\usepackage{graphicx}
\usepackage{wrapfig}
\usepackage{subcaption}
\usepackage{booktabs} % for professional tables

\usepackage{hyperref}
\usepackage{titletoc} 

\usepackage[accepted]{icml2026}

\usepackage{amsmath}
\usepackage{amssymb}
\usepackage{mathtools}
\usepackage{amsthm}

\usepackage{textcomp, gensymb}
\usepackage{multirow}
\usepackage{colortbl}
\usepackage[table]{xcolor}
\usepackage{overpic}
\usepackage{soul}
\usepackage{rotating}
\usepackage[capitalize,noabbrev]{cleveref}

\crefname{section}{Sec.}{Secs.}
\Crefname{section}{Section}{Sections}
\Crefname{table}{Table}{Tables}
\crefname{table}{Tab.}{Tabs.}
\Crefname{figure}{Figure}{Figures}
\crefname{figure}{Fig.}{Figs.}

\usepackage{tikz} % 添加TikZ支持
\usetikzlibrary{positioning} % 用于精确定位
\usetikzlibrary{calc} % 用于计算相对位置

\usepackage{afterpage}
\theoremstyle{plain}
\newtheorem{theorem}{Theorem}[section]
\newtheorem{proposition}[theorem]{Proposition}

\theoremstyle{definition}

\theoremstyle{remark}

\usepackage[textsize=tiny]{todonotes}
\usepackage{dirtree}

\icmltitlerunning{ Efficient Multi-Scale Transformer for Accumulative Context Weather Forecasting }

\begin{document}

\twocolumn[
  \icmltitle{  EMFormer: Efficient Multi-Scale Transformer for\\Accumulative Context Weather Forecasting   }



  \icmlsetsymbol{equal}{*}

  \begin{icmlauthorlist}
    % \icmlauthor{Hao Chen}{equal,yyy}
    % \icmlauthor{Tao Han}{equal,yyy,comp}
    \icmlauthor{Hao Chen}{yyy}
    \icmlauthor{Tao Han}{yyy}
    \icmlauthor{Jie Zhang}{yyy}
    \icmlauthor{Song Guo}{yyy}
    \icmlauthor{Fenghua Ling}{comp}
    \icmlauthor{Lei Bai}{comp}
  \end{icmlauthorlist}

  \icmlaffiliation{yyy}{Department of Computer Science and Engineering, Hong Kong University of Science and Technology, Hong Kong, China}
  \icmlaffiliation{comp}{Shanghai AI Laboratory, Shanghai, China}

  \icmlcorrespondingauthor{Song Guo}{songguo@ust.hk}
  \icmlcorrespondingauthor{Jie Zhang}{csejzhang@ust.hk}

  \icmlkeywords{Weather Forecasting, Efficient Transformer, Typhoon Track Prediction}

  \vskip 0.3in
]



\printAffiliationsAndNotice{}  % no special notice (required even if empty)

\begin{abstract}

Long-term weather forecasting is critical for socioeconomic planning and disaster preparedness.
While recent approaches employ finetuning to extend prediction horizons, they remain constrained by the issues of catastrophic forgetting, error accumulation, and high training overhead.
To address these limitations, we present a novel pipeline across pretraining, finetuning and forecasting to enhance long‑context modeling while reducing computational overhead. First, we introduce an Efficient Multi‑scale Transformer (EMFormer) to extract multi‑scale features through a single convolution in both training and inference. Based on the new architecture, we further employ an accumulative context finetuning to improve temporal consistency without degrading short‑term accuracy. Additionally, we propose a composite loss that dynamically balances different terms via a sinusoidal weighting, thereby adaptively guiding the optimization trajectory throughout pretraining and finetuning.
Experiments show that our approach achieves great performance in weather forecasting and extreme event prediction, substantially improving long-term forecast accuracy. Moreover, EMFormer demonstrates strong generalization on vision benchmarks (ImageNet-1K and ADE20K). Code: \url{https://github.com/chenhao-zju/emformer}

% while delivering a 5.69$\times$ speedup over conventional multi-scale modules
  
\end{abstract}

\section{Introduction}
\label{sec:intro}

\textbf{Why do we need data-driven methods in weather forecasting?} 
Long‑term weather forecasting is a critical challenge with significant socioeconomic implications, affecting sectors such as aviation, maritime navigation, and finance. Traditional Numerical Weather Prediction (NWP) generates forecasts by solving partial differential equations~\cite{nwp, lynch2008origins, kalnay2002atmospheric}, but it suffers from cumulative errors and high computation. In contrast, data‑driven models~\cite{fourcastnet, panguweather, fengwu, aurora} learn atmospheric patterns directly from historical observations. By producing forecasts from learned representations rather than iterative physical integration, these approaches minimize error propagation and achieve greater computational efficiency.

\begin{figure}[t] % 使用figure*环境确保图片出现在页面顶部
    \centering
    \begin{minipage}{\columnwidth} % 限制为单栏宽度
        \centering

        \begin{subfigure}{1.0\textwidth}
            \centering
            \captionsetup{justification=centering}
            \caption{Z500 Denormalized RMSE $\downarrow$ (6-hour prediction) }
            \vspace{-0.1cm} % 行间距
            \includegraphics[width=1.0\linewidth]{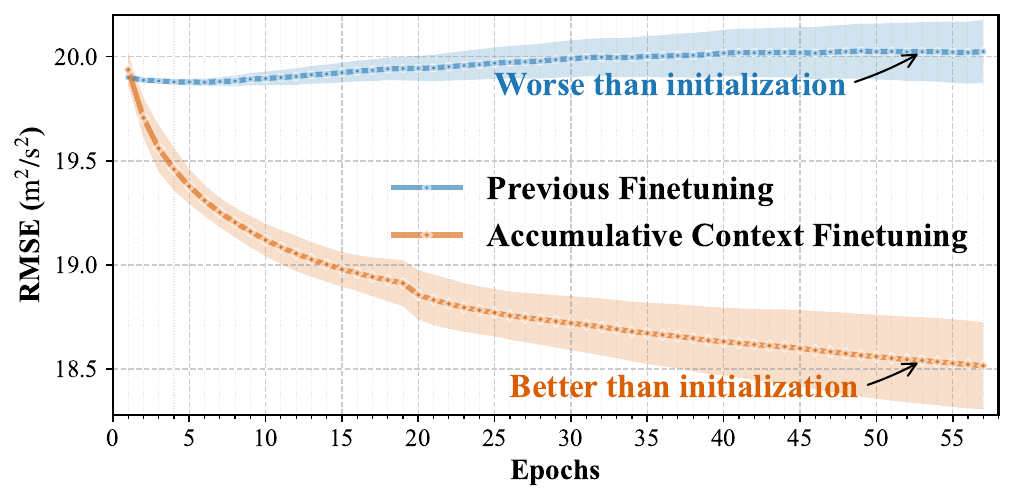}
            \label{fig:motiv_a}
        \end{subfigure}%
        % \hspace{0.3\textwidth} % 居中占位
        % \hfill
        \vspace{-0.6cm} % 行间距
        \begin{subfigure}{1.0\textwidth}
            \centering
            \captionsetup{justification=centering}
            \caption{Z500 Denormalized RMSE $\downarrow$ (5-day prediction)}
            \vspace{-0.1cm} % 行间距
            \includegraphics[width=1.0\linewidth]{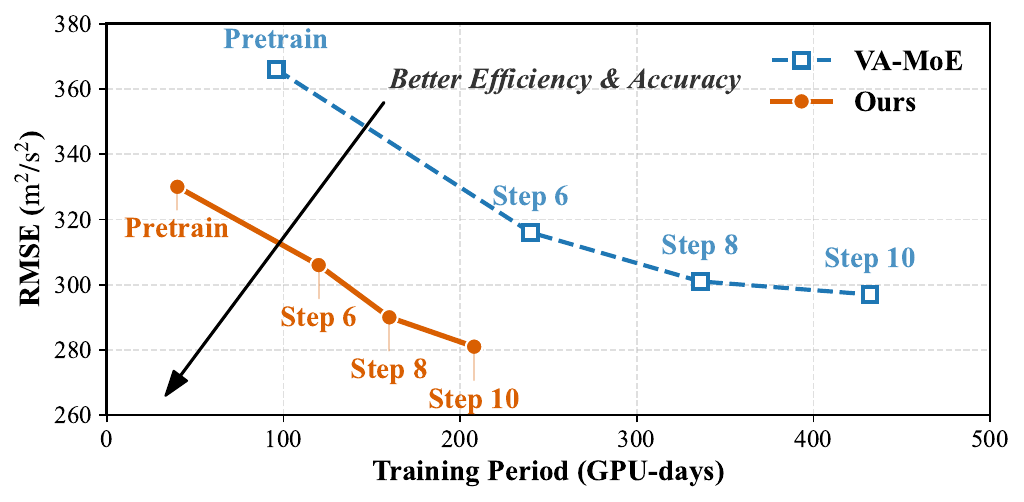}
            \label{fig:motiv_b}
            % \vspace{-0.6cm} % 行间距
        \end{subfigure}
        \vspace{-1cm}
        % 整体大标题
        \caption{ Denormalized Z500 RMSE ($m^2/s^2$) for short-term (6-hour) and medium-term (5-day) forecasts. (a) Training convergence comparison in 6-step finetuning: \textcolor[HTML]{d95f02}{accumulative context finetuning} and \textcolor[HTML]{1f77b4}{previous finetuning}. (b) Medium-term forecast performance: The \textcolor[HTML]{d95f02}{proposed method} consistently outperforms \textcolor[HTML]{1f77b4}{VA-MoE} across pretraining and multi-step finetuning (6, 8, 10 steps). Two models are trained by us with A100.
        }
        \vspace{-4mm}
        \label{fig:motivation}
    \end{minipage}
\end{figure}

% motivation:

% \textbf{Why do we need efficient structure in finetuning?} 
% Although AI-driven methods performs better than NWP by learning historical patterns, but still suffering from step-wise error accumulation in auto-regression. To improve the accuracy in long-term weather forecasting, many AI-based methods adopt the long-term finetuning strategy. Although this mechanism can effectively increase the length of reasoning, it still faces two problems. 1) As the length of inference deepens, the model will inevitably forget several early stages. 2) The longer the time step of finetuning, the longer the training time required. To tackle the first issue, this work introduces kv cache finetuning to improve the temporal consistency and maintain the accuracy of earliest step forecasts. But the kv-cache finetuning will further increase the training burden. Thus, it is crucial to design an efficient model to reduce the training burden and further lower the burden during the finetuning stage.

\textbf{Why is an Efficient Architecture Essential for Finetuning?} While data-driven methods outperform NWP by capturing atmospheric dynamics, they remain prone to stepwise error accumulation during auto-regressive forecasting. To enhance long-term accuracy, many approaches employ extended finetuning on multi-step sequences. Although this improves forecast horizons, it introduces two critical limitations: (1) as inference length increases, the model gradually forgets information from earlier steps (\textcolor[HTML]{1f77b4}{blue line} in \cref{fig:motiv_a}), and (2) longer finetuning sequences demand greater training time and computational resources (\textcolor[HTML]{1f77b4}{blue line} in \cref{fig:motiv_b}). To address the first issue, we introduce accumulative context finetuning,  which explicitly preserves historical information to ensure temporal consistency.
Yet, this technique inevitably exacerbates the computational burden. Consequently, it is crucial to design an efficient framework that reduces computational cost during training and finetuning.

To address the efficiency and stability bottlenecks in long-context forecasting, we propose \textbf{a novel pipeline} comprising pretraining, finetuning, and forecasting. During pretraining, we introduce a hierarchical pruning-recovering framework coupled with an Efficient Multi‑scale Transformer (EMFormer) to lower computational cost. For finetuning, an accumulative context mechanism is employed to strengthen long‑horizon representations, ensuring temporal consistency across multi‑step predictions. To further refine optimization, we design a variable‑ and geography‑aware loss that adapts to the inherent heterogeneity of atmospheric variables across both physical properties and geographical regions.

We make three core contributions: (i) \textbf{Efficient Multi-Scale Architecture via Hard-aware Design.} 
While multi-scale transformers are effective, their training costs are often prohibitive. 
% \textbf{Efficient multi-convolution layer via custom CUDA kernels.} 
% While multi-scale transformers and re-parameterization techniques are well-established, their optimization during training remains computationally expensive. 
To address this, we propose EMFormer, which integrates a novel multi-convolution (multi-convs) layer optimized with custom CUDA kernels. Unlike standard re-parameterization methods that only accelerate inference, our method enables multi-scale feature capture via a single convolution during both training and inference.
% fuse branches for inference, we redesign the underlying CUDA kernel to enable multi-scale feature capture via a single convolution during training. 
This redesign preserves representational power while accelerating forward and backward passes by 5.69$\times$ compared to traditional multi-scale implementations. (ii) \textbf{Accumulative Finetuning for Long-context Consistency.} We introduce a specialized finetuning strategy tailored for long-context weather forecasting. 
By injecting historical Key–Value (KV) pairs into current generation steps and employing a memory-pruning mechanism, we explicitly bound memory usage while strengthening long-term temporal dependencies.
This ensures sustained accuracy over extended horizons without compromising short-term performance. 
% This method (a) injects historical Key–Value (KV) pairs into current generation steps to enhance long-term temporal dependencies, and (b) employs a memory module to prune the KV pairs, explicitly bounding memory usage. 
% This design improves long-range forecast accuracy without compromising short-term performance or computational efficiency. 
(iii) \textbf{Sinusoidal Weighted Optimization Objective.} We design a composite objective function featuring a sinusoidal weighting mechanism to address the heterogeneity of atmospheric data. This includes a latitude-adaptive term, which accounts for spatial distortion and a variable-adaptive term that balances learning dynamics across distinct physical variables, ensuring robust optimization across diverse atmospheric conditions.

In addition to theoretical analysis, experiments are presented in \cref{fig:motivation}. \cref{fig:motiv_a} plots the denormalized Z500 RMSE (\(m^2/s^2\)) curves for 6‑hour forecasts under two finetuning strategies. While conventional finetuning leads to progressively deteriorating first‑step RMSE, the proposed accumulative finetuning steadily reduces the error, demonstrating superior stability in short‑term forecasts. \cref{fig:motiv_b} compares the 5‑day predictions between VA‑MoE~\cite{vamoe} and Ours with distinct finetuning steps. Our approach achieves lower RMSE with shorter finetuning period: at Step 10, it reaches about 280 \(m^2/s^2\) after 210 GPU‑days, whereas VA‑MoE requires 430 GPU‑days to attain 295 \(m^2/s^2\). Although both methods improve with extended finetuning, our method outperforms VA‑MoE with fewer GPU‑days. Beyond atmosphere, EMFormer also delivers competitive performance on vision tasks such as classification and segmentation, surpassing existing methods.

% conclusion and main contributions:

% \input{sections/relatedwork}
\section{Methodology} 
\label{sec:method}

\begin{figure*}[t]
  \centering
   \includegraphics[width=0.95\linewidth]{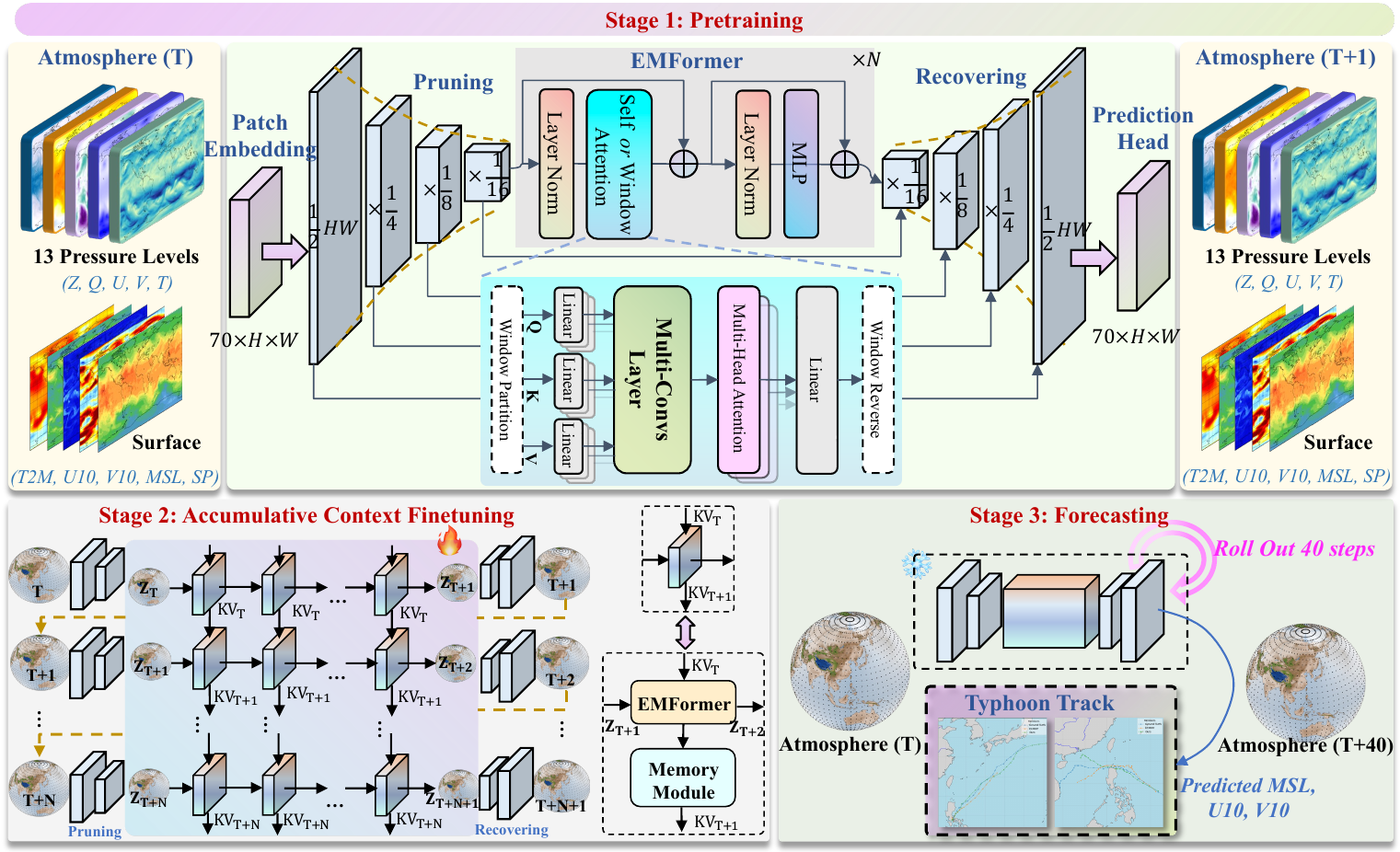}
   % \vspace{-0.3cm}
   \caption{Illustration of the novel pipeline with three stages. \textcolor{red}{Stage 1:} EMFormer is pretrained on atmospheric variables with pruning-recovering architecture that includes a pruning module, a series of EMFormer blocks, and a recovering module; \textcolor{red}{Stage 2:} accumulative context finetuning; \textcolor{red}{Stage 3:} The forecasting stage with weather forecasting and typhoon track prediction.
   }
   \label{fig:framework}
   \vspace{-11pt}
\end{figure*}

This section presents a novel pipeline for weather forecasting, comprising: (1) single‑step pretraining, (2) multi‑step finetuning, and (3) multi‑step forecasting. The forecasting task is defined and the framework is outlined in \cref{subsec:problem}. The core contribution, EMFormer, which incorporates a multi‑convolution (multi‑convs) layer, is introduced in \cref{subsec:emtransformer}. \cref{subsec:kv} then describes the accumulative context finetuning with a memory module for cache management. Finally, \cref{subsec:loss} details a sinusoidal‑weighted loss that combines variable‑adaptive and latitude‑weighted terms. 

In addition, two propositions concerning the multi‑convs layer (\cref{thm:multiconvs}) and loss function (\cref{thm:loss}) are formulated. Proofs of the propositions and supporting experiments are provided in \cref{sec:proof} and \cref{subsec:add_proofresult}, respectively. 

% The pseudo-codes of multi-convs layer and accumulative finetuning are provided in \cref{alg:multi_convs_layer}.

% In this section, we first formally define the problem in \cref{subsec:problem}. Then, we present the architecture of Variable-Aware Transformer with Channel=Adapted Expert(CAE) and Shared Expert in \cref{subsec:transformer}. Finally, we introduce the implementation details of our method in \cref{subsec:implementation}. 

\subsection{Overview} 
\label{subsec:problem}

% STCast is composed of three components: the Encoder, Processor, and Decoder. The Processor employs an alternating strategy that integrates window-based attention with self-attention. This hybrid design allows the model to capture both local and global dependencies within the input distribution. \textbf{More details are provided in the Appendix.}

% As illustrated in \cref{fig:framework}, this pipeline is applied to the weather forecasting task, where the model \(\Phi\) predicts future atmospheric states \(\mathbf{X}^{t+1}\) based on historical inputs \(\mathbf{X}^{t}\). Specifically, \(\mathbf{X}^{t+1}=\Phi(\mathbf{X}^{t})\), where \(\mathbf{X}^{t}\) includes upper-air variables \(\mathbf{P}^{t}\in\mathbb{R}^{H\times W\times 13\times N}\) across 13 pressure levels and surface variables \(\mathbf{S}^{t}\in\mathbb{R}^{H\times W\times M}\), with N and M denoting the number of variables per pressure and surface level, respectively. \textbf{More details are provided in the Appendix.}

As illustrated in \cref{fig:framework}, this pipeline addresses weather forecasting, in which a model \(\Phi\) predicts future atmospheric states \(\mathbf{X}^{t+1}\) from historical inputs \(\mathbf{X}^{t}\), such that \(\mathbf{X}^{t+1}=\Phi(\mathbf{X}^{t})\). The input \(\mathbf{X}^{t}\) comprises upper-air variables \(\mathbf{P}^{t}\in\mathbb{R}^{H\times W\times 13\times N}\) across 13 pressure levels and surface variables \(\mathbf{S}^{t}\in\mathbb{R}^{H\times W\times M}\), where \(N\) and \(M\) denote the number of variables per level and surface, respectively. 

% \textbf{More details are provided in \cref{sec:add_structure}.}

% In the single-step pretraining stage, this work introduces an efficient hierarchical structure to learn the atmospheric patterns, which prunes the input variables \(\mathbf{X}^{t}\) from a high resolution $HW$ to $\frac{1}{8}HW$, and then trains the latent embedding \(\mathbf{Z}^{t}\in\mathbb{R}^{C\times \frac{1}{8}HW}\) with a series of Efficient Hierarchical Transformer (EHTramsformer), where $C$ denotes the embedding dimension. 

% In single-step pretraining, we introduce an efficient hierarchical framework to learn atmospheric patterns. The input variables \(\mathbf{X}^{t}\) are first patchly divided and spatially pruned from a high resolution of \(HW\) to \(\frac{1}{16}HW\). The resulting latent representation, \(\mathbf{Z}^{t} \in \mathbb{R}^{C \times \frac{1}{16}HW}\), is then processed by a series of Efficient Hierarchical Transformers (EMFormers), where \(C\) denotes the dimension of embedding.

In single‑step pretraining, we employ an efficient framework to capture atmospheric patterns. The input variables \(\mathbf{X}^{t}\) are first partitioned into patches and spatially pruned from the resolution \(HW\) to \(\frac{1}{16}HW\). The latent representation, \(\mathbf{Z}^{t} \in \mathbb{R}^{C \times \frac{1}{16}HW}\), is then processed through a stack of EMFormers, where \(C\) denotes dimension.

During multi‑step accumulative‑context finetuning, the model learns to perform iterative forecasting by leveraging a memory module that selectively prunes and propagates historical KV states. In each step after initialization, the model conditions its prediction and the retained KV pairs from the preceding step, outputting both the prediction and an updated cache: \(\mathbf{X}^{t+2}, \mathbf{KV}^{t+1} = \Phi(\mathbf{X}^{t+1}, \mathbf{KV}^{t})\). In initial step, KV values are empty and thus no KV are injected.

% In the multi-step forecasting stage, this work unifies two key subtasks to address diverse forecasting challenges: global deterministic forecasting \(\Phi_g\), and typhoon track prediction \(\Phi_{tc}\). In the global forecasting task, this work adopts an auto-regression strategy to generate multi-step forecasts, the generated prediction will serve as the input of the next step, which is formulated as $\mathbf{X}^{t+N+1}, \mathbf{KV}^{t+N}=\Phi(\mathbf{X}^{t+N}, \mathbf{KV}^{t+N-1})$. And the generated results also provide some variables, like MSL, U10, and V10, which are used to predict the typhoon track.

% The principal contributions of this work are the EMFormer module and KV cache finetuning strategy, detailed in the following subsections. Apart from those two contribution, this paper also provides the loss function.

The multi-step forecasting stage unifies two key subtasks: global deterministic forecasting \(\Phi_g\) and typhoon track prediction \(\Phi_{tc}\). For global forecasting, autoregressive generates iterative predictions, where each output serves as the input for the subsequent step:\(\mathbf{X}^{t+N+1}, \mathbf{KV}^{t+N} = \Phi(\mathbf{X}^{t+N}, \mathbf{KV}^{t+N-1})\). The resulting forecasts provide atmospheric variables, such as mean sea-level pressure (MSL) and 10-meter wind components (U10, V10), which are subsequently utilized for typhoon track prediction.

The principal contributions are: (1) the \textbf{Efficient Multi-scale Transformer (EMFormer)}, and (2) the \textbf{Accumulative Context Finetuning} with a memory module, both of which are detailed in the subsections. In addition, we introduce a \textbf{novel loss} specifically designed for the atmosphere.

\subsection{EMFormer}
\label{subsec:emtransformer}

% The EMFormer consists of self or window attention, MLP, and two-hop residual connection, which is similar to the previous Transformer structure. Within the attention structure, the module can switch between self- and window- attention by adding and removing window partition and window reverse. The specific setting is introducing a extra Multi-Convs Layer to capture the hierarchical features from the latent space $\mathbf{Z}^{t}\in\mathbb{R}^{3C\times \frac{1}{8}HW}$. $3C$ denotes the concatenated features of Q, K, V. Different from the plain multi-scale module that capture the multi-scale information with three convolution layers of different convolution kernels, as shown in \cref{fig:multiconvs} (a), EMFormer introduces a novel multi-convs layer, which can capture multi-scale information by operating one convolution layer, as shown in \cref{fig:multiconvs} (c). At the same time, the multi-convs layer has the same function with plain multi-scale module, while be accelerated many times. To verify this function, this paper splits it into two parts, including forward path and backward path.

Building upon standard Transformer components, self- or window-based attention, MLPs, and residual connections, EMFormer introduces an efficient Multi-Convs Layer to capture multi-scale features from the latent space \(\mathbf{Z}^{t} \in \mathbb{R}^{3C \times \frac{1}{16}HW}\). Here, \(3C\) corresponds to the query-key-value concatenation. Unlike conventional multi-scale modules that employ separate convolution layers with different kernel sizes (Fig. \ref{fig:multiconvs}a), the proposed Multi-Convs Layer captures multi-scale information through a \textbf{single, fused convolution operation} (Fig. \ref{fig:multiconvs}c). This design preserves the representational capacity of a multi-branch network while achieving significant acceleration. To validate equivalence, we examine both the forward and backward paths. The pseudo-code of CUDA kernel is provided in \cref{alg:multi_convs_layer}.

In the forward pass, the equivalence of the fused operation to separate multi‑scale convolutions follows the linearity of convolution. We consider three kernels \(K_1, K_3, K_5\) of sizes \(1\times1\), \(3\times3\), and \(5\times5\), all applied with stride 1. The latent embedding \(\mathbf{Z}^{t}\) is reshaped from \(\mathbb{R}^{3C \times \frac{1}{16}HW}\) to \(\mathbb{R}^{3C \times H_0 \times W_0}\), where \(H_0 = \frac{1}{4}H\) and \(W_0 = \frac{1}{4}W\). The summed multi‑scale result among three separate convolutions (Fig.~\ref{fig:multiconvs}a) can be obtained equivalently by a single convolution:
\begin{align}
\mathbf{Z}'^{t} = \sum_{i=0}^{H_0-1} \sum_{j=0}^{W_0-1} \bigl( K_{1} \oplus K_{3} \oplus K_{5} \bigr) \odot \mathbf{Z}^{t}[i,j,5],
\end{align}
where \(K_{1} \oplus K_{3} \oplus K_{5}\) denotes by aligning and adding \(K_1, K_3, K_5\) at their centers (Fig.~\ref{fig:multiconvs}b,c), \(\oplus\) denotes element‑wise addition after zero‑padding each kernel to the largest kernel (\(5\times5\)), and \(\mathbf{Z}^{t}[i,j,5]\) is the \(5\times5\) region centered at \((i,j)\). This formulation preserves mathematical equivalence while reducing the forward pass to a single convolution, thereby lowering computation cost.

\begin{figure}[t]
  \centering
   \begin{overpic}[width=\linewidth]{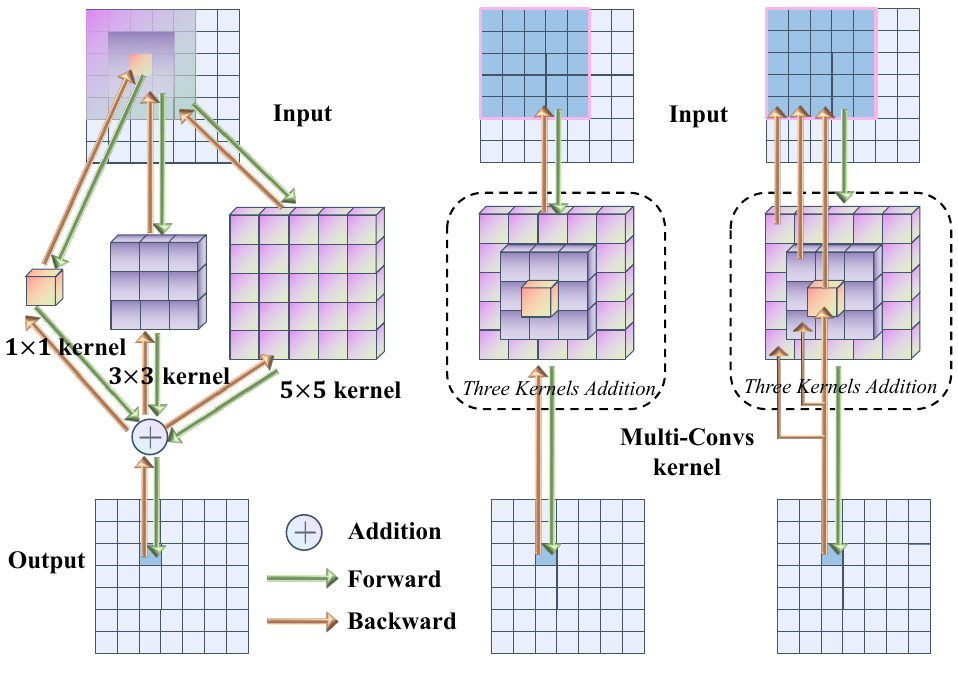} 

    % \put(7, 71.6){\fontsize{7pt}{16.8pt}\selectfont \textbf{(a) Plain Convs} }
    % \put(39, 71.6){\fontsize{7pt}{16.8pt}\selectfont \textbf{(b) Re-Parameterization} }
    % \put(71, 71.6){\fontsize{7pt}{16.8pt}\selectfont \textbf{(c) Multi-Convs Layer} }
    \put(7.7 , 0){\fontsize{7pt}{16.8pt}\selectfont \textbf{(a) Plain Convs} }
    \put(39  , 0){\fontsize{7pt}{16.8pt}\selectfont \textbf{(b) Re-Parameterization} }
    \put(71  , 0){\fontsize{7pt}{16.8pt}\selectfont \textbf{(c) Multi-Convs Layer} }
    
    \end{overpic}
   \vspace{-5.5mm}
   \caption{Illustration of Multi-Convs Layer within EMFormer.}
   \label{fig:multiconvs}
   \vspace{-0.3cm}
\end{figure}

In the backward pass, gradients for the three distinct kernels are computed separately, allowing each kernel to learn scale‑specific features while leveraging the fused forward operation. After the forward pass, the outputs of the architectures in Fig. \ref{fig:multiconvs} a, b, and c are identical, denoted \(\mathbf{Z}'^{t}\). Unlike the standard implementation in Fig. \ref{fig:multiconvs} a, our multi‑convs layer merges the three gradient‑computation loops into one, updating the gradients of all three convolutions in a single parallelized loop. With \(L\) and \(\mathbf{Z}^{t}[i,j,r]\) denoting loss and \(r \times r\) region centered at \((i,j)\), the process is formulated as:
\begin{align}
    \frac{\partial L}{\partial K_{1}},\,
    &\frac{\partial L}{\partial K_{3}},\,
    \frac{\partial L}{\partial K_{5}}
    = \sum_{i=0}^{H_0-1} \sum_{j=0}^{W_0-1}
    \Bigl[ 
    \frac{\partial L}{\partial \mathbf{Z}'^{t}[i,j]} \cdot \mathbf{Z}^{t}[i,j,1], \nonumber \\
    & \frac{\partial L}{\partial \mathbf{Z}'^{t}[i,j]} \cdot \mathbf{Z}^{t}[i,j,3], 
    \frac{\partial L}{\partial \mathbf{Z}'^{t}[i,j]} \cdot \mathbf{Z}^{t}[i,j,5] \Bigr],
\end{align}
The key difference between Fig.\ref{fig:multiconvs} a/c in backward is thus the consolidation of the gradient updates from three sequential computations into a single, parallel computation.

The standard re-parameterization module (Fig. \ref{fig:multiconvs}b) achieves computational efficiency during inference. However, applying this technique during training, the combined module mathematically degenerates into single convolution. Under standard differentiation, the gradients for the separate branches become coupled, making them functionally equivalent to single kernel, differing only in its initialization. To preserve the distinct optimization dynamics of each scale, we implement a custom CUDA kernel (Fig. \ref{fig:multiconvs}c) that decouples the backward pass, maintaining independent gradient paths for each kernel and ensuring each branch retains its unique representation throughout training.

% \begin{theorem}
%   \label{thm:multiconvs}
%   If the input feature and the initialization are the same, our multi-convs layer has the same function with the plain multi-scale module with fewer computation cost.
% \end{theorem}

% \begin{theorem}
% \label{thm:multiconvs}
% Given identical input features and initialization, the proposed multi‑convs layer produces functionally equivalent outputs to the plain multi‑scale module, while achieving a lower computational cost in both forward and backward passes.
% \end{theorem}

\begin{proposition}[Efficiency and Equivalence of Multi-Conv Layer]
\label{thm:multiconvs}
Let $\mathcal{M}_{\text{plain}}$ denote a standard multi-scale module with kernels $K_{r \in \{1,3,5\}}$, and let $\mathcal{M}_{\text{mc}}$ denote the multi-convs layer. Given identical input features $\mathbf{Z}^t$ and identical kernel initialization, the following properties hold:

\qquad \textbf{1. Function equivalence}: For spatial position $(i,j)$,
\begin{equation}
    \mathcal{M}_{\text{plain}}(\mathbf{Z}^t)[i,j] = \mathcal{M}_{\text{mc}}(\mathbf{Z}^t)[i,j].
\end{equation}

\qquad \textbf{2. Gradient equivalence}: The gradients to each weight $K_r$ are identical:
\begin{equation}
    \frac{\partial L}{\partial K_r}\bigg|_{\mathcal{M}_{\text{plain}}} = \frac{\partial L}{\partial K_r}\bigg|_{\mathcal{M}_{\text{mc}}}, \quad \forall r \in \{1,3,5\}
\end{equation}

\qquad \textbf{3. Computational efficiency}: The multi-convs layer reduces the computation complexity from $\mathcal{O}(N_{\text{kernels}} \cdot H_0 \cdot W_0 \cdot r^2)$ to $\mathcal{O}(H_0 \cdot W_0 \cdot r_{\max}^2)$, where $N_{\text{kernels}}=3$, $r_{\max}=5$.

% \begin{enumerate}
%     \item \textbf{Function equivalence}: For all spatial positions $(i,j)$,
%     \begin{equation}
%         \mathcal{M}_{\text{plain}}(\mathbf{Z}^t)[i,j] = \mathcal{M}_{\text{mc}}(\mathbf{Z}^t)[i,j].
%     \end{equation}

%     \item \textbf{Gradient equivalence}: The gradients to each weight $K_r$ are identical:
%     \begin{equation}
%         \frac{\partial L}{\partial K_r}\bigg|_{\mathcal{M}_{\text{plain}}} = \frac{\partial L}{\partial K_r}\bigg|_{\mathcal{M}_{\text{mc}}}, \quad \forall r \in \{1,3,5\}
%     \end{equation}
    
%     \item \textbf{Computational efficiency}: The multi-convs layer reduces the computation complexity from $\mathcal{O}(N_{\text{kernels}} \cdot H_0 \cdot W_0 \cdot r^2)$ to $\mathcal{O}(H_0 \cdot W_0 \cdot r_{\max}^2)$ for both forward and backward, where $N_{\text{kernels}}=3$, $r_{\max}=5$.
% \end{enumerate}

\end{proposition}

\textbf{Why does EMFormer introduce Multi-Convs Layer in Transformer?} The Multi‑Convs Layer is introduced to enable efficient capture of multi‑scale patterns, which is essential for tasks where target structures vary widely in size. By extracting features across multiple receptive fields in a single forward pass, the layer provides a computationally efficient alternative to stacking separate convolutional branches. These multi‑scale features allow the subsequent Multi‑Head Attention (MHA) module to compute affinities not only between individual points, but also between regions of different spatial extents. In contrast to standard attention, EMFormer can therefore model relationships among spatially heterogeneous structures, enhancing representational flexibility while maintaining low computation. 

% \textbf{The computation analysis is in \cref{param:computation}.}

% \begin{theorem}
%   \label{thm:multiconvs}
%   If $f:X\to Y$ is bijective, the cardinality of $X$ and $Y$ are the same.
% \end{theorem}

\subsection{Accumulative Context Finetuning}
\label{subsec:kv}

% To eliminate the error accumulation and improve the temporal consistency of weather forecasting, this work adopts a KV Cache Finetuning strategy in auto-regression to generate multi-step predictions. As shown in \cref{fig:kvcache}, the KV values from the previous steps will be concatenated and reserved in the KV cache. With the increasing of inference steps, the reserved cache may exceeds memory limitation. Thus, this work introduces a memory module to prune the overflow cache.

To mitigate error accumulation and enhance temporal consistency in multi-step forecasting, this work introduces accumulative finetuning within auto-regressive. As illustrated in \cref{fig:kvcache}, the KV pairs from previous steps of every block in EMFormer are concatenated and stored in cache. However, as the inference horizon extends, the accumulated cache may exceed memory constraints. To address this, we incorporate a memory module that dynamically prunes values while retaining information critical for maintaining consistency. The pseudo-code of KV pruning is in \cref{alg:h2o-simple}.

The memory module updates scores and prunes the cache with three steps. \textbf{(1) First}, the query \(\mathbf{Q} \in \mathbb{R}^{C \times L}\) is multiplied by the concatenated key \(\mathbf{K} \in \mathbb{R}^{C \times NL}\), which includes both the current key and cached keys from previous steps. This produces an attention map \(\mathbf{Attn} = (\mathbf{Q})^\top \mathbf{K}\), where \(N=5\) is the cache length and \(L = \frac{1}{16}HW\) represents the spatial dimension. The resulting attention map \(\mathbf{Attn} \in \mathbb{R}^{L \times NL}\) is then normalized along the key dimension using a softmax and mean function to generate the current scores, \(\mathbf{S}_{\text{cur}} = \text{Mean}(\text{Softmax}(\mathbf{Attn}))\in \mathbb{R}^{N}\). 

\textbf{(2) Next}, the scores are updated by blending the first \(N-1\) entries of \(\mathbf{S}_{\text{cur}}\) with the historical scores \(\mathbf{S}_{\text{his}}\). In the first step, since no historical scores exist, \(\mathbf{S}_{\text{his}}\) is initialized as a zero vector. For subsequent steps, the combination is governed by a hyperparameter \(\lambda =0.9\), as follows: \(\mathbf{S}_{\text{new}}[1:N-1] = \lambda \mathbf{S}_{\text{cur}}[1:N-1] + (1-\lambda)\mathbf{S}_{\text{his}}\). The \(N\)-th token in \(\mathbf{S}_{\text{cur}}\) is preserved without blending, i.e., \(\mathbf{S}_{\text{new}}[N] = \mathbf{S}_{\text{cur}}[N]\).

\textbf{(3) Finally}, the cache is pruned by selecting the top \(N-2\) scores from \(\mathbf{S}_{\text{new}}[1:N-1]\). These selected entries, together with the default \(N\)-th token, form the final \(N-1\) key-value pairs, passed to the next step as \(\mathbf{S}_{\text{his}}\). 

% \(\mathbf{S}_{\text{new}}\) also serves as \(\mathbf{S}_{\text{his}}\) in the next step.
% This process ensures efficient cache management while maintaining critical information.

\begin{figure}[t]
  \centering
   \includegraphics[width=1.0
   \linewidth]{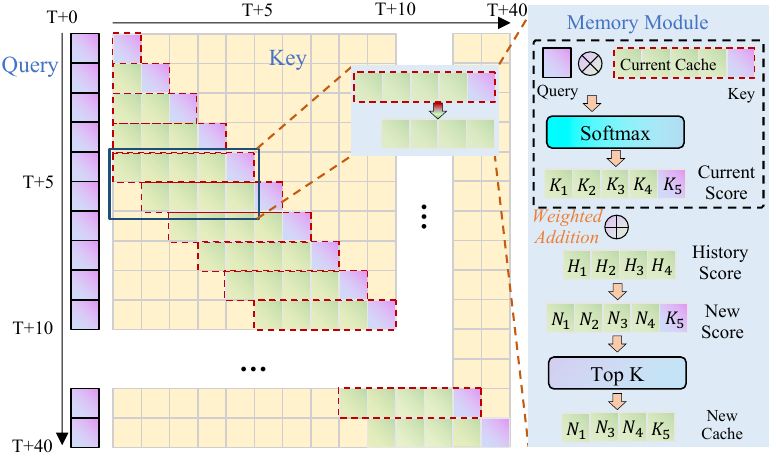}
   % \vspace{-0.05cm}
   \caption{Illustration of accumulative context finetuning.}
   \label{fig:kvcache}
   \vspace{-0.3cm}
\end{figure}

\begin{table*}[t]
    % \vspace{-5pt}
    % \vskip 0.13in
    \centering
    \caption{Performance of Ours with 4 baselines on 1.4\degree ERA5. Small RMSE (normalized, $\downarrow$) and bigger ACC (denormalized, $\uparrow$) indicate better. The best and second-best results are in \textbf{bold} and \underline{underline}. All competitors are collected from Oneforecast~\cite{oneforecast}. }
    
     \makebox[\textwidth]{
\resizebox{\linewidth}{!}{
    % \begin{small}
    %     \begin{sc}
            \renewcommand{\multirowsetup}{\centering}
            \begin{tabular}{l|cc|cc|cc|cc|cc}
                \toprule
                \multirow{4}{*}{Model} & \multicolumn{10}{c}{Metric}  \\
                \cmidrule(lr){2-11}
                &  \multicolumn{2}{c}{6-hour} & \multicolumn{2}{c}{1-day} & \multicolumn{2}{c}{4-day} & \multicolumn{2}{c}{7-day} & \multicolumn{2}{c}{10-day}   \\
                \cmidrule(lr){2-11}
               & RMSE& ACC & RMSE& ACC & RMSE& ACC & RMSE& ACC & RMSE& ACC \\

                \midrule
                % Pangu-weather \cite{panguweather} &0.0921&0.9958 &0.1635&0.9868&0.3723&0.9305&0.5576&0.8414&0.6680&0.7763     \\
                % % Fengwu \cite{fengwu} &0.1428  &0.9913 &0.2584 &0.9738  &0.5186 &0.8958 & 0.7874&0.8061 &1.2370  &0.7409    \\
                % Graphcast \cite{graphcast} &0.0887  &\underline{0.9962} &0.1634 & \underline{0.9878}  &0.3962 & 0.9308 &0.6026&0.8362 &0.7374  &0.7625    \\
                % Fuxi \cite{fuxi} &0.1106 & 0.9941 &0.1838&0.9838  &0.4329&0.9055& 0.6056&0.8116&  0.6913&0.7571  \\
                % % \midrule
                % Oneforecast \cite{oneforecast} &\textbf{0.0683} &\textbf{0.9976} &\underline{0.1294} &\textbf{0.9917} &\underline{0.2954} &\textbf{0.9577}& \underline{0.4823} &\underline{0.8839} &\underline{0.6306} &\underline{0.8025}  \\
                Pangu-weather\cite{panguweather} &0.0826&0.9876 &0.1571&0.9581&0.3380&0.8167&0.5092&0.5738&0.6215&0.3542     \\
                Graphcast\cite{graphcast} & 0.0626  & 0.9928  & 0.1304  & 0.9705   & 0.2861  & 0.8705  & 0.4597  & 0.6692  & 0.6009   & 0.4275    \\
                Fuxi\cite{fuxi} &0.0987 & 0.9820 &0.1708&0.9511  &0.4128&0.7379& 0.5972&0.4446&  0.6981&0.2391  \\
                Oneforecast\cite{oneforecast} &\textbf{0.0549} &\textbf{0.9943} &\underline{0.1231}    &\underline{0.9737}     &\underline{0.2732} &\underline{0.8825} & \underline{0.4468} &\underline{0.6888} &\underline{0.5918} &\underline{0.4457}  \\
                % \midrule
                % Graphcast \cite{graphcast} & 0.0751  & 0.9951 & 0.1451 &  0.9829  &0.3250 & 0.9300 &0.5109 &0.8452 &0.6427  &0.7644    \\
                
                % Oneforecast &\textbf{0.0726} &\textbf{0.9952} &\underline{0.1413} &\textbf{0.9835} &\underline{0.3126} &\textbf{0.9350}& \underline{0.4981} &\underline{0.8512} &\underline{0.6355} &\underline{0.7637}  \\
                
                % VA-MoE\cite{vamoe} & 0.0617 & \textbf{0.9956} & \textbf{0.1197} & \textbf{0.9740} & \textbf{0.2578} & \textbf{0.8927} & \textbf{0.4348} & \textbf{0.7019} & \textbf{0.5763} & \textbf{0.4715}  \\

                % Hourglass+Mamba & 0.0811 & 0.9872 & 0.1368 & 0.9655 & 0.2829 & 0.8719 & 0.4595 & 0.6695 & 0.5978 & 0.4365  \\

                % Ours\_v1 & 0.0814 & - & 0.1363 & - & 0.2797 & - & 0.4575 & - & 0.5997 & -  \\
                
                % Hourglass  & 0.0811 & 0.9872 & 0.1347 & 0.9664 & 0.2696 & 0.8811 & 0.4407 & 0.6878 & 0.5819 & 0.4514  \\

                \rowcolor{lightgray} Ours (w/o finetuning) & \underline{0.0626} & \underline{0.9931} & \textbf{0.1219} & \textbf{0.9749} & \textbf{0.2673} & \textbf{0.8845} & \textbf{0.4327} & \textbf{0.6978} & \textbf{0.5719} & \textbf{0.4614}  \\

                \midrule
                \midrule

                Ours (w/ finetuning) & 0.0599 & 0.9949 & 0.1139 & 0.9775 & 0.2539 & 0.8936 & 0.4072 & 0.7223 & 0.5094 & 0.5389  \\

                \bottomrule
            \end{tabular}
 %        \end{sc}
	% \end{small}
    }}
    \vspace{-0.1cm}
    
    % \vspace{-3mm}
    \label{tab:global}
\end{table*}

\begin{figure*}[t]
  \centering
   \includegraphics[width=1.0\linewidth]{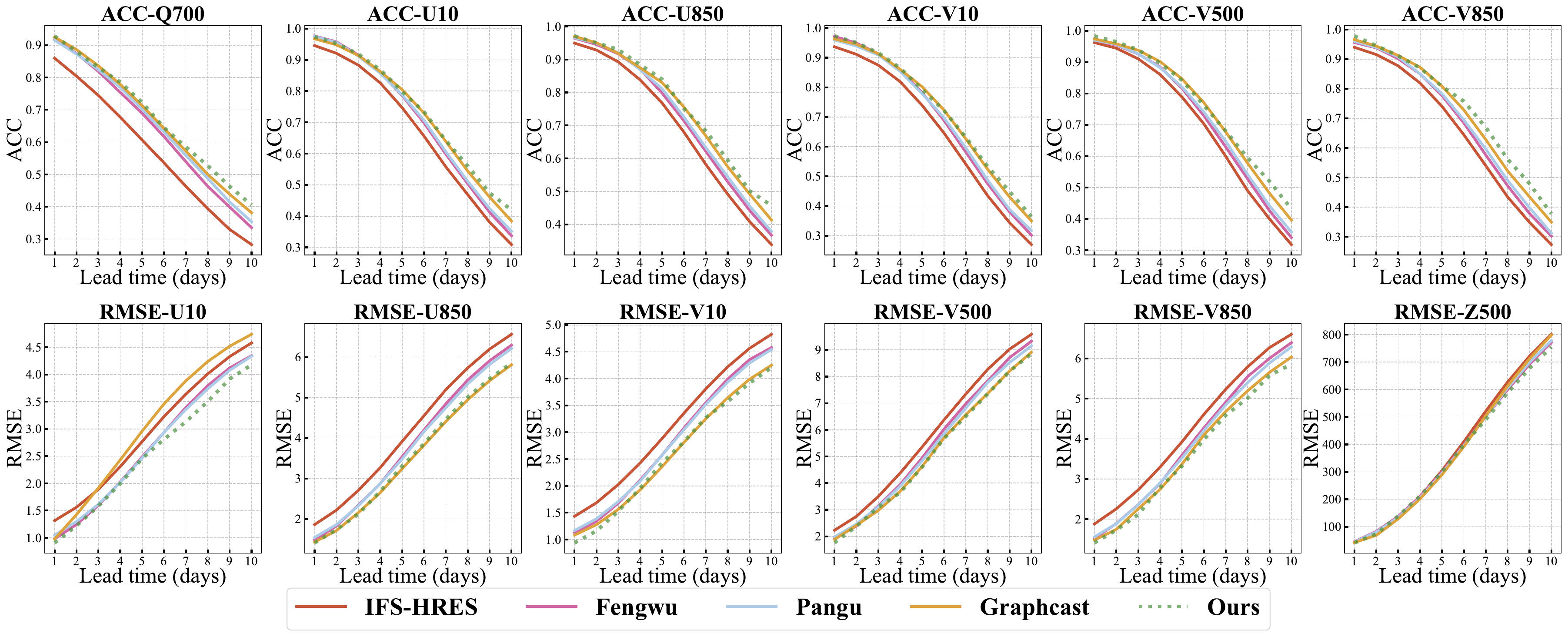}
   % \vspace{-0.8cm}
   \caption{ Comparison of our method with 4 competitors on denormalized RMSE $\downarrow$ and ACC $\uparrow$ in 0.25\degree ERA5. } 
   \label{fig:global}
   \vspace{-0.3cm}
\end{figure*}

% In this work, KV cache finetuning strategy provides the history information from the previous steps, which support the prediction in later steps and maintain the temporal consistency in the long context. This strategy guarantees the model to remember the history features even in the long-term inference, which avoids error accumulation by bringing the low-error history and avoids the degradation of the initial forecasts. To avoid the explosive growth of computation in the long-term forecasting, this work introduces a memory module to prune the KV cache, which maintain the most important information while removing the low-relevant information.

% The KV-cache finetuning strategy preserves historical information from earlier forecast steps, thereby supporting later predictions and maintaining temporal consistency across extended sequences. This approach enables the model to retain relevant features throughout long-term inference, mitigating error accumulation by incorporating low-error historical states and preventing degradation of initial forecasts. To address the growing computational burden of long-horizon forecasting, a memory module is introduced to selectively prune the KV cache, preserving the most salient information while discarding less relevant entries.

The accumulative context finetuning retains historical information from earlier steps, ensuring temporal consistency in long-term predictions. This method allows the model to preserve relevant features during long-term forecasting, reducing error accumulation by leveraging low-error historical states and preventing the degradation of initial forecasts. To mitigate the increasing computational demands of long-context forecasting, a memory module is employed to prune the KV pairs, retaining only the most critical information while discarding less relevant entries. This selective pruning balances efficiency and accuracy in long-term inference.

\subsection{Loss Function}
\label{subsec:loss}

Given that some variables exhibit minimal variation over time, we introduce a variable‑adaptive loss that optimizes different variables with distinct rates. In addition, we adopt a geography-adaptive loss for the latitudinal distribution. The final loss is designed to balance these two losses.

% The prediction objective minimizes the discrepancy between the forecast \(\mathbf{\hat{X}}^{t+1}\) and the ground truth \(\mathbf{X}^{t+1}\). 

% To adaptively adjust the importance of different variables during training, we adopt a dynamic weighting scheme:
% \begin{equation}
% \mathcal{L}_{\text{var}} = \frac{(\mathbf{\hat{X}}^{t+1} - \mathbf{X}^{t+1})^2}{e^{\mathbf{w}}} + \mathbf{w},
% \end{equation}  
% where \(\mathbf{\hat{X}}^{t+1}\) and \(\mathbf{X}^{t+1}\) are prediction and ground truth, \(e\) is the base of the natural logarithm and \(\mathbf{w} \in \mathbb{R}^{1 \times 1 \times C}\) is a learnable vector. The formulation allows the model to automatically balance the contribution of each variable through the scaling term \(e^{\mathbf{w}}\), while the penalty \(\mathbf{w}\) prevents excessive suppression of any variable.

To adaptively adjust the importance of different variables during training, we adopt a dynamic weighting scheme:
\begin{equation}
\mathcal{L}_{\text{var}} = \text{Mean}\bigl(\frac{(\mathbf{\hat{X}}^{t+1} - \mathbf{X}^{t+1})^2}{e^{\mathbf{w}}} + \mathbf{w}\bigr),
\end{equation}  
where \(\mathbf{\hat{X}}^{t+1}\) and \(\mathbf{X}^{t+1}\) denote prediction and ground‑truth states, \(e\) is the natural exponential base, and \(\mathbf{w} \in \mathbb{R}^{1 \times 1 \times C}\) is a learnable vector. To obtain a scalar loss value, `\text{Mean}' denotes Latitude-Longitude Mean function. This formulation enables the model to optimize each variable with distinct learning rates via the scaling factor \(e^{\mathbf{w}}\), while the additive penalty \(\mathbf{w}\) prevents excessive suppression of any variable.

Apart from dynamic loss, we introduce a latitude‑weighted loss that scales prediction errors according to the area of each grid cell, a practice consistent with standard evaluation metrics in meteorology. The loss is defined as:
\begin{equation}
\mathcal{L}_{\text{lat}} = \text{Mean}\bigl(\mathbf{L} \odot (\mathbf{\hat{X}}^{t+1} - \mathbf{X}^{t+1})^2\bigr),
\label{equ:latloss}
\end{equation}
where \(\mathbf{L} \in \mathbb{R}^{H \times 1}\) is a weight vector whose elements \(L_i\) are repeated across the longitude, and \(\odot\) denotes element‑wise multiplication. The weight \(L_i\) for latitude row \(i\) is given by:
\begin{equation}
L_i = N_{\text{lat}} \times \frac{\cos \phi_i}{\sum_{j=1}^{N_{\text{lat}}} \cos \phi_j},
\label{equ:l_define}
\end{equation}
where \(\phi_i\) and \(\phi_j\) represent the latitudes of \(i\) and \(j\), and \(N_{\text{lat}}\) is the total number of latitudes. This formulation ensures that errors are scaled in proportion to the area represented by each grid cell, aligning with evaluation metric.

% To balance the latitude-weighted and variable-weighted loss function, this work introduces a sine function to dynamically control the direction of the model convergence. The total loss function can be formulated as:
% \begin{equation}
%     \text{Loss} = \frac{1}{2}(1-sin(\theta))\text{Loss}_{lat}+\frac{1}{2}(1+sin(\theta))\text{Loss}_{var},
% \end{equation}
% where $sin(\theta)$ denotes the sine function, which control the value range of $\frac{1}{2}(1-sin(\theta))$ and $\frac{1}{2}(1+sin(\theta))$ to change from 0-1. 

% To balance the contributions of the latitude-weighted loss \(\mathcal{L}_{\text{lat}}\) and the variable-weighted loss \(\mathcal{L}_{\text{var}}\), we introduce a sinusoidal weighting scheme that dynamically adjusts their relative emphasis during training. The total loss is defined as a convex combination of the two terms:

% \begin{equation}
%     \mathcal{L} = \frac{1}{2}\bigl(1 - \sin(\theta)\bigr) \mathcal{L}_{\text{lat}} \;+\; \frac{1}{2}\bigl(1 + \sin(\theta)\bigr) \mathcal{L}_{\text{var}},
% \end{equation}

% where \(\theta\) is a learnable or scheduled parameter. The coefficients \(\frac{1}{2}(1 \mp \sin(\theta))\) each lie in \([0,1]\) and sum to one, thereby providing a smooth interpolation between the two loss components as \(\theta\) varies. This formulation allows the training to adaptively shift focus between latitude-corrected error minimization and variable-specific weighting throughout the optimization process.

To balance the latitude-weighted loss \(\mathcal{L}_{\text{lat}}\) and the variable-weighted loss \(\mathcal{L}_{\text{var}}\), this work introduces a sinusoidal weighting scheme that smoothly interpolates between them during training. The total loss is defined as the convex combination:
\begin{equation}
    \mathcal{L} = \frac{1}{2}\bigl(1 - \sin(\theta)\bigr) \mathcal{L}_{\text{lat}} \;+\; \frac{1}{2}\bigl(1 + \sin(\theta)\bigr) \mathcal{L}_{\text{var}},
    \label{equ:totalloss}
\end{equation}
where \(\theta\) is a learnable parameter. The coefficients \(\frac{1}{2}(1 \mp \sin\theta)\) lie in \([0,1]\) and sum to 1, enabling the optimization to shift adaptively from latitude‑corrected toward variable‑specific as training progresses.

\begin{table*}[t]
    % \vspace{-5pt}
    % \vskip 0.13in
    \centering
    \caption{ Comparative performance of Ours with 8 baselines on ten typhoon track forecasts in 2024. Best results are highlighted in \textbf{bold}. All competitors are collected from CMA (\url{tcdata.typhoon.org.cn}). }
    
    % Performance of Ours with 8 baselines on 10 Typhoon Track Forecasts in 2024. The best results are in \textbf{bold}. The average error is aggregated over lead times 6, 12, …, 96 hours.   
    
     \makebox[\textwidth]{
\resizebox{\linewidth}{!}{
    % \begin{small}
    %     \begin{sc}
            \renewcommand{\multirowsetup}{\centering}
            \begin{tabular}{l|cccccccccc|c}
                \toprule
                \multirow{2}{*}{Model} & \multicolumn{10}{c}{Average Error (km) $\downarrow$ }  \\
                % \cmidrule(lr){2-11}
                % &  \multicolumn{2}{c}{6-hour} & \multicolumn{2}{c}{1-day} & \multicolumn{2}{c}{4-day} & \multicolumn{2}{c}{7-day} & \multicolumn{2}{c}{10-day}   \\
                \cmidrule(lr){2-12}
               & AMPIL & BEBINCA & Ewiniar & GAEMI & KONG\-REY & KRATHON & MAN\-YI  & SHANSHAN  & YAGI  & Yinxing & Average \\

                \midrule
                AIFS &55.0  &213.4   &106.0    &98.5   &79.6   &187.1    &138.7   &161.2   &107.3   & \textbf{44.9}  &  119.17  \\
                
                AVNO &66.4   &160.3    &95.5    &147.0   &-   &190.9   &256.5   &281.0   &119.2   &128.9   & -   \\  %144.57
                
                ECMF &87.0   &211.5   &51.1    &123.7    &78.1    &132.4   &263.9   &179.7   &99.5  &167.5   & 139.44  \\
                
                Fengwu &99.3  &214.4  &95.3    &119.6    &80.0    &195.3     &147.7    &\textbf{90.8}   &122.5  &87.4   &  125.23  \\

                Fengqing  &122.9   & 167.6  &   -    &87.9    &62.0    & \textbf{87.5}     &195.6    & 190.0   & 103.5  & 141.1   &  -  \\

                Fourcastnet &163.8 &155.6   &111.3    &154.9    &114.2    &315.4    &252.1   &313.0   &285.3   &315.5  &  218.11   \\
                
                Pangu &78.3 &233.5   &\textbf{39.8}   &101.4    &54.9   &148.4   &130.0   &188.7   &114.7   &149.7  & 123.94    \\
                
                Graphcast & 122.4  & 171.5  & 40.2   & 134.2    & 91.1   & 167.1  & 136.2   & 182.9    & 106.9   & 55.3  &  120.78  \\
                
                Fuxi & \textbf{49.8} & 145.0   &76.7  &97.6  &57.9  &141.5  & 136.0   & -  &  125.2  &93.4  & -  \\
                \midrule
                \rowcolor{lightgray} Ours  & 66.7 & \textbf{138.2}   & 49.8 & \textbf{84.5}  & \textbf{44.9}   & 130.3   & \textbf{115.3}  & 120.3  & \textbf{75.6}   & 59.3  & \textbf{88.49} \\

                \bottomrule
            \end{tabular}
 %        \end{sc}
	% \end{small}
    }}
    % \vspace{-0.1cm}
    
    \vspace{-1mm}
    \label{tab:tc}
\end{table*}

\begin{table*}[t]
    % \vspace{-5pt}
    % \vskip 0.13in
    \centering
    % \small
    \caption{Ablation study on finetuning strategies. `w/o' and `w/' denote `without' and `with'.}
    
     \makebox[\textwidth]{
\resizebox{\linewidth}{!}{
    % \begin{small}
    %     \begin{sc}
            \renewcommand{\multirowsetup}{\centering}
            \begin{tabular}{l|cc|cc|cc|cc|cc|ccc}
                \toprule
                % \multirow{4}{*}{Model} & \multicolumn{10}{c}{Metric}  \\
                % \cmidrule(lr){2-11}
                Finetuning  &  \multicolumn{2}{c}{6-hour} & \multicolumn{2}{c}{1-day} & \multicolumn{2}{c}{4-day} & \multicolumn{2}{c}{7-day} & \multicolumn{2}{c|}{10-day} & Params & GPU & Latency   \\
                \cmidrule(lr){2-11}
                Strategies  & RMSE& ACC & RMSE& ACC & RMSE& ACC & RMSE& ACC & RMSE& ACC & (M) & (G) & Avg/(ms)   \\

                \midrule

                % Ours\_v1 & 0.0814 & - & 0.1363 & - & 0.2797 & - & 0.4575 & - & 0.5997 & -  \\
                
                % Hourglass  & 0.0811 & 0.9872 & 0.1347 & 0.9664 & 0.2696 & 0.8811 & 0.4407 & 0.6878 & 0.5819 & 0.4514  \\

                \multicolumn{14}{c}{   \textbf{w/o Accumulative Finetuning}   } \\
                \midrule
                
                w/o Finetuning & 0.0626 & 0.9931 & 0.1219 & 0.9749 & 0.2673 & 0.8845 & 0.4327 & 0.6978 & 0.5719 & 0.4614   & 157.9  & 0.703  &  98.3 \\

                w/ Finetuning   & 0.0645 & 0.9917 & 0.1189 & 0.9756 & 0.2491 & 0.8869 & 0.4261 & 0.7098 & 0.5339 & 0.4958   & 157.9  &  0.726  &  99.1 \\

                w/ Lora Finetuning  & 0.0626 & 0.9931 & 0.1173 & 0.9759 & 0.2473 & 0.8936 & 0.4224 & 0.6981 & 0.5278 & 0.5041  & 171.3  & 4.919  &  111.3   \\

                \midrule
                \multicolumn{14}{c}{   \textbf{w/ Accumulative Finetuning (KV Cache Length=5)}   } \\
                \midrule

                w/o Pruning  & 0.0599 & 0.9950 & 0.1152 &  0.9766 &    0.2466 & 0.8852 &    0.4109 & 0.7191 &     0.5153 & 0.5316  &     157.9  & 4.576  &  118.5  \\

                w/ Pruning  & 0.0599 & 0.9949 & 0.1139 & 0.9775 &     0.2439 & 0.8936 &      0.4072 & 0.7223 & 0.5094 & 0.5389  &      157.9  & 4.913  &  120.5  \\

                \bottomrule
            \end{tabular}
 %        \end{sc}
	% \end{small}
    }}
    % \vspace{-0.1cm}
    
    \vspace{-1mm}
    \label{tab:ablation_finetune}
\end{table*}

\begin{table*}[t!]
    % \vspace{-5pt}
    % \vskip 0.13in
    \centering
    \small
    \caption{Ablation study on the backbone's blocks. All results are gained with 1.4\degree ERA5.}
    
%      \makebox[\textwidth]{
% \resizebox{\linewidth}{!}{
    % \begin{small}
    %     \begin{sc}
            \renewcommand{\multirowsetup}{\centering}
            \begin{tabular}{l|cc|cc|cc|cc|cc}
                \toprule
                % \multirow{4}{*}{Structure} & \multicolumn{10}{c}{Metric}  \\
                % \cmidrule(lr){2-11}
                \multirow{2}{*}{Block of Backbone}  &  \multicolumn{2}{c}{6-hour} & \multicolumn{2}{c}{1-day} & \multicolumn{2}{c}{4-day} & \multicolumn{2}{c}{7-day} & \multicolumn{2}{c}{10-day}   \\
                \cmidrule(lr){2-11}
               & RMSE& ACC & RMSE& ACC & RMSE& ACC & RMSE& ACC & RMSE& ACC \\

                \midrule

                % Hourglass, Hourglass+RWKV, Hourglass+windows

                Self-Attention Block  & 0.0711 & 0.9912 & 0.1287 & 0.9664 & 0.2696 & 0.8711 & 0.4603 & 0.6015 & 0.5978 & 0.4499  \\

                MambaVision Block & 0.0711 & 0.9912 & 0.1308 & 0.9655 & 0.2829 & 0.8619 & 0.4595 & 0.6695 & 0.5978 & 0.4365  \\

                % VRWKV Block  & 0.0814 & - & 0.1363 & - & 0.2797 & - & 0.4575 & - & 0.5997 & -  \\

                Self/Windows Attn Block  & 0.0709    & 0.9915     & 0.1298    & 0.9665    & 0.2786   & 0.8679   & 0.4566    & 0.6806    & 0.6016 & 0.4486  \\

                % Ours\_v1 & 0.0814 & - & 0.1363 & - & 0.2797 & - & 0.4575 & - & 0.5997 & -  \\
                
                % EHTransformer  & 0.0599 & 0.9949 & 0.1139 & 0.9775 & 0.2439 & 0.8936 & 0.4072 & 0.7223 & 0.5094 & 0.5389  \\

                EMFormer  & 0.0626 & 0.9931 & 0.1219 & 0.9749 & 0.2673 & 0.8845 & 0.4327 & 0.6978 & 0.5719 & 0.4614 \\

                \bottomrule
            \end{tabular}
 %        \end{sc}
	% \end{small}
    % }}
    % \vspace{-0.1cm}
    
    \vspace{-3mm}
    \label{tab:ablation_structure}
\end{table*}

% In the initialization, $\theta$ is set to $-\frac{1}{2}\pi$ when the learning rate is large, which means \(\frac{1}{2}\bigl(1 - \sin(\theta)\bigr)=1\) and \(\frac{1}{2}\bigl(1 + \sin(\theta)\bigr)=0\). With the decline of learning rate, $\theta$ gradually increases and making \(\frac{1}{2}\bigl(1 - \sin(\theta)\bigr)=0\) and \(\frac{1}{2}\bigl(1 + \sin(\theta)\bigr)=1\) gradually. The above process

% \begin{theorem}
%   \label{thm:loss}
%   Through the balancing of the sin function \(\frac{1}{2}(1 \mp \sin\theta)\), the loss function in \cref{equ:totalloss} will automatically transform from latitude‑weighted loss to variable-weighted loss during training.
% \end{theorem}

% \begin{theorem}
% \label{thm:loss}
% The sinusoidal weighting coefficients \(\frac{1}{2}(1 \mp \sin\theta)\) in \cref{equ:totalloss} induce an automatic transition during training: the total loss gradually shifts emphasis from the latitude‑weighted term toward the variable‑weighted term as the parameter \(\theta\) evolves from \(-\frac{1}{2}\pi\) to \(\frac{1}{2}\pi\).
% \end{theorem}

% \begin{theorem}
% \label{thm:loss}
% With the sinusoidal weighting coefficients \(\frac{1}{2}(1 \mp \sin\theta)\) in the total loss (\cref{equ:totalloss}), the parameter \(\theta\) automatically evolves from \(-\pi/2\) to \(\pi/2\) during training. Consequently, the loss function smoothly transitions its emphasis from the latitude‑weighted term to the variable‑weighted term.
% \end{theorem}

\begin{proposition}[Adaptive Loss Weighting]
\label{thm:loss}
Consider the loss function in \cref{equ:totalloss}. With learning rate $\eta > 0$,  $\theta_0 = -\pi/2$, and $\mathbf{w}_0 = \mathbf{0}$, the following properties hold:

\qquad 1. The parameter $\theta$ evolves monotonically from $-\pi/2$ to $\pi/2$ during training;

\qquad 2. The loss function $\mathcal{L}$ automatically transitions its emphasis from the latitude-weighted term $\mathcal{L}_{\text{lat}}$ to the variable-aware term $\mathcal{L}_{\text{var}}$ as training progresses.

% \begin{enumerate}
%     \item The parameter $\theta$ evolves monotonically from $-\pi/2$ to $\pi/2$ during training;
    
%     % \item There exists a finite training iteration $k^*$ such that for all $k \geq k^*$, the gradient $\frac{\partial \mathcal{L}}{\partial \theta} < 0$, ensuring continuous progression toward $\theta = \pi/2$;
    
%     \item The loss function $\mathcal{L}$ automatically transitions its emphasis from the latitude-weighted term $\mathcal{L}_{\text{lat}}$ to the variable-aware term $\mathcal{L}_{\text{var}}$ as training progresses.
% \end{enumerate}

\end{proposition}

\textbf{Why do we need sinusoidal weighting?} Atmospheric variables evolve at different rates (\textit{e.g.}, geopotential shows minimal variation over 6‑hour intervals, whereas temperature exhibits pronounced changes), making it essential to adjust their learning speeds separately, especially during the low-learning-rate stage. We therefore employ a variable-adaptive loss to assign effective per-variable learning rates, alongside a latitude-weighted loss to respect geographical scaling. Simply combining these losses degrades performance, as it conflates spatial and physical adjustments prematurely. Instead, we introduce a sinusoidal weighting mechanism that smoothly transitions the loss emphasis: initially prioritizing latitude-corrected error for stable coarse learning, and later shifting focus to variable-specific refinement for fine-grained optimization. This curriculum aligns with the multi-stage nature of training and the inherent heterogeneity of atmospheric data, yielding superior forecasts.

\section{Experiments}
\label{sec:experiment}

\begin{table}[t]
\renewcommand\arraystretch{.9}
\centering
% \tablestyle{6pt}{1.00}
\caption{Comparison of classification benchmarks on \textbf{ImageNet-1K} dataset~\cite{deng2009imagenet}. Throughput (TP) is measured on A100 with batch size of 128. All are tested with $224\times 224$ size. } 
% \vspace{-1mm}
\resizebox{1.0\linewidth}{!}{
\setlength{\tabcolsep}{.5mm}{
\begin{tabular}[t]{lcccc}
\toprule
\multirow{2}{*}{Model} & \#Params & FLOPs & TP & Top-1 \\
      & (M) & (G) & (Img/Sec) & (\%) \\

\midrule
\multicolumn{5}{c}{\textbf{Tiny Model}} \\
\midrule

ConvNeXt-T~\cite{liu2022convnet} &   28.6 & 4.5 & 3196 & 82.0\\
ResNetV2-50~\cite{wightman2021resnet} &  25.5 & 4.1 & 6402 & 80.4\\
Swin-T~\cite{liu2021swin} &   28.3 & 4.4 & 2758 & 81.3\\
TNT-S~\cite{han2021transformer} &   23.8 &  4.8 & 1478 & 81.5\\
Twins-S~\cite{chu2021twins} &   24.1 &  2.8 & 3596 & 81.7\\
DeiT-S~\cite{touvron2021training} &  22.1 & 4.2 & 4608 & 79.9 \\
PoolFormer-S36~\cite{yu2022metaformer} &  30.9 & 5.0 & 1656 & 81.4 \\
CrossViT-S~\cite{chen2021crossvit} &26.9 & 5.1 & 2832 & 81.0 \\
NextViT-S~\cite{li2022next} &  31.7 & 5.8 & 3834 & 82.5 \\
EfficientFormer-L3~\cite{li2022efficientformer} &   31.4 &  3.9 & 2845 & 82.4 \\
VMamba-T~\cite{liu2024vmamba} & 30.0 & 4.9 & 1282 & 82.6\\
EfficientVMamba-B~\cite{efficientvmamba} & 33.0 & 4.0 & 1482 & 81.8\\
VRWKV-S~\cite{duan2025visionrwkv}   & 23.8  & 4.6  &   -  &   80.1   \\
MambaVision-T~\cite{hatamizadeh2025mambavision} &  31.8 &  4.4 & 6298 &  82.3 \\
MambaOut-Tiny~\cite{yu2025mambaout}    &  27.0    &  4.0  &   -   &  82.7  \\
\rowcolor{lightgray} EMFormer-T    &  28.7    &  5.1  &   3378   &  \textbf{83.2}  \\

\midrule
\multicolumn{5}{c}{\textbf{Small Model}} \\
\midrule

ConvNeXt-S~\cite{liu2022convnet}  &  50.2 & 8.7 & 2008 & 83.1\\
ResNetV2-101~\cite{wightman2021resnet} &  44.5 &  7.8 & 4019 & 82.0\\
Swin-S~\cite{liu2021swin}  &  49.6 &  8.5 & 1720 & 83.2\\
Twins-B~\cite{chu2021twins} &  56.1 &  8.3 & 1926 & 83.1\\
PoolFormer-M36~\cite{yu2022metaformer} & 56.2 & 8.8 & 1170 & 82.1\\ 
NextViT-L~\cite{li2022next} &  57.8 & 10.8 & 2360 & 83.6\\
FasterViT-1~\cite{hatamizadeh2023fastervit} &  53.4 &  5.3 & 4188 &  83.2\\
VMamba-S~\cite{liu2024vmamba} &50.0 & 8.7 & 843 & 83.6\\
MambaVision-S~\cite{hatamizadeh2025mambavision} &  50.1 &  7.5 & 4700 &  83.3 \\
MambaOut-Small~\cite{yu2025mambaout}     &  48.0    &  9.0  &   -   &  \textbf{84.1}  \\
\rowcolor{lightgray} EMFormer-S &  45.5 &  7.4 & 2512 &  \textbf{84.1} \\

\midrule
\multicolumn{5}{c}{\textbf{Base Model}} \\
\midrule

ConvNeXt-B~\cite{liu2022convnet} &  88.6 & 15.4 & 1485 & 83.8\\
Twins-L~\cite{chu2021twins} &  99.3 & 14.8 & 1439 & 83.7\\
CrossViT-B~\cite{chen2021crossvit} & 105.0 & 20.1 & 1321 & 82.2\\
EfficientFormer-L7~\cite{li2022efficientformer} &  82.2 & 10.2 & 1359 & 83.4\\
VMamba-B~\cite{liu2024vmamba} &89.0 & 15.4 & 645 & 83.9\\
DeiT-B~\cite{touvron2021training} &  86.6 & 16.9 & 2035 & 82.0\\
DeiT3-B~\cite{touvron2022deit} & 86.6 & 16.9 & 670 & 83.8 \\
VRWKV-B~\cite{duan2025visionrwkv}   & 93.7  & 18.2  &  -   &   82.0   \\
MambaVision-B~\cite{hatamizadeh2025mambavision} & 97.7 & 15.0 & 3670 &  84.2 \\
MambaOut-Base~\cite{yu2025mambaout}    &  85.0    &  15.8  &   -   &  84.2  \\
\rowcolor{lightgray} EMFormer-B & 80.6 & 12.3 & 1693   &  \textbf{84.4} \\

% \rowcolor{Gray}
\bottomrule
\end{tabular}
}
}
\vspace{-0.4cm}
\label{tab:imgnet}
\end{table}

\begin{table}
\centering
\caption{Semantic segmentation with \textbf{ADE20K}~\cite{ade20k}. All models are trained using a crop resolution of $512\times 512$.}
\resizebox{1.0\linewidth}{!}{
\footnotesize
\setlength{\tabcolsep}{2.5pt}
  \begin{tabular}{lccc}
    \toprule
    Backbone  & Param & FLOPs  & mIoU   \\
              & (M)   & (G)    &        \\
    
    \midrule	 
    DeiT-Small/16~\cite{touvron2021training}  & 52 & 1099 & 44.0\\
    Swin-T~\cite{liu2021swin}  & 60 & 945 & 44.5\\
    ResNet-101~\cite{he2016deep}  & 86 & 1029 & 44.9\\
    Focal-T~\cite{yang2021focal}  & 62 & 998 & 45.8\\
    % ConvNeXt-T~\cite{liu2022convnet}  & 60 & 939 & 46.0\\ rowcolor{Gray}
    VMamba-T~\cite{liu2024vmamba}    & 62  & 949   &  48.0 \\
    EfficientVMamba-B~\cite{efficientvmamba}    & 65  & 930   &  46.5 \\
    MambaVision-T~\cite{hatamizadeh2025mambavision}  & 55 & 945 & 46.0\\
    \rowcolor{lightgray}   EMFormer-S  & 48 & 238 & 46.7 \\
    
    \midrule
    Swin-S~\cite{liu2021swin}  & 81 & 1038 & 47.6\\
    Twins-SVT-B~\cite{chu2021twins}  & 89 & - & 47.7\\
    Focal-S~\cite{yang2021focal}  & 85 & 1130 & 48.0\\
    VMamba-S~\cite{liu2024vmamba}    & 76  & 1028   &  50.6 \\
    MambaVision-S~\cite{hatamizadeh2025mambavision}  & 84 & 1135 & 48.2 \\
    % ConvNeXt-S~\cite{liu2022convnet}  & 84 & 1027 & \textbf{48.7}\\
    \rowcolor{lightgray}   EMFormer-B  & 69  & 251  & 49.6 \\
    
    % \midrule
    % Swin-B~\cite{liu2021swin}  & 121 & 1188 & 48.1\\
    % Twins-SVT-L~\cite{chu2021twins}  & 133 & - & 48.8\\
    % Focal-B~\cite{yang2021focal}  & 126 & 1354 & 49.0\\
    % VMamba-B~\cite{liu2024vmamba}    & 122  & 1170   &  51.0 \\
    % MambaVision-B~\cite{hatamizadeh2025mambavision}  & 126 & 1342 & 49.1 \\
    % % Focal-B~\cite{yang2021focal}  & 126 & 1354 & 49.0\\
    % % ConvNeXt-B~\cite{liu2022convnet}  & 122 & 1170 & \textbf{49.1}\\ 

    \bottomrule
  \end{tabular} 
  }
%   \vspace{3pt}
    
    \vspace{-0.1cm}
    \label{tab:ade_segmentation}
\end{table}

\begin{table}[t]
    \centering
    \small
    \caption{Comparison of plain multi-scale module and multi-convs layer on our method with weather forecasting without finetuning. }
            \renewcommand{\multirowsetup}{\centering}
            \resizebox{1.0\linewidth}{!}{
            \begin{tabular}{l|cccc|cc}
                \toprule
                \multirow{2}{*}{Model} & \multicolumn{4}{c|}{6-hour prediction (RMSE $\downarrow$)}  &  Params  & Time \\
                \cmidrule(lr){2-5}  
               & Z500   & T2M    & T850    & U10  & (M)   & (Hours) \\

                \midrule

                Multi-scale &   18.6   &   0.531    &   0.421   &  0.525  & 157.9   &  83    \\
                Multi-convs &   18.6   &   0.533    &   0.422   &  0.523  & 157.9  &  60  \\
            
                \bottomrule
            \end{tabular} }
    \vspace{-0.1cm}
    \label{tab:add_theory1_weather}
\end{table}

\begin{table}[t!]
    \centering
    \small
    \caption{Comparison of plain multi-scale module and multi-convs layer on our base method with image classification task. }
            \renewcommand{\multirowsetup}{\centering}
            \resizebox{1.0\linewidth}{!}{
            \begin{tabular}{l|cc|cc}
                \toprule
               Model & Top-1 & Top-5& Params  & Times  \\
                & (\%) & (\%) & (M)  & (Hours)  \\
            
                \midrule

                EMFormer-B (Multi-scale) &   84.3 &  97.1  &  80.6  &  123   \\
                EMFormer-B (multi-convs) &   84.4 &  96.9  &  80.6  &  107   \\
            
                \bottomrule
            \end{tabular}  }
    \vspace{-0.4cm}
    \label{tab:add_theory1_imagenet}
\end{table}

\subsection{Dataset and Implementation Details} 
\label{subsec:dataset}

We evaluate the proposed model with three benchmarks: ERA5 reanalysis dataset~\cite{era5} for atmospheric forecasting, alongside the ImageNet‑1K~\cite{deng2009imagenet} and ADE20K~\cite{ade20k} benchmarks for classification and segmentation, respectively. 

All experiments are conducted on 16 NVIDIA Tesla A100 GPUs. The codes will be released.

% ERA5 provides continuous atmospheric fields from 1979 to 2020 at a spatial resolution of 0.25\degree (corresponding to \(721 \times 1440\) ). To reduce computational cost, a subset of experiments uses a spatially downsampled version at 1.4\degree (\(128 \times 256\)). 

\textbf{Statement 1.} Related Works (\cref{sec:related}), Dataset and Implementation Details (\cref{sec:add_dataset}), Evaluation Metric (\cref{sec:add_metric}), Computation Comparisons(\cref{subsec:times}), Additional Results (\cref{subsec:add_result}), \textbf{Additional Ablation Study} (\cref{subsec:add_ablation}), and Visualization (\cref{subsec:add_vis}) are provided in Appendix.

\textbf{Statement 2.} More competitors, ClimaX~\cite{climax}, EWMoE~\cite{ewmoe}, Keisler~\cite{keisler}, FourCastNet~\cite{fourcastnet}, ClimODE~\cite{climode}, WeatherGFT~\cite{weathergft}, FGN~\cite{alet2025skillful}, and GenCast~\cite{gencast}, are in \cref{subsec:add_result}.

% 1. performance on weather forecasting (1.4， 0.25)
%     weather forecasting
%     typhoon track

% 2. performance on basic vision tasks
%     classification
%     detection
%     segmentation

% 3. parameter analysis and loss curves
%     分析为什么现在的模型比之前的更好

% 3. ablation study
    
%     1) multiconv layer:
%         efficientvit (w/ w/o multiconv)
%         params, flops

%     1.5) 在global forecasts任务上测试不同结构给模型性能带来的影响。

%     2) kv-cache finetuning
%         ordinary finetuning vs kv-cache finetuning
%         训练过程中的变化曲线，随着learning rate

%         ？在现有的模型上测试，普通finetuning和kv-cache finetuning的性能差异

%     3) loss function
%         损失函数的选择
%         超参数的选择

\subsection{Main Results}
\label{subsec:results}

% \textbf{Weather forecasting.} We evaluate model performance using two standard metrics: RMSE and ACC. Due to the substantial differences in physical scales across atmospheric variables, a direct comparison using absolute RMSE values is not meaningful. Therefore, we report ACC and normalized RMSE in \cref{tab:global} with 1.4\degree dataset, where our method demonstrates consistent and superior performance against all baseline models, with particularly notable gains in long-term forecasts. Additional validation via ACC and denormalized RMSE for 1–10 day predictions in \cref{fig:global} with 0.25\degree dataset further confirms state-of-the-art results across multiple variables. These improvements are driven by two key components: our efficient multi-scale transformer, which captures atmospheric patterns across varied receptive fields, and the KV-cache fine-tuning strategy, which enhances temporal consistency and reduces error accumulation in extended forecasts.

\textbf{Weather forecasting.} Model performance is evaluated with Root Mean Square Error (RMSE) and Anomaly Correlation Coefficient (ACC). Owing to the substantial differences in physical scales among variables, we report mean ACC and mean normalized RMSE in \cref{tab:global}, where experiments are conducted on a downsampled 1.4\degree grid. Our method consistently outperforms baseline models, with especially pronounced gains in longer‑term forecasts. Additional validation on the 0.25\degree grid with ACC and denormalized RMSE for 1–10‑day predictions (\cref{fig:global}) further confirms better results across multiple atmospheric variables. These improvements stem from two core contributions: (1) the efficient multi‑scale transformer, which captures patterns across varied receptive fields, and (2) the accumulative context finetuning, which enhances temporal consistency and mitigates catastrophic forgetting in extended forecasts.

\textbf{Typhoon track prediction.} To assess performance under extreme conditions, we evaluate 10 typhoons from the 2024~\cite{ying2014overview,lu2021western}. \cref{tab:tc} compares the mean 96‑hour track errors (aggregated over lead times 6, 12, …, 96 hours) against 9 baselines: AIFS~\cite{lang2024aifs}, AVNO, ECMWF~\cite{molteni1996ecmwf}, FengWu~\cite{fengwu}, Fengqing, FourCastNet~\cite{fourcastnet}, Pangu~\cite{panguweather}, Graphcast~\cite{graphcast}, and FuXi~\cite{fuxi}. Our method obtains the lowest error on 5 typhoons and remains competitive on the remaining five. The overall mean error across all typhoons is 88.49 km, the best among all competitors and notably lower than the next‑best result (119.17 km). These results, together with the consistent ability to capture cyclone evolution, confirm its potential for extreme event forecasting. The error of individual lead times are in \cref{fig:tcs_part1,fig:tcs_part2}.

\textbf{Image classification.} Beyond weather forecasting, we also evaluate EMFormer on the ImageNet-1K and compare it with recent SoTA methods. As summarized in \cref{tab:imgnet}, models are grouped into three parameter scales: \textasciitilde 30M (tiny), \textasciitilde 50M (small), and \textasciitilde 90M (base). Our EMFormer achieves the highest accuracy in all three categories, 83.2\%, 84.1\%, and 84.4\% for the tiny, small, and base models, respectively. Notably, at the base level, our model surpasses MambaOut~\cite{yu2025mambaout}, MambaVision~\cite{hatamizadeh2025mambavision}, and VRWKV~\cite{duan2025visionrwkv} while using fewer parameters and FLOPs. These results indicate that EMFormer is not only effective for weather forecasting, but also generalizes well to vision tasks, owing to its ability to capture multi‑scale patterns through varied receptive fields and to accelerate computation via the fused multi‑convs layer.

% Apart from weather forecasting tasks, this work also presents the comparison between our method with some recent SoTA competitors. As shown in \cref{tab:imgnet}, the model comparisons are split into three levels according to parameters, ~30M, ~50M, and ~90M. Our work achieves the best performance on all three comparison with the highest accuracy 83.2\%, 84.1\%, and 84.4\% in tiny, small, and base model. Especially in base model level, our work performs better than MambaOut~\cite{yu2025mambaout}, MambaVision~\cite{hatamizadeh2025mambavision}, and VRWKV~\cite{duan2025visionrwkv} with lower parameters and flops. From the experimental results, we see that our structure not only makes sense in weather forecasting, but also in basic vision task, whihch captures different pixel levels with various receptive fields and accelerates the computation with multi-convs layer.

% \textbf{Semantic Segmentation.} To further evaluate the performance on vision task, this work also provides the experimental results on ADE20K in \cref{tab:ade_segmentation}. From the experimental results, we find that our method achieves a comparable performance 46.7 and 49.6 on small and base models with a extremely lower computation costs, almost just 75\% parameters and 25\% FLOPs with other methods. The experiments indicate that our work can effectively support visional tasks and reduces the computation cost.

\textbf{Semantic Segmentation.} To further assess the generalization of our EMFormer, we evaluate its performance on the ADE20K (see \cref{tab:ade_segmentation}). Our small and base models achieve mIoU scores of 46.7 and 49.6, respectively, while requiring approximately 75\% of the parameters and 25\% of the FLOPs compared to other methods. These results demonstrate that the proposed framework not only delivers competitive accuracy on dense prediction tasks, but also maintains significantly lower computational overhead.

\subsection{Ablation Study}
\label{subsec:ablation}

\textbf{Multi-Convs layer. } To further validate the efficacy of the multi‑convs layer, we conduct comparative experiments on ERA5 (\cref{tab:add_theory1_weather}) and ImageNet‑1K (\cref{tab:add_theory1_imagenet}). Results show that models equipped with either the plain multi‑scale module or the multi‑convs layer achieve nearly identical performance on both tasks. However, the multi‑convs layer yields substantially lower computational cost, reducing the required training time by approximately 25\% for weather forecasting and 20\% for image classification.

% \textbf{Finetuning strategy. } \cref{tab:ablation_finetune} provides the ablation study on different finetuning strategies with three settings, without finetuning, with finetuning, and with kv-cache finetuning. From the experimental results, we observe that plain finetuning and kv-cache finetuning strategies both performs better than the model without finetuning in the long-term prediction. In the short-term forecasting, the kv-cache finetuning performs better than the non-finetuning model, while the plain finetuning is a little worse. From the experimental results, we draw the conclusion that kv-cache finetuning has the ability to maintain the temporal consistency, ensuring a good performance from short-term to long-term forecasts.

% \textbf{Finetuning strategy. } \cref{tab:ablation_finetune} presents an ablation study on 3 strategies: no finetuning, finetuning, and accumulative context finetuning. Experimental results show that both finetuning and accumulative finetuning outperform the non‑finetuned model in long‑term predictions. For short‑term forecasts, accumulative finetuning yields better results than the non‑finetuned baseline, whereas finetuning is slightly inferior. These findings demonstrate that accumulative finetuning effectively preserves temporal consistency, delivering robust performance across both short‑ and long‑term forecasts.

\textbf{Finetuning strategy. }  \cref{tab:ablation_finetune} presents an ablation study on three finetuning strategies: no finetuning, standard finetuning, and accumulative‑context finetuning. Results show that both finetuning variants outperform the non‑finetuned baseline in long‑term predictions. For short‑term forecasts, accumulative finetuning achieves better accuracy than the baseline, whereas standard one performs slightly worse. These findings indicate that accumulative finetuning maintains temporal consistency across forecast horizons, delivering robust performance in both short‑ and long‑term forecasts.

% \textbf{Structure in the backbone. } \cref{tab:ablation_structure} presents an ablation study on the structure of backbone, including self-attention, mamba block, self/windows attention, and our EHTransformer. There is limited difference between several structures in the short-term and long-term forecasts. But EHTransformer earns an extra 0.0083 RMSE improvement in the 6-hour forecasts, and 0.0159 RMSE improvement in the 10-day forecasts. The experiments indicate that our EHTransformer has a better performance in the weather forecasting task, and the multi-scale module is suitable for capturing atmospheric patterns with various regions.

\textbf{Block in backbone.} \cref{tab:ablation_structure} presents an ablation study comparing different blocks: self‑attention, Mamba, self/window attention, and EMFormer. While all variants yield only minor variations in both short‑ and long‑term forecasts, EMFormer achieved a further reduction of 0.0083 for 6‑hour and 0.0297 for 10‑day predictions in RMSE. These results indicate that EMFormer delivers superior performance on forecasting, and its integrated multi‑convs layer is effective at capturing atmospheric patterns across different scales.

\textbf{Ablation studies} about Loss function (\cref{tab:ablation_loss}), $\lambda$ in finetuning (\cref{tab:ablation_lambda}), Cache length (\cref{tab:ablation_cachelength}), Finetuning steps (\cref{tab:ablation_steps}), and Window size (\cref{tab:ablation_windows}) are all in Appendix.

\section{Conclusion}
\label{sec:conclusion}

In this paper, we present a novel pipeline for weather forecasting. This work first employs an efficient framework with a multi‑convs layer to capture multi‑scale atmospheric patterns during training. It then adopts accumulative context finetuning to improve long‑context consistency while maintaining short‑term forecast quality. A sine‑balanced loss is introduced to adaptively combine variable‑ and latitude‑weighted objectives during optimization. The design is supported by theoretical analysis that validates both the efficiency of the multi‑conv layer and the convergence behavior of the loss. Experiments show that the proposed method achieves great results in atmospheric field and delivers competitive performance on vision benchmarks, demonstrating its broad applicability and effectiveness.

\section*{Acknowledgements}

This research was supported by fundings from the Hong Kong RGC General Research Fund (152228/23E, 162161/24E, 162116/25E, 162180/25E), National Natural Science Foundation of China (NSFC) Key Program (No.62532005), Collaborative Research Fund (No. C1042-23GF, No. C5097-25G), NSFC/RGC Collaborative Research Scheme (Grant No. 62461160332 \& CRS\_HKUST602/24), Research Impact Fund (No. R5011-23F), Areas of Excellence Scheme (AoE/E-601/22-R), and the InnoHK (HKGAI).

% \textbf{Do not} include acknowledgements in the initial version of the paper

% If a paper is accepted, the final camera-ready version can (and usually should)
% include acknowledgements.  Such acknowledgements should be placed at the end of
% the section, in an unnumbered section that does not count towards the paper
% page limit. Typically, this will include thanks to reviewers who gave useful
% comments, to colleagues who contributed to the ideas, and to funding agencies
% and corporate sponsors that provided financial support.

\section*{Impact Statement}

This paper presents work whose goal is to advance the field of Machine
Learning. There are many potential societal consequences of our work, none
which we feel must be specifically highlighted here.

\nocite{langley00}

\bibliography{example_paper}
\bibliographystyle{icml2026}

%%%%%%%%%%%%%%%%%%%%%%%%%%%%%%%%%%%%%%%%%%%%%%%%%%%%%%%%%%%%%%%%%%%%%%%%%%%%%%%
%%%%%%%%%%%%%%%%%%%%%%%%%%%%%%%%%%%%%%%%%%%%%%%%%%%%%%%%%%%%%%%%%%%%%%%%%%%%%%%
% APPENDIX
%%%%%%%%%%%%%%%%%%%%%%%%%%%%%%%%%%%%%%%%%%%%%%%%%%%%%%%%%%%%%%%%%%%%%%%%%%%%%%%
%%%%%%%%%%%%%%%%%%%%%%%%%%%%%%%%%%%%%%%%%%%%%%%%%%%%%%%%%%%%%%%%%%%%%%%%%%%%%%%
\newpage
\appendix
\onecolumn

\section*{Contents}
\startcontents
\printcontents{}{1}{\setcounter{tocdepth}{2}}
\newpage

\section{Proof of Proposition}
\label{sec:proof}

\subsection{\texorpdfstring{\cref{thm:multiconvs}}{}: Equivalence and Efficiency of the Multi-Conv Layer}
\label{subsec:proof1}

\begin{proposition}[Efficiency and Equivalence of Multi-Conv Layer]
% \label{thm:multiconvs}
Let $\mathcal{M}_{\text{plain}}$ denote a standard multi-scale convolutional module with kernels $\{K_r\}_{r \in \{1,3,5\}}$ of sizes $r \times r$, and let $\mathcal{M}_{\text{multi-convs}}$ denote the proposed multi-convolution layer with the same kernel weights. Given identical input features $\mathbf{Z}^t$ and identical kernel initialization, the following properties hold:
\begin{enumerate}
    \item \textbf{Functional equivalence}: For all spatial positions $(i,j)$,
    \begin{equation}
    \mathcal{M}_{\text{plain}}(\mathbf{Z}^t)[i,j] = \mathcal{M}_{\text{multi-convs}}(\mathbf{Z}^t)[i,j].
    \end{equation}
    \item \textbf{Gradient equivalence}: The gradients with respect to each kernel weight $K_r$ are identical in both modules:
    \begin{equation}
    \frac{\partial L}{\partial K_r}\bigg|_{\mathcal{M}_{\text{plain}}} = \frac{\partial L}{\partial K_r}\bigg|_{\mathcal{M}_{\text{multi-convs}}}, \quad \forall r \in \{1,3,5\}.
    \end{equation}
    \item \textbf{Computational efficiency}: The multi-conv layer reduces the computational complexity from $\mathcal{O}(N_{\text{kernels}} \cdot H_0 \cdot W_0 \cdot r^2)$ to $\mathcal{O}(H_0 \cdot W_0 \cdot r_{\max}^2)$ for both forward and backward passes, where $N_{\text{kernels}}=3$, $r_{\max}=5$, and $H_0 \times W_0$ is the spatial dimension of the feature map.
\end{enumerate}
\end{proposition}

\textit{Proof.} We present a formal derivation to demonstrate the equivalence between the forward and backward computations of the proposed multi‑convs layer and those of separate multi‑scale convolutions, taking as an example the fusion of three convolution layers $K_{1}, K_{3}, K_{5}$ with kernel sizes 1×1, 3×3, and 5×5.

\paragraph{Preliminaries and Notation}
For clarity, the derivation focuses on the spatial dimensions (height $H_0$ and width $W_0$); the conclusions extend directly to the full tensor case including batch and channel dimensions. Let $\mathbf{Z}^{t}[i,j,r]$ denote the $r \times r$ input region centered at spatial location $(i,j)$. The convolution operation is defined as:
\begin{equation}
    \mathbf{Z}'^{t} = \sum_{i=0}^{H_0-1} \sum_{j=0}^{W_0-1} K_{r} \odot \mathbf{Z}^{t}[i,j,r],
\end{equation}
where $\odot$ represents element-wise multiplication followed by summation over the kernel support. The gradient of a loss function $L$ with respect to a kernel weight $K_r$ is:
\begin{equation}
    \frac{\partial L}{\partial K_{r}} = \sum_{i=0}^{H_0-1} \sum_{j=0}^{W_0-1} \frac{\partial L}{\partial \mathbf{Z}'^{t}[i,j]} \odot \mathbf{Z}^{t}[i,j,r].
\end{equation}

\paragraph{Proof of 1: Functional Equivalence (Forward Pass)}
The output of a standard multi-scale module $\mathcal{M}_{\text{plain}}$ is the sum of three separate convolutions:
\begin{align}
\mathbf{Z}'^{t} &= \sum_{i=0}^{H_0-1} \sum_{j=0}^{W_0-1} \Bigl( K_{1} \odot \mathbf{Z}^{t}[i,j,1] + K_{3} \odot \mathbf{Z}^{t}[i,j,3] + K_{5} \odot \mathbf{Z}^{t}[i,j,5] \Bigr). \label{equ:plain}
\end{align}
A key observation is that the supports of the three kernels are nested: $K_1$ occupies only the center point, $K_3$ a $3\times3$ region, and $K_5$ a $5\times5$ region. Therefore, the summation in \eqref{equ:plain} can be rewritten by aligning and adding the kernels at their common center point (denoted by operator $\oplus$, with zero-padding for alignment):
\begin{align}
\mathbf{Z}'^{t} = \sum_{i=0}^{H_0-1} \sum_{j=0}^{W_0-1} \bigl( K_{1} \oplus K_{3} \oplus K_{5} \bigr) \odot \mathbf{Z}^{t}[i,j,5]. \label{equ:multi2}
\end{align}
Equation \eqref{equ:multi2} is precisely the forward pass of the multi‑convs layer $\mathcal{M}_{\text{multi-convs}}$, which performs a single convolution with the composite kernel $(K_{1} \oplus K_{3} \oplus K_{5})$. Since \eqref{equ:plain} and \eqref{equ:multi2} are mathematically equivalent, the outputs of the two modules are identical, proving property (1).

\paragraph{Proof of 2: Gradient Equivalence (Backward Pass)}
In $\mathcal{M}_{\text{plain}}$, gradients for the three kernels are computed independently:
\begin{align}
    \frac{\partial L}{\partial K_{1}} &= \sum_{i=0}^{H_0-1} \sum_{j=0}^{W_0-1} \frac{\partial L}{\partial \mathbf{Z}'^{t}[i,j]} \cdot \mathbf{Z}^{t}[i,j,1], \label{equ:plain_back1}\\
    \frac{\partial L}{\partial K_{3}} &= \sum_{i=0}^{H_0-1} \sum_{j=0}^{W_0-1} \frac{\partial L}{\partial \mathbf{Z}'^{t}[i,j]} \cdot \mathbf{Z}^{t}[i,j,3], \label{equ:plain_back3}\\
    \frac{\partial L}{\partial K_{5}} &= \sum_{i=0}^{H_0-1} \sum_{j=0}^{W_0-1} \frac{\partial L}{\partial \mathbf{Z}'^{t}[i,j]} \cdot \mathbf{Z}^{t}[i,j,5]. \label{equ:plain_back5}
\end{align}
Since $\mathcal{M}_{\text{multi-convs}}$ produces the same output $\mathbf{Z}'^{t}$ from the same input $\mathbf{Z}^{t}$ and kernels $\{K_r\}$, and because gradient computation for each kernel depends only on $\partial L / \partial \mathbf{Z}'^{t}$ and the corresponding input region $\mathbf{Z}^{t}[i,j,r]$, the gradient formulas remain unchanged. The multi‑convs layer computes these gradients in a fused, parallel loop:
\begin{align}
    \frac{\partial L}{\partial K_{1}},\,
    \frac{\partial L}{\partial K_{3}},\,
    \frac{\partial L}{\partial K_{5}}
    = \sum_{i=0}^{H_0-1} \sum_{j=0}^{W_0-1}
    \Bigl[ &
    \frac{\partial L}{\partial \mathbf{Z}'^{t}[i,j]} \cdot \mathbf{Z}^{t}[i,j,1], \nonumber \\
    & \frac{\partial L}{\partial \mathbf{Z}'^{t}[i,j]} \cdot \mathbf{Z}^{t}[i,j,3], \nonumber \\
    & \frac{\partial L}{\partial \mathbf{Z}'^{t}[i,j]} \cdot \mathbf{Z}^{t}[i,j,5] \Bigr], \label{equ:multi_back}
\end{align}
which is computationally consolidated but mathematically identical to evaluating \eqref{equ:plain_back1}–\eqref{equ:plain_back5} separately. This proves property (2).

\paragraph{Proof of 3: Computational Efficiency}
\label{param:computation}
We analyze the computational complexity for a standard convolution operation. Let $H_0 \times W_0$ denote the spatial resolution, and let $C_{\text{in}}$ and $C_{\text{out}}$ denote the number of input and output channels, respectively. We define $\mathcal{C}_{\text{mult-add}}$ as the cost of a multiply-accumulate operation and $\mathcal{C}_{\text{mem}}$ as the cost of a memory access.

In the standard multi-scale module $\mathcal{M}_{\text{plain}}$, $N$ separate convolutions (with kernel sizes $r \in \mathcal{R}$) are executed sequentially. The total arithmetic cost involves summing the operations for each kernel:
\begin{equation}
    \mathcal{C}_{\text{arith}}^{\text{(plain)}} = \mathcal{C}_{\text{mult-add}} \cdot H_0 W_0 C_{\text{in}} C_{\text{out}} \sum_{r \in \mathcal{R}} r^2.
\end{equation}
Assuming a standard implementation where each convolution layer triggers separate kernel launches, the memory access cost (accounting for input reads, weight reads, and output writes) is:
\begin{equation}
    \mathcal{C}_{\text{mem}}^{\text{(plain)}} = \mathcal{C}_{\text{mem}} \cdot \sum_{r \in \mathcal{R}} \bigl( H_0 W_0 C_{\text{in}} + r^2 C_{\text{in}} C_{\text{out}} + H_0 W_0 C_{\text{out}} \bigr).
\end{equation}

In the proposed $\mathcal{M}_{\text{multi-convs}}$, the kernels are fused into a single composite kernel with size $r_{\max} = \max(\mathcal{R})$. The forward pass effectively performs a single convolution:
\begin{equation}
    \mathcal{C}_{\text{arith}}^{\text{(fused)}} = \mathcal{C}_{\text{mult-add}} \cdot H_0 W_0 C_{\text{in}} C_{\text{out}} \cdot r_{\max}^2.
\end{equation}
Crucially, the fused operation significantly reduces memory traffic by reading the input feature map and writing the output feature map only once, rather than $N$ times. The memory cost becomes:
\begin{equation}
    \mathcal{C}_{\text{mem}}^{\text{(fused)}} = \mathcal{C}_{\text{mem}} \cdot \bigl( H_0 W_0 C_{\text{in}} + r_{\max}^2 C_{\text{in}} C_{\text{out}} + H_0 W_0 C_{\text{out}} \bigr).
\end{equation}

Comparing the two approaches, the multi-convs layer achieves efficiency gains in two aspects:
1. \textbf{Arithmetic Reduction:} Since $\sum_{r \in \mathcal{R}} r^2 > r_{\max}^2$ (for non-trivial sets $\mathcal{R}$ containing multiple kernels), the total FLOPs are reduced.
2. \textbf{Memory Bandwidth Optimization:} The dominant term in memory cost for large feature maps, $H_0 W_0 (C_{\text{in}} + C_{\text{out}})$, is reduced by a factor of $N$ (the number of branches), as the fused kernel eliminates redundant I/O operations associated with intermediate results in the multi-branch structure.

\subsection{\texorpdfstring{\cref{thm:loss} }{}: Loss Function}
\label{subsec:proof2}

\begin{proposition}[Adaptive Loss Weighting]
\label{thm:adaptive_loss}
Consider the composite loss function
\begin{equation}
    \mathcal{L}(\theta, \mathbf{w}) = \frac{1}{2}\bigl(1 - \sin(\theta)\bigr) \bigl( \mathbf{L} \odot \mathcal{E} \bigr) + \frac{1}{2}\bigl(1 + \sin(\theta)\bigr) \Bigl( \mathcal{E} e^{-\mathbf{w}} + \mathbf{w} \Bigr),
    \label{equ:adaptive_loss}
\end{equation}
where $\mathcal{E} = (\mathbf{\hat{X}}^{t+1} - \mathbf{X}^{t+1})^2$ denotes the squared prediction error, $\mathbf{L}$ is a fixed latitude-weight matrix with non-negative entries, $\mathbf{w} \in \mathbb{R}^{1\times 1\times C}$ represents uncertainty parameters, and $\theta \in [-\pi/2, \pi/2]$ is a learnable balancing parameter. Under gradient descent optimization with learning rate $\eta > 0$, and with initial conditions $\theta_0 = -\pi/2 + \epsilon$ (where $0 < \epsilon \ll 1$) and $\mathbf{w}_0 = \mathbf{0}$, the following properties hold:
\begin{enumerate}
    \item The parameter $\theta$ evolves monotonically from $-\pi/2$ to $\pi/2$ during training;
    \item The loss function $\mathcal{L}$ automatically transitions its emphasis from the latitude-weighted term $\mathbf{L} \odot \mathcal{E}$ to the variable-aware term $\mathcal{E} e^{-\mathbf{w}} + \mathbf{w}$ as training progresses.
\end{enumerate}
\end{proposition}

\textit{Proof.} We prove the proposition by analyzing the coupled dynamics of $\mathbf{w}$ and $\theta$ under gradient descent. For convenience, we rewrite the loss as a convex combination of two terms:
\begin{equation}
    \mathcal{L} = \underbrace{\frac{1}{2}\bigl(1 - \sin(\theta)\bigr)}_{\alpha(\theta)} \underbrace{(\mathbf{L} \odot \mathcal{E})}_{A} \;+\; \underbrace{\frac{1}{2}\bigl(1 + \sin(\theta)\bigr)}_{\beta(\theta)} \underbrace{(\mathcal{E} e^{-\mathbf{w}} + \mathbf{w})}_{B},
    \label{eq:loss_ab}
\end{equation}
with $\alpha(\theta) + \beta(\theta) = 1$. Note that $A$ is the latitude-weighted error, and $B$ is the variable-aware loss with uncertainty parameter $\mathbf{w}$.

\paragraph{Step 1: Optimal $\mathbf{w}$ for a fixed error $\mathcal{E}$}
We first analyze the behavior of term $B$ with respect to $\mathbf{w}$. The gradient of $\mathcal{L}$ w.r.t. $\mathbf{w}$ is
\begin{equation}
    \frac{\partial \mathcal{L}}{\partial \mathbf{w}} = \beta(\theta) \cdot \frac{\partial B}{\partial \mathbf{w}} = \frac{1}{2}(1 + \sin\theta) \left( -\mathcal{E} e^{-\mathbf{w}} + 1 \right).
\end{equation}
Setting the gradient to zero gives the optimal $\mathbf{w}^*$ for a given $\mathcal{E}$:
\begin{equation}
    -\mathcal{E} e^{-\mathbf{w}^*} + 1 = 0 \quad \Rightarrow \quad e^{-\mathbf{w}^*} = \frac{1}{\mathcal{E}} \quad \Rightarrow \quad \mathbf{w}^* = \ln(\mathcal{E}).
    \label{eq:w_optimal}
\end{equation}
Substituting $\mathbf{w}^*$ back into $B$ yields its minimized value:
\begin{equation}
    B^* = \mathcal{E} \cdot \frac{1}{\mathcal{E}} + \ln(\mathcal{E}) = 1 + \ln(\mathcal{E}).
    \label{eq:B_optimal}
\end{equation}
Thus, as training progresses and the error $\mathcal{E}$ decreases, $B^*$ decreases and can become negative when $\mathcal{E}$ is small, whereas $A$ remains non-negative.

\paragraph{Step 2: Dynamics of $\theta$}
The gradient of $\mathcal{L}$ with respect to $\theta$ is
\begin{equation}
    \frac{\partial \mathcal{L}}{\partial \theta} = -\frac{1}{2}\cos(\theta) A + \frac{1}{2}\cos(\theta) B = \frac{1}{2}\cos(\theta) (B - A).
    \label{eq:grad_theta}
\end{equation}
Under gradient descent with learning rate $\eta$, the update for $\theta$ is
\begin{equation}
    \theta_{k+1} = \theta_k - \eta \frac{\partial \mathcal{L}}{\partial \theta} = \theta_k + \frac{\eta}{2}\cos(\theta_k)(A - B).
    \label{eq:update_theta}
\end{equation}

\paragraph{Step 3: Monotonic increase of $\theta$}
To determine the sign of the gradient, we compare the magnitudes of $A$ and $B$. Here, we formally invoke an adiabatic approximation regarding the optimization dynamics.

\textit{Assumption A.1 (Adiabatic Dynamics).} We assume a time-scale separation where the uncertainty parameters $\mathbf{w}$ converge to their local optima $\mathbf{w}^*$ significantly faster than the evolution of the balancing parameter $\theta$. This implies that for the analysis of $\theta$, we can approximate the instantaneous value $B$ with its optimized lower bound $B^*$.

Under this assumption, we compare the asymptotic scaling of $A$ and $B^*$:

The term $A$ represents the latitude-weighted error. Since the weights $\mathbf{L}$ are normalized ($\langle \mathbf{L} \rangle = 1$) and bounded, $A$ scales \textit{linearly} with the error magnitude:
\begin{equation}
    A = \mathbf{L} \odot \mathcal{E} \approx \mathcal{O}(\langle \mathcal{E} \rangle).
\end{equation}
In contrast, as the model converges and the error $\langle \mathcal{E} \rangle$ becomes small (specifically, $\langle \mathcal{E} \rangle \ll 1$), the logarithmic term in \cref{eq:B_optimal} dominates:
\begin{equation}
    B^* \approx 1 + \ln(\langle \mathcal{E} \rangle) \to -\infty.
\end{equation}
Consequently, in the regime of low error (typical after early training stages), we strictly have $B^* < 0 < A$, which implies $B - A < 0$, (see \cref{fig:ABcompare}).

\begin{wrapfigure}{r}{0.35\textwidth}
    \centering
    \vspace{-10pt}
    \includegraphics[width=0.33\textwidth]{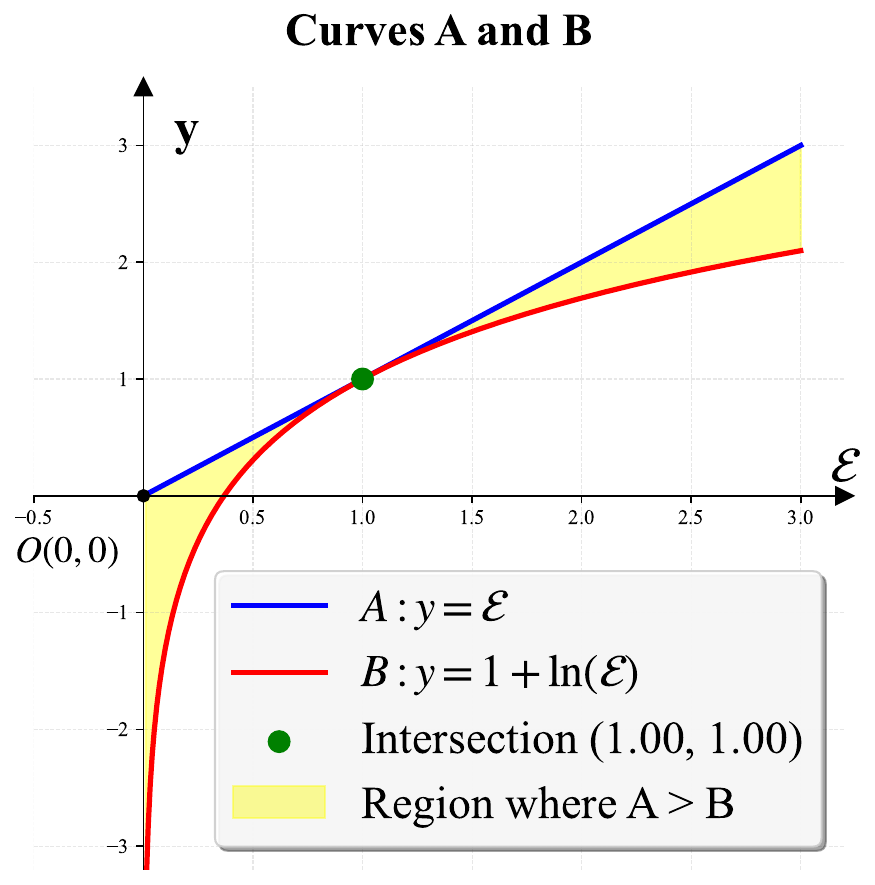}
    \caption{Asymptotic behavior of the terms $A$ (linear) and $B^*$ (logarithmic) as the mean error $\langle \mathcal{E} \rangle$ decreases. The intersection point where $A = B^*$ marks the threshold beyond which $\theta$ increases monotonically.}
    \label{fig:ABcompare}
    \vspace{-10pt}
\end{wrapfigure}

Substituting \(A - B > 0\) into the update rule \eqref{eq:update_theta}, and recalling that after the initial perturbation we have \(\cos(\theta_k) > 0\), we obtain
\begin{equation}
\theta_{k+1} = \theta_k + \underbrace{\frac{\eta}{2}\cos(\theta_k)}_{>0} \underbrace{(A - B)}_{>0}.
\end{equation}
Therefore, \(\partial \mathcal{L}/\partial \theta < 0\) for all subsequent steps, causing \(\theta\) to increase monotonically from its initial perturbed value (\(> -\pi/2\)) toward \(\pi/2\).

\paragraph{Step 4: Convergence and automatic loss transition}
As $\theta$ increases, $\sin\theta \to 1$, which implies
\begin{equation}
    \alpha(\theta) = \frac{1}{2}(1-\sin\theta) \to 0, \qquad 
    \beta(\theta) = \frac{1}{2}(1+\sin\theta) \to 1.
\end{equation}
Consequently, the weighting coefficient $\alpha(\theta)$ on the latitude-weighted term $A$ diminishes to zero, while $\beta(\theta)$ on the variable-aware term $B$ approaches one. This dynamic implements a soft curriculum schedule: the loss initially emphasizes the stable latitude-weighted error for coarse learning and automatically shifts focus to the uncertainty-aware term for fine-grained refinement as training progresses.

\vspace{0.7cm}
\paragraph{Statement.} The experimental verification of the theoretical results in \cref{subsec:proof1} (\cref{thm:multiconvs}) and \cref{subsec:proof2} (\cref{thm:loss}) is provided in \cref{subsec:add_proofresult}. The runtime comparison between the plain multi-scale module and the proposed multi-convs layer is detailed in \cref{tab:conv_performance}. Furthermore, the evolution of the balancing parameter $\theta$ during training is visualized in \cref{fig:lr_theta}.
\vspace{0.7cm}

\section{Related work}
\label{sec:related}

\subsection{Data-Driven Weather Forecasting}

Before the advent of data-driven methods, numerical weather prediction (NWP) was the predominant approach, producing forecasts by solving partial differential equations on high-resolution global grids~\cite{nwp, lynch2008origins, kalnay2002atmospheric}. While these methods deliver physically consistent and rigorously validated forecasts~\cite{molteni1996ecmwf, ritchie1995implementation}, they remain computationally prohibitive, especially at fine spatial scales.

The availability of large-scale atmospheric reanalysis datasets has since catalyzed a shift toward data-driven weather forecasting. Early exemplars such as FourCastNet~\cite{fourcastnet} and Pangu-Weather~\cite{panguweather} employed Fourier Neural Operators (FNO)~\cite{fno} and 3D Swin Transformers~\cite{liu2021swin}, respectively, to approximate atmospheric evolution. Subsequent research has largely bifurcated into two streams: (i) neural operators, including KNO~\cite{koopmanlab}, SFNO~\cite{sfno}, ClimODE~\cite{climode}, FANO~\cite{lin2026fano}, SMgNO~\cite{hu2025spherical}, and TNO~\cite{diab2025temporal}, that directly learn the temporal evolution operator; and (ii) neural network, such as FengWu~\cite{fengwu}, GraphCast~\cite{graphcast}, FuXi~\cite{fuxi}, GenCast~\cite{gencast}, Stormer~\cite{stormer}, ClimaX~\cite{climax}, OneForecast~\cite{oneforecast}, and VAMoE~\cite{vamoe}, which incorporate inductive biases tailored to atmospheric dynamics. Both categories achieve competitive forecast skill with substantially lower computational cost.

Notably, Aurora~\cite{aurora} addresses error accumulation by adopting a Low-Rank Adaptation (LoRA) finetuning strategy that assigns distinct parameters to different forecast time steps. While this approach can mitigate error drift, it introduces additional overhead from loading and saving multiple model states. In contrast, our method employs accumulative context finetuning, which shares the same parameters across all time steps and reduces error accumulation by retaining key-value pairs from previous steps, thereby avoiding unnecessary computational complexity.

\subsection{Long Context Learning}

The advent of large language models has spurred extensive research into long‑context learning, aiming to improve contextual coherence and accelerate inference over extended sequences~\cite{tikochinski2025incremental}. Existing methods can be broadly divided into two categories: (1) those that use dynamic sparse attention masks to skip redundant computations~\cite{dsflashattn, dota, ecsa, spatten}, and (2) those that prune the KV pairs adaptively based on input‑sequence characteristics~\cite{NEURIPS2023_cdaac2a0, pmlr-v202-sheng23a, ge2024model}. Both directions have successfully extended the effective context length while reducing computational overhead.

Inspired by these advances, recent work has adapted long‑context techniques to video generation, leading to three main paradigms: training‑free KV‑cache compression~\cite{yi2025deep, kim2025infinipotv}, efficient frameworks via sparse attention matrices~\cite{zhang2025spargeattention, zhang2025clusterattn}, and dual‑model architectures that decouple memory from generation~\cite{zhu2025memorize, shang2025love}. These approaches have improved both the speed and temporal consistency of long‑form video synthesis.

% Moving beyond language and video domains, this work pioneers the use of accumulative context learning with memory management in weather forecasting. By integrating autoregressive prediction with accumulative learning, our model retains information from historical time steps and uses it to inform current predictions. This design preserves long‑range temporal consistency and mitigates error accumulation, thereby extending the skillful forecast horizon.

Moving beyond language and video domains, this work introduces memory-augmented accumulative‑context learning to weather forecasting. By coupling autoregressive prediction with accumulative‑context modeling, the system retains relevant information from past time steps and leverages it to condition current forecasts. This design preserves long‑range temporal consistency and mitigates error accumulation, thereby extending the skillful forecast horizon.

% From the booming of Large Language Model(LLM), many works have been exploring the KV-Cache techniques to improve context consistency and accelerate the long-context LLM inference. The mainstream techniques can be divided into two paths, one of which is skipping computations with dynamic sparse attention masks~\cite{dsflashattn, dota, ecsa, spatten}, another is discarding KV cache based on the input sequence~\cite{NEURIPS2023_cdaac2a0, pmlr-v202-sheng23a, ge2024model}. Those two paths both successfully increase the reasoning length and reduce the computation complexity.

% Since the booming of KV cache compression in the long context inference, some recent works try to introduce the KV cache memory module to long-term video generation. There are three paradigms behind this technique, including training-free KV cache compression~\cite{yi2025deep, kim2025infinipotv}, efficient compression based on sparse attention matrix~\cite{zhang2025spargeattention, zhang2025clusterattn}, and dual-model architecture of memory and generation~\cite{zhu2025memorize}. Those methods are all make sense in accelerating video generation and increasing the temporal consistency.

% Apart from LLM and video generation, this work is the pioneer to introduce the KV cache memory module to long-term weather forecasting. By integrating the autoregressive method with memory module, this work can achieve long context consistency with low error accumulation.

\subsection{Basic Vision Models}

% Since the flourishing of deep learning, basic vision tasks have been one of the most important fields in artificial intelligence(AI). Many parts related to basic visual tasks have greatly promoted the development of deep learning, including Imagenet-K~\cite{deng2009imagenet} and AlexNet~\cite{alexnet}. From then on, many works about visional structure have been explored the vision tasks. Those methods can be classified into four categories, including Conv-based~\cite{wightman2021resnet, liu2022convnet}, Transformer-based~\cite{chen2021crossvit, liu2021swin, chu2021twins, touvron2022deit, cai2023efficientvit}, Conv-Transformer based~\cite{chen2021crossvit, hatamizadeh2023fastervit, li2022next}, and Mamba-based~\cite{liu2024vmamba, yu2025mambaout, efficientvmamba, hatamizadeh2025mambavision}. Those works all achieves great performance in basic vision tasks and promotes the development of deep learning in different stages.

Since the rise of deep learning, fundamental vision tasks have been a central and driving domain in artificial intelligence (AI). Landmark contributions, such as the ImageNet dataset~\cite{deng2009imagenet} and the AlexNet architecture~\cite{alexnet}, have profoundly shaped the field's evolution. Subsequent research has explored a wide variety of visual architectures, which can be broadly categorized into four paradigms: convolutional (Conv) networks~\cite{wightman2021resnet, liu2022convnet}, vision transformers (ViTs)~\cite{chen2021crossvit, liu2021swin, chu2021twins, touvron2022deit, cai2023efficientvit}, hybrid Conv‑Transformer designs~\cite{chen2021crossvit, hatamizadeh2023fastervit, li2022next}, and more recent Mamba‑based models~\cite{liu2024vmamba, yu2025mambaout, efficientvmamba, hatamizadeh2025mambavision}. Each of these lines of work has achieved strong performance on core vision benchmarks and has collectively advanced deep learning at different stages of its development.

% Same as the previous categories, our work is one of the conv-transformer structures. Our work is also a model that focuses on multi-scale module and efficient structure. Different from the previous work that adopts the plain multi-scale module directly, our work proposes a novel layer that extracts multi-scale information with one convolution layer. To achieve this effect, this work redesigns the forward path and rewrites the backward path of the traditional convolution layer. Compared to the previous plain multi-scale module, this work proposes a novel multi-convs layer, which can accelerates the forward and backward paths with 5.69 times.

Following the hybrid Conv-Transformer paradigm, our model likewise prioritizes multi-scale representation and computational efficiency. Unlike previous works that directly employ plain multi-scale modules, we propose a novel layer capable of capturing multi-scale information through a single convolutional operation. To realize this capability, we redesign both the forward pass and the backward propagation of the standard convolution layer. Compared with conventional multi-scale modules, our proposed multi-convs layer accelerates the forward and backward computations by a factor of 5.69×.

\textit{Comparison with structural re‑parameterization methods.} Prior work on structural re‑parameterization has focused on merging multi‑branch modules into a single operation during inference~\cite{ding2021repvgg}. Although later extensions claim to support training~\cite{huang2022dyrep}, a fundamental limitation remains: the re‑parameterized module effectively reduces to a single operation during backpropagation, losing its original multi‑scale representation. As a result, such re‑parameterized versions are not functionally equivalent to true multi‑branch designs. In contrast, our multi‑convs layer implements a CUDA‑level multi‑branch gradient‑update mechanism that maintains independent gradient flows for each kernel throughout training. This ensures exact equivalence, in both forward and backward passes, between the conventional multi‑scale module and our proposed layer, while preserving the representational benefits of multi‑scale feature extraction.

\begin{algorithm}[t]
\caption{The Proposed Model in Weather Forecasting}
\label{alg:stcast_global}
\renewcommand{\algorithmicrequire}{\textbf{Input:}}
\renewcommand{\algorithmicensure}{\textbf{Output:}}

\begin{algorithmic}[1]
    
    \REQUIRE Atmospheric variables $\mathbf{X}^t$ at timestep $t$
    \ENSURE Forecasted atmospheric variables $\mathbf{X}^{t+1}$ at timestep $t+1$
    
    \STATE Apply high-stride convolution for patch embedding: $\mathbf{X}^t = \text{Conv}_{p \times p}(\mathbf{X}^t)$
    \STATE Add positional embedding: $\mathbf{X}^t = \mathbf{X}^t + \mathbf{P}$
    % \STATE Project to latent space via MLP: $\mathbf{X}^t = \text{MLP}(\mathbf{X}^t)$
    \STATE 

    \STATE \textbf{Pruning}
    \FOR{each $k \in range(4)$ }
        \STATE $\mathbf{X}^t_{\text{down}}= \mathbf{X}^t_{\text{res}} = \text{DownSample}(\mathbf{X}^t)$
        \STATE $\mathbf{X}^t = \text{Cross-Attention}(\mathbf{X}^t_{\text{down}}, \mathbf{X}^t)$
        \STATE $\mathbf{X}^t = \text{Self-Attention}(\mathbf{X}^t)$
    \ENDFOR
    \STATE 
    
    \STATE \textbf{Processor}

    \FOR{each $i \in $ [blocks' number]}
        \IF {$i\%2 == 0$}
            % \STATE $j = i // 2$
            \STATE Windows Size = (4, 4) \textit{or} (2, 8) \textit{or} (8, 2)
            % \IF {$j\%3 == 0$}
            %     \STATE Windows Size = (4, 4)
            % \ELSIF{$j\%3 == 1$}
            %     \STATE Windows Size = (2, 8)
            % \ELSE
            %     \STATE Windows Size = (8, 2)
            % \ENDIF
            \STATE Attention = Windows-EMFormer(Windows Size)
        \ELSE
            \STATE Attention = EMFormer()
        \ENDIF
        
        \STATE Apply multi-head attention with residual connection: $\mathbf{A}^t = \text{LN}(\text{Attention}(\mathbf{X}^t)) + \mathbf{X}^t$
            \STATE Apply MLP with residual connection: $\mathbf{X}^{t+1} = \text{MLP}(\mathbf{A}^t) + \mathbf{A}^t$
            
    \ENDFOR
    \STATE
    
    \STATE \textbf{Recovering}
    \FOR{each $k \in range(4)$ }
        \STATE $\mathbf{X}^{t+1}_{up} = \text{UpSample}(\mathbf{X}^{t+1})+\mathbf{X}^t_{\text{res}}$
        \STATE $\mathbf{X}^{t+1} = \text{Cross-Attention}(\mathbf{X}^{t+1}_{up}, \mathbf{X}^{t+1})$
        \STATE $\mathbf{X}^{t+1} = \text{Self-Attention}(\mathbf{X}^{t+1})$
    \ENDFOR
    \STATE 
    
    \STATE Reconstruct atmospheric variables to longitude-latitude grids with prediction head: $\mathbf{X}^{t+1} = \text{Linear}(\mathbf{X}^{t+1})$
\end{algorithmic}
\end{algorithm}

\begin{algorithm}[t]
\caption{Multi-Convs Layer}
\label{alg:multi_convs_layer}
\renewcommand{\algorithmicrequire}{}
\begin{algorithmic}[1]
\REQUIRE $\mathbf{input}\in\mathbb{R}^{B\times C_{in}\times H\times W}, \mathbf{weight1}\in\mathbb{R}^{C_{out}\times C_{in}\times 1\times 1}, \mathbf{weight3}\in\mathbb{R}^{C_{out}\times C_{in}\times 3\times 3}, \mathbf{weight5}\in\mathbb{R}^{C_{out}\times C_{in}\times 5\times 5} $
% \ENSURE Forecasted atmospheric variables $\mathbf{X}^{t+1}$ at timestep $t+1$
% \STATE \textbf{TripleConvForward}
% \STATE Extract dimensions: $B, C_{in}, H, W \gets \text{shape}(\mathbf{input})$, $C_{out} \gets \text{shape}(\mathbf{weight1})[0]$
% \STATE Initialize $\mathbf{output} \gets$ zeros($B, C_{out}, H, W$)
% \STATE \textbf{Return} $\mathbf{output}$
% \STATE
% \STATE \textbf{TripleConvBackward}
% \STATE $\mathbf{grad\_output}, \mathbf{input}, \mathbf{weights}$
% \STATE \textbf{Parallel weight gradient computation:}
% \STATE Launch kernel for $\mathbf{grad\_weight1}, \mathbf{grad\_weight3}, \mathbf{grad\_weight5}$
% \STATE \textbf{Parallel input gradient computation:}
% \STATE Launch kernel for $\mathbf{grad\_input}$
% \STATE Synchronize CUDA device
% \STATE \textbf{Return} gradients
% \STATE
\STATE \textbf{ForwardKernel}
\STATE Each thread computes one output element $(n, c_{out}, h, w)$
\STATE Initialize $\text{result} \gets 0$
\FOR{$c_{in} = 0$ to $C_{in}-1$}
    \FOR{each spatial offset $(kh, kw)$}
        % \STATE Compute input coordinates
        % \IF{within bounds}
        \STATE Compute combined weight = $\mathbf{weight1} + \mathbf{weight3} + \mathbf{weight5}$
        \STATE $\text{result} \gets \text{result} + \text{input} \times \text{combined\_weight}$
        % \ENDIF
    \ENDFOR
\ENDFOR
\STATE
\STATE \textbf{WeightGradKernel}
\STATE Each thread processes one output gradient element
\STATE Load gradient value $\delta$
\FOR{$c_{in} = 0$ to $C_{in}-1$}
    \FOR{each spatial offset $(kh, kw)$}
        % \IF{within bounds}
        % \STATE Load input value $x$
        \STATE Compute contribution $g \gets \delta \times \text{input}$
        \STATE \textbf{Parallel accumulation:} Add $g$ to $\mathbf{grad\_weight1}$, $\mathbf{grad\_weight3}$, $\mathbf{grad\_weight5}$ using atomic operations
        % \ENDIF
    \ENDFOR
\ENDFOR
\STATE
\STATE \textbf{InputGradKernel}
\STATE Each thread computes one input gradient element $(n, c_{in}, h, w)$
\STATE Initialize $\text{grad} \gets 0$
\FOR{$c_{out} = 0$ to $C_{out}-1$}
    \FOR{each spatial offset $(kh, kw)$}
        % \STATE Compute output coordinates
        % \IF{within bounds}
        \STATE Load gradient value $\delta$
        \STATE Compute combined weight from three scales
        \STATE $\text{grad} \gets \text{grad} + \delta \times \text{combined\_weight}$
        % \ENDIF
    \ENDFOR
\ENDFOR
\end{algorithmic}
\end{algorithm}

\begin{algorithm}[t]
\caption{KV Values Pruning in Accumulative Context Finetuning }
\label{alg:h2o-simple}
\begin{algorithmic}[1]

\STATE The maximum number of KV pairs in cache: $N$
\STATE Project input to key-value pairs: $Q, K_{\text{new}}, V_{\text{new}} \gets \text{LinearProj}(x)$
\STATE Length of K, V values: $L \gets \text{len}(K_{\text{new}})$ 
\STATE
\IF{$(K \neq \text{None}) \And (V \neq \text{None}) $}
    \STATE $K_{\text{cache}}, V_{\text{cache}} \gets K, V$
    \STATE $T \gets \text{len}(K_{\text{cache}})$
    
    \IF{$T \geq N*L$}    % \COMMENT{Cache pruning condition}
        \STATE Get historical importance: $\text{scores} \gets \text{importance\_scores}$ 
        \STATE Most important tokens: $\text{Index}_{topk} \gets \text{TopK}(\text{scores}[0 : T-L], (N-2)*L)$ 
        \STATE Most recent tokens: $\text{Index}_{recent} \gets [T-L : T]$ 
        \STATE Combine the important and recent tokens: $\text{Index}_{keep} \gets \text{Unique}(\text{Index}_{topk} \cup \text{Index}_{recent})$
        \STATE Preserve order: $\text{Index}_{keep} \gets \text{Sort}(\text{Index}_{keep})$ 
        \STATE $K_{\text{cache}} \gets K_{\text{cache}}[:,\text{Index}_{keep}]$
        \STATE $V_{\text{cache}} \gets V_{\text{cache}}[:,\text{Index}_{keep}]$
    \ENDIF
    
    \STATE $K \gets [K_{\text{cache}}, K_{\text{new}}]$
    \STATE $V \gets [V_{\text{cache}}, V_{\text{new}}]$
\ELSE
    \STATE $K \gets K_{\text{new}}$
    \STATE $V \gets V_{\text{new}}$
\ENDIF
\STATE

\STATE $\text{output} \gets \text{Attention}(Q, K, V)$

\STATE $\text{attn} \gets \text{Softmax}(QK^\top / \sqrt{d})$
\STATE Average across heads and batch: $\text{imp} \gets \text{Mean}(\text{attn})$ 
\STATE $\text{old} \gets \text{importance\_scores}$
\STATE $\text{importance\_scores} \gets 0.9 \times \text{old} + 0.1 \times \text{imp}$

\STATE $\text{Output: }$ $(K, V)$

\end{algorithmic}
\end{algorithm}

\section{Detailed Structure}
\label{sec:add_structure}

As shown in \cref{alg:stcast_global}, the global weather forecasting framework comprises three core components: Pruning, Processor, and Recovering, and two supplementary settings: patch embedding and prediction head. In this work, the Pruning, EMFormer, and Recovering implementations follow Flash-Attention~\cite{flashattention}. For the forecasting task, the AI model $\Phi$ predicts future atmospheric states $\mathbf{X}^{t+1}$ from historical fields $\mathbf{X}^{t}$ as $\mathbf{X}^{t+1}=\Phi(\mathbf{X}^{t})$. Detailed configurations for all three elements are provided below.

\subsection{Pruning}

% In the initial stage, the atmospheric variables across pressure levels are organized into a 3D tensor \(\mathbf{X}^{t}\in \mathbb{R}^{H\times W\times N}\), where $H$ and $W$ denote the global grid height and width, respectively, and $N$ is the number of variables. The original variables are encoded into 2D patch-grid with \(\mathbf{X}^{t}\in \mathbb{R}^{HW\times C}\), where \(C\) is the embedding dimension.

% To reduce the computation cost, this work applies a series of downsampling operations to the input features with cross-attention and self-attention. At first, the input feature is reshaped to \(\mathbf{X}^{t}\in \mathbb{R}^{\frac{1}{2}HW\times 2C}\), and projected to reduce the embedding dimension with Linear layer. The output feature in one downsampling operation is \(\mathbf{X}^{t}_{down}\in \mathbb{R}^{\frac{1}{2}HW\times C}\).

% To supply the lost information in the mentioned doansampling, this work introduces a cross-attention module to fuse the downsampled feature and the original feature. The process is formulated as:
% \begin{equation}
%     \mathbf{X}^t = \text{Cross-Attention}(\mathbf{X}^t_{down}, \mathbf{X}^t).
% \end{equation}
% To further enhance the correlation of the downsampled feature, a series of self-attention modules are introduced to it. The process is \(\mathbf{X}^t = \text{Self-Attention}(\mathbf{X}^t)\).

At the encoding stage, atmospheric variables across pressure levels are structured as a 3D tensor \(\mathbf{X}^{t} \in \mathbb{R}^{H \times W \times N}\), where \(H\) and \(W\) denote the spatial dimensions of the global grid, and \(N\) is the number of variables. These variables are first embedded into a 2D patch representation \(\mathbf{X}^{t} \in \mathbb{R}^{HW \times C}\), with \(C\) being the embedding dimension.

% To lower computational overhead, a series of downsampling steps is applied, each combining cross‑attention and self‑attention. First, the input feature is reshaped to \(\mathbf{X}^{t} \in \mathbb{R}^{\frac{1}{2}HW \times 2C}\) and then projected via a linear layer to reduce the channel dimension, yielding a downsampled feature \(\mathbf{X}^{t}_{\text{down}} \in \mathbb{R}^{\frac{1}{2}HW \times C}\). At the same time, \(\mathbf{X}^{t}_{\text{down}}\) also serves as the residual feature \(\mathbf{X}^{t}_{\text{res}}\), that is added to the upsampled features in the recovering module with a residual connection.

% To reduce computational cost, a series of downsampling operations are applied, each integrating cross‑attention and self‑attention. First, the input feature is reshaped to \(\mathbf{X}^{t} \in \mathbb{R}^{\frac{1}{2}HW \times 2C}\) and then projected through a linear layer, compressing the channel dimension to produce a downsampled feature \(\mathbf{X}^{t}_{\text{down}} \in \mathbb{R}^{\frac{1}{2}HW \times C}\). This downsampled feature is also retained as a residual signal \(\mathbf{X}^{t}_{\text{res}}\), which is later added to the corresponding upsampled features in the recovering module via a residual connection.

To reduce computational cost, a series of downsampling operations are applied, each integrating cross‑attention and self‑attention. First, the input feature is reshaped to \(\mathbf{X}^{t} \in \mathbb{R}^{\frac{1}{2}HW \times 2C}\) and then projected through a linear layer parameterized by a weight matrix \(\mathbf{W} \in \mathbb{R}^{2C \times C}\), compressing the channel dimension to produce a downsampled feature \(\mathbf{X}^{t}_{\text{down}} \in \mathbb{R}^{\frac{1}{2}HW \times C}\). This linear transformation effectively halves the channel dimension while preserving essential spatial–spectral information. The downsampled feature is also retained as a residual signal \(\mathbf{X}^{t}_{\text{res}}\), which is later added to the corresponding upsampled features in the recovering module via a residual connection.

To preserve information that may be lost during downsampling, a cross‑attention module is introduced to fuse the downsampled feature with the original input:
\begin{equation}
    \mathbf{X}^t = \text{Cross‑Attention}(Q=\mathbf{X}^t_{\text{down}}, K=\mathbf{X}^t, V=\mathbf{X}^t).
\end{equation}
Subsequently, multiple self‑attention blocks are applied to the fused feature to strengthen its internal correlations:
\begin{equation}
    \mathbf{X}^t = \text{Self‑Attention}(\mathbf{X}^t).
\end{equation}

\subsection{EMFormer}

% The medium part of this work consists of a series of EHTransformer blocks, each comprising multi-head attention, a MLP module, layer normalization, and residual connections. The operations within a single block can be formally expressed as:
% \begin{align}
%     \mathbf{A}^{t}&= \text{LN}(\text{Attention}(\mathbf{X}^{t}))+\mathbf{X}^{t},\\
%     \mathbf{X}^{t+1}&= \text{LN}(\text{TMoE}(\mathbf{A}^{t}))+\mathbf{A}^{t},
% \end{align}
% Where $\text{Attention}$ denotes the efficient multi-scale attention mechanism, $\text{LN}$ represents layer normalization, and $\text{MLP}$ refers to the Multi-Layer Perceptron module.

The core component of the architecture consists of a sequence of EMFormer blocks. Each block contains a multi‑head attention module, an MLP (Multi‑Layer Perceptron) module, layer normalization, and residual skip connections. The operations within one block are formulated as:
\begin{align}
    \mathbf{A}^{t} &= \text{LN}(\text{EMFormer}/\text{Windows-EMFormer}(\mathbf{X}^{t})) + \mathbf{X}^{t},\\
    \mathbf{X}^{t+1} &= \text{LN}(\text{MLP}(\mathbf{A}^{t})) + \mathbf{A}^{t},
\end{align}
where \(\text{EMformer}\) denotes the efficient multi‑scale attention mechanism with the multi-convs layer, \(\text{LN}\) is layer normalization, and \(\text{MLP}\) refers to the feed‑forward module. The pseudo-code of multi-convs layer is provided in \cref{alg:multi_convs_layer}.

% Following the design principles of FlashAttention ~\cite{flashattention} and VA-MoE ~\cite{vamoe}, we adopt an alternating strategy that combines window-based attention and global self-attention. This hybrid approach enables the model to effectively capture both local and global dependencies in the input distribution.

% Drawing upon the designs of FlashAttention~\cite{flashattention} and VA‑MoE~\cite{vamoe}
% Considering the various sizes of different atmospheric patterns, we employ a hybrid attention strategy that alternates between window‑based attention and global self‑attention. This design balances local feature interaction with long‑range dependency modeling, enabling the network to capture both fine‑grained patterns and broader contextual relationships in the atmospheric data.
To accommodate the multi‑scale nature of atmospheric structures, we adopt a hybrid attention mechanism that interleaves window‑based attention with global self‑attention. This design effectively balances local feature extraction with long‑range dependency modeling, allowing the network to capture both fine‑grained details and broader contextual relationships within atmospheric data.

\subsection{Recovering}

% The recovering module is the reverse process of the pruning module. Thus, we also need a series of upsampling operations to recover the hidden embedding to the original size. At each step, the hidden feature \(\mathbf{X}^{t+1} \in \mathbb{R}^{\frac{1}{2}HW \times C}\) is reshaped to \(\mathbf{X}^{t+1}_{\text{up}} \in \mathbb{R}^{HW \times \frac{1}{2}C}\). And then the upsampled feature is projected via a linear layer to increase the channel dimension, yielding a upsampled feature \(\mathbf{X}^{t}_{\text{up}} \in \mathbb{R}^{HW \times C}\). To supply the low-level feature from the first few layers, there is a residual connection between the pruning and recovering modules. The process is formulated as: $\mathbf{X}^{t+1}_{up} = \text{UpSample}(\mathbf{X}^{t+1})+\mathbf{X}^t_{\text{res}}$.

% Same as the pruning module, there are also a cross-attention and a series of self-attention in the recovering stage. The process is formulated as:
% \begin{equation}
%     \mathbf{X}^{t+1} = \text{Cross‑Attention}(\mathbf{X}^{t+1}_{\text{up}}, \mathbf{X}^{t+1}),
% \end{equation}
% \begin{equation}
%     \mathbf{X}^{t+1} = \text{Self‑Attention}(\mathbf{X}^{t+1}).
% \end{equation}

% In the last, a prediction head is introduced to project the feature from latent space to the original global range. The prediction head is a linear layer.

The recovering module mirrors the pruning module in reverse order, employing a series of upsampling stages to restore the hidden representation to its original spatial size. At each step, the hidden feature \(\mathbf{X}^{t+1} \in \mathbb{R}^{\frac{1}{2}HW \times C}\) is reshaped to \(\mathbf{X}^{t+1}_{\text{up}} \in \mathbb{R}^{HW \times \frac{1}{2}C}\) and then passed through a linear projection to expand the channel dimension, yielding \(\mathbf{X}^{t+1}_{\text{up}} \in \mathbb{R}^{HW \times C}\). To retain fine‑grained details from earlier layers, a residual connection links the corresponding stages of the pruning and recovering branches:
\begin{equation}
\mathbf{X}^{t+1}_{\text{up}} = \text{UpSample}(\mathbf{X}^{t+1}) + \mathbf{X}^t_{\text{res}}.
\end{equation}

As in the pruning stage, cross‑attention and multiple self‑attention blocks are applied to refine the upsampled feature:
\begin{align}
    \mathbf{X}^{t+1} &= \text{Cross‑Attention}(Q=\mathbf{X}^{t+1}_{\text{up}}, K=\mathbf{X}^{t+1}, V=\mathbf{X}^{t+1}), \\
    \mathbf{X}^{t+1} &= \text{Self‑Attention}(\mathbf{X}^{t+1}).
\end{align}

Finally, a prediction head, implemented as a linear layer, projects the latent representation back to the original variable space to produce the forecast output.

\begin{table*}[t]
  \centering
  \caption{A summary of atmospheric variables. The 13 levels are 50, 100, 150, 200, 250, 300, 400, 500, 600, 700, 850, 925, 1000 hPa. `Single' denotes the variables under earth's surface. }
  \small
  % \footnotesize
  % \resizebox{\linewidth}{!}{
    \begin{tabular}{clcc}
    \toprule
    \textit{Name} & \multicolumn{1}{c}{Description} & Levels  & Time   \\
    \midrule

    Z & Geopotential & 13   & 1979-2020  \\
    Q & Specific humidity & 13   & 1979-2020 \\
    U & x-direction wind & 13   & 1979-2020 \\
    V & y-direction wind & 13   & 1979-2020 \\
    T & Temperature & 13   & 1979-2020 \\
    
    \midrule
    t2m & Temperature at 2m height & Single   & 1979-2020 \\
    u10 & x-direction wind at 10m height & Single   & 1979-2020 \\
    v10 & y-direction wind at 10m height & Single   & 1979-2020 \\
    msl & Mean sea-level pressure & Single   & 1979-2020 \\
    sp & Surface pressure & Single   & 1979-2020 \\
    
    \bottomrule
    \end{tabular}
    % }%
  \label{tab:vars}
\end{table*}

\begin{table}[t]
\centering
\small
\caption{Implementation details on weather forecasting.}
\begin{tabular}{l|ll|ll}
\toprule
\textbf{Category} & \textbf{Parameter} & \textbf{Value} & \textbf{Parameter} & \textbf{Value} \\
\midrule
\multirow{3}{*}{Model Architecture} & Input Size & $128\times 256$ / $721\times 1440$ & Input Channels & 70 \\
 & Output Channels & 70  & Number of Blocks & 24     \\

 & Dimension & 768 & Patch Size & 2/4 \\
\midrule
\multirow{3}{*}{Training Configuration} & Optimizer & AdamW &  LR Scheduler & CosineAnnealingLR \\
 & Initial LR & $2 \times 10^{-4}$  & Minimum LR & $1 \times 10^{-7}$ \\
 & Maximum Epochs & 100 & Loss Function & Hybrid Loss \\
\midrule
\multirow{3}{*}{Finetuning Configuration} & Optimizer & AdamW & LR Scheduler & CosineAnnealingLR \\
 & Initial LR & $5 \times 10^{-5}$  & Minimum LR & $1 \times 10^{-7}$ \\
 & Maximum Epochs & 50 & Loss Function & Hybrid Loss \\
\midrule

\multirow{3}{*}{Data Settings} & Input Steps & 1  & Training Output Steps & 1 \\
 & Testing Output Steps & 40 & Training Period & [1979, 2020] \\
 & Test Period & [2021, 2021] & Grid Resolution & 1.4\degree/0.25\degree \\
\midrule
\multirow{3}{*}{Experimental Setup} & Batch Size & 1 & Global Batch Size & 16 \\
 & \#Data Workers & 20 & Mixed Precision & Enabled \\
 & World Size & 16  &  &\\
\bottomrule
\end{tabular}
\label{tab:params}
\end{table}

\begin{table}[t]
\centering
\caption{Implementation details for ImageNet-1K dataset.}
\label{tab:implementation_details}
\small % 适应小字体
\begin{tabular}{lr|lr}
\toprule
\textbf{Configuration} & \textbf{Value} & \textbf{Configuration} & \textbf{Value} \\
\midrule
\rowcolor{lightgray}
\multicolumn{4}{c}{\textbf{Data Configuration}} \\
Dataset & ImageNet-1K & Batch size (per GPU) & 128 \\
Number of workers & 8 & Image size schedule & 128, 160, 192, 224  \\
Training interpolation & Random & Validation interpolation & Bicubic \\
Validation crop ratio & 1.0 & & \\
\addlinespace

\rowcolor{lightgray}
\multicolumn{4}{c}{\textbf{Data Augmentation}} \\
RandAugment & $N=2$, $M=5$ & Random erase & $p=0.2$ \\
Mixup & $\alpha=0.2$, probability=1.0 & CutMix & $\alpha=0.2$, probability=1.0 \\
Label smoothing & 0.1 &  BCE loss & Enabled \\
\addlinespace

\rowcolor{lightgray}
\multicolumn{4}{c}{\textbf{Training Schedule}} \\
Total epochs & 300 & Base learning rate & 1.5e-4 \\
Warmup epochs & 20 &  Warmup learning rate & 0.0 \\
Learning rate schedule & Cosine decay & Optimizer & AdamW \\
Optimizer parameters & $\beta_1=0.9$, $\beta_2=0.999$, $\epsilon=1e-8$ &  Weight decay & 0.1 \\
Weight decay exclusion & Norm layers, bias terms & Gradient clipping & 2.0 \\
EMA decay & 0.9998  & & \\
\addlinespace

\rowcolor{lightgray}
\multicolumn{4}{c}{\textbf{Model Configuration}} \\
Input resolution & 224 &  Drop path rate & 0.1 (linear decay)  \\
Stochastic depth & Enabled &  Dropout rate & 0.0 \\
\addlinespace

\rowcolor{lightgray}
\multicolumn{4}{c}{\textbf{Additional Techniques}} \\
BatchNorm reset & Enabled &  Reset batch size & 16,000 samples \\
MESA threshold & 0.25 & Reset batch size per iteration & 100   \\
MESA ratio & 2.0   &  &\\

\bottomrule
\end{tabular}
\end{table}

\begin{table*}[t]
\centering
\addtolength{\tabcolsep}{-2pt}
\caption{Architecture configurations of EMFormer backbone models. }
% \resizebox{1.0\linewidth}{!}{
\small
\begin{tabular}{c|c|c|c|c}
\toprule
 & \begin{tabular}[c]{@{}c@{}}Output Size \\ (Downs. Rate)\end{tabular} & EMFormer-T  & EMFormer-S & EMFormer-B  \\
\midrule

\multirow{2}{*}{Stem} & \multirow{2}{*}{\begin{tabular}[c]{@{}c@{}}112$\times$112\\ ($\times\frac{1}{2}$)\end{tabular}} & $\text{Conv-BN-Act}, \text{C:32, S:2}$   & $\text{Conv-BN-Act}, \text{C:32, S:2}$  &$\text{Conv-BN-Act}, \text{C:32, S:2}$    \\

% \cline{3-5}
& & $(\text{ResBlock}, \text{C:32})$ $\times$ 2   & $(\text{ResBlock}, \text{C:32})$ $\times$ 2    & $(\text{ResBlock}, \text{C:32})$ $\times$ 2    \\

\midrule
\multirow{2}{*}{Stage 1} & \multirow{2}{*}{\begin{tabular}[c]{@{}c@{}}56$\times$56\\ ($\times\frac{1}{4}$)\end{tabular}} & Pruning, C:64, S:2  & Pruning, C:64, S:2  & Pruning, C:64, S:2  \\

% \cline{3-5}
& & $(\text{ResBlock}, \text{C:64})$ $\times$ 2   & $(\text{ResBlock}, \text{C:64})$ $\times$ 2    & $(\text{ResBlock}, \text{C:64})$ $\times$ 6   \\

\midrule
\multirow{2}{*}{Stage 2} & \multirow{2}{*}{\begin{tabular}[c]{@{}c@{}}28$\times$28\\ ($\times\frac{1}{8}$)\end{tabular}} & Pruning, C:128, S:2  & Pruning, C:128, S:2  & Pruning, C:128, S:2  \\

% \cline{3-5}
& & $(\text{ResBlock}, \text{C:128})$ $\times$ 2   & $(\text{ResBlock}, \text{C:128})$ $\times$ 4    & $(\text{ResBlock}, \text{C:128})$ $\times$ 6   \\

\midrule
\multirow{3}{*}{Stage 3} & \multirow{3}{*}{\begin{tabular}[c]{@{}c@{}}14$\times$14\\ ($\times\frac{1}{16}$)\end{tabular}} & Pruning, C:256, S:2 & Pruning, C:256, S:2  & Pruning, C:256, S:2  \\

% \cline{3-5}
& & $(\text{EMFormer}, \text{C:256})$ $\times$ 2  & $(\text{EMFormer}, \text{C:256})$ $\times$ 5    & $(\text{EMFormer}, \text{C:256})$ $\times$ 6  \\

& & $(\text{Windows-EMFormer}, \text{C:256})$ $\times$ 2  & $(\text{Windows-EMFormer}, \text{C:256})$ $\times$ 5    & $(\text{Windows-EMFormer}, \text{C:256})$ $\times$ 6  \\

\midrule
\multirow{3}{*}{Stage 4} & \multirow{3}{*}{\begin{tabular}[c]{@{}c@{}}7$\times$7\\ ($\times\frac{1}{32}$)\end{tabular}} & Pruning, C:512, S:2  & Pruning, C:512, S:2 & Pruning, C:512, S:2   \\

% \cline{3-5}
& & $(\text{EMFormer}, \text{C:512})$ $\times$ 1   & $(\text{EMFormer}, \text{C:512})$ $\times$ 2    & $(\text{EMFormer}, \text{C:512})$ $\times$ 6  \\

& & & $(\text{Windows-EMFormer}, \text{C:512})$ $\times$ 2    & $(\text{Windows-EMFormer}, \text{C:512})$ $\times$ 6  \\

\bottomrule
\end{tabular}
% }
% \normalsize

\label{tab:emformer-arch}
\end{table*}

\section{Dataset Details and Implementation Details}
\label{sec:add_dataset}

\subsection{Dataset Details}

We evaluate the proposed model using three benchmark datasets spanning meteorological and vision tasks. For atmospheric field, this work employs the ERA5 global reanalysis dataset~\cite{era5}. For vision benchmarks, we adopt ImageNet‑1K~\cite{deng2009imagenet} for classification and ADE20K~\cite{ade20k} for segmentation, which allow us to assess the architecture’s generalization beyond atmospheric modeling.

% In this work, we conduct experiments on a popular weather dataset, \textit{i.e.}, ERA5\footnote{\url{https://cds.climate.copernicus.eu/}}~\cite{era5}, provided by the ECMWF~\cite{molteni1996ecmwf}. As shown in \cref{tab:vars}, ERA5 dataset is a reanalysis atmospheric dataset, consisting of the atmospheric variables from 1979 to the present day with a 0.25\degree spatial resolution with $721 \times 1440$. The atmospheric variables include 5 upper-air variables (Z, Q, U, V, T) on 13 levels and 5 surface variables (T2M, U10, U10, MSL, SP). The model is trained on 70 variables ($5\times 13 + 5$) with 40 years atmospheric dataset from 1979 to 2020 with 0.25\degree spatial resolution. The model is tested on 1-year dataset in 2021 wiht the same resolution.

\textbf{ERA5.} In this work, experiments are conducted using the widely employed ERA5 reanalysis dataset\footnote{\url{https://cds.climate.copernicus.eu/}}~\cite{era5}, produced by the European Centre for Medium-Range Weather Forecasts (ECMWF). As detailed in \cref{tab:vars}, ERA5 provides a comprehensive atmospheric record from 1979 onward at a spatial resolution of 0.25\degree (global grid dimensions of \(721 \times 1440\)). The selected variables comprise 5 upper‑air quantities, geopotential (Z), specific humidity (Q), zonal wind (U), meridional wind (V), and temperature (T), across 13 pressure levels, together with 5 surface variables: 2‑m temperature (T2M), 10‑m zonal wind (U10), 10‑m meridional wind (V10), mean sea‑level pressure (MSL), and surface pressure (SP). This results in a total of 70 input variables (\(5 \times 13 + 5\)). The model is trained on data spanning 1979–2020 and evaluated on the held‑out year 2021, both at the native 0.25\degree resolution.

\textbf{Imagenet-1K}~\cite{deng2009imagenet} consists of approximately 1.28 million high-resolution training images and 50,000 validation images, manually annotated into 1,000 distinct object categories. 

\textbf{ADE20K}~\cite{ade20k} is a challenging task for semantic segmentation, which contains 150 categories, including 25574 natural images for training and 2000 images for validation.

\subsection{Data Processing}
\label{subsec:process}

% To address disparities among variables, all model inputs are normalized to ensure consistency. Using the training dataset spanning 1979–2019, we compute the mean and standard deviation for each variable. Normalization is then performed by subtracting the respective mean and dividing by the corresponding standard deviation.

To account for differences in scale and distribution across variables, all model inputs are standardized. Using the training data from 1979–2019, we compute the mean and standard deviation for each variable. Each input is then normalized by subtracting its mean and dividing by its standard deviation.

\subsection{Implementation Details}

\textbf{Weather forecasting.}

% The main structure of this work follows the backbone of Flash Attention ~\cite{flashattention}. In the global forecasting stage, we train the whole model. While in the regional forecasting stage, we only need to train the Spatial-Aligned Attention(SAA) module and freeze the main structure. The detailed training parameters are provided in \cref{tab:params}.

Our architecture builds upon the Flash Attention backbone~\cite{flashattention}. For weather forecasting, the entire model is trained end‑to‑end. Detailed training hyperparameters are provided in \cref{tab:params}.

\textbf{Image classification.}

% Our architecture builds upon the Flash Attention backbone~\cite{flashattention}. For image classification on Imagenet-1K, the entire model is trained end‑to‑end. Detailed training hyperparameters are provided in \cref{tab:implementation_details}. And the backbones of EMFormer-T, EMFormer-S, and EMFormer-B are detailed in \cref{tab:emformer-arch}.

Our implementation adopts the FlashAttention architecture as its core backbone~\cite{flashattention}. For the ImageNet‑1K image‑classification task, the full model is trained in an end‑to‑end manner. Comprehensive training hyperparameters (e.g., learning rates, schedules, regularization) are listed in \cref{tab:implementation_details}. The configurations of the three model variants, EMFormer‑T (Tiny), EMFormer‑S (Small), and EMFormer‑B (Base), are detailed in \cref{tab:emformer-arch}.

\section{Evaluation Metric}
\label{sec:add_metric}

\subsection{Weather Forecasting}

We evaluate forecasting performance using three standard metrics: Root Mean Square Error (RMSE), Normalized RMSE (NRMSE), and Anomaly Correlation Coefficient (ACC). To account for the convergence of meridians toward the poles, all metrics incorporate a latitude‑dependent weight \(L_i\). The definitions are:

\begin{align}
    \text{RMSE}(t) &=\sqrt{  \frac{  \sum_{i=1}^{N_{\text{lat}}} \sum_{j=1}^{N_{\text{lon}}} L_{i} \bigl(\hat{X}_{i,j}^t - X_{i,j}^t\bigr)^{2} }{N_{\text{lat}}\times N_{\text{lon}}} }, \\[6pt]
    \text{NRMSE}(t) &=\sqrt{  \frac{  \sum_{i=1}^{N_{\text{lat}}} \sum_{j=1}^{N_{\text{lon}}} L_{i} \bigl(\hat{Z}_{i,j}^t - Z_{i,j}^t\bigr)^{2} }{N_{\text{lat}}\times N_{\text{lon}}} }, \quad \text{with} \quad Z_{i,j}^t = \frac{X_{i,j}^t - \mu}{\sigma}, \quad \hat{Z}_{i,j}^t = \frac{\hat{X}_{i,j}^t - \mu}{\sigma}, \\[6pt]
    \text{ACC}(t) &=\frac{ \sum_{i=1}^{N_{\text{lat}}} \sum_{j=1}^{N_{\text{lon}}} L_{i} \; \hat{X}_{i,j}^t \; X_{i,j}^t }{ \sqrt{ \sum_{i=1}^{N_{\text{lat}}} \sum_{j=1}^{N_{\text{lon}}} L_{i} \bigl(\hat{X}_{i,j}^t\bigr)^2 \; \times \; \sum_{i=1}^{N_{\text{lat}}} \sum_{j=1}^{N_{\text{lon}}} L_{i} \bigl(X_{i,j}^t\bigr)^2 } },
\end{align}

where \(\hat{X}_{i,j}^t\) and \(X_{i,j}^t\) denote the predicted and ground‑truth values at grid point \((i, j)\) and forecast time \(t\); \(N_{\text{lat}}\) and \(N_{\text{lon}}\) are the numbers of grid cells in latitude and longitude. For NRMSE, \(\mu\) and \(\sigma\) are the temporal mean and standard deviation of the ground‑truth variable over the training period, used to normalize both predictions and targets to a common scale. The latitude weight \(L_i\) is defined as:

\begin{equation}
    L_i = N_{\text{lat}} \times \frac{ \cos \phi_i }{ \sum_{j=1}^{N_{\text{lat}}} \cos \phi_j },
\end{equation}

with \(\phi_i\) representing the latitude of the \(i\)-th grid row. To obtain a single scalar assessing overall performance across multiple variables, we also report the \textbf{Mean Normalized RMSE (MNRMSE)}, calculated as the arithmetic mean of the NRMSE values over all predicted variables.

\subsection{Typhoon Tracking}

% This study focuses on 10 intense tropical cyclones: AMPIL (2024.08), BEBINCA (2024.09), EWINIAR (2024.05), GAEMI (2024.07), KONG-REY (2024.10), KRATHON (2024.09), MAN-YI (2024.11), SHANSHAN (2024.08), YAGI (2024.09), and YINXING (2024.11). Those typhoons are all developed in East Asia or the Western Pacific across the whole year in 2024. Initial conditions for both cyclones are set at 00:00 UTC on their respective formation dates. All competitors and Ground‑truth are obtained from CMA TCData\footnote{\url{tcdata.typhoon.org.cn}}. 

% This study focuses on two intense tropical cyclones: Typhoon Ewiniar and Typhoon Yinxing. Typhoon Ewiniar formed east of Mindanao on May 24, 2024, crossed the Philippine Sea, and recurved northeastward over the Okinawa–Ryukyu region. Typhoon Yinxing developed east of Yap Island on November 4, 2024, traversed the Philippine Sea, and entered the South China Sea. Initial conditions for both cyclones are set at 00:00 UTC on their respective formation dates. Ground‑truth tracks and ECMWF forecasts are obtained from TCData\footnote{\url{tcdata.typhoon.org.cn}}; other model tracks are generated using official released codes.

This study evaluates the proposed forecasting model using ten tropical cyclones that occurred across the 2024 season, selected as test cases for their intensity and representative tracks in the Western Pacific and East Asia region. The cyclones, listed with their international names and approximate formation periods, are: AMPIL (August 2024), BEBINCA (September 2024), EWINIAR (May 2024), GAEMI (July 2024), KONG‑REY (October 2024), KRATHON (September 2024), MAN‑YI (November 2024), SHANSHAN (August 2024), YAGI (September 2024), and YINXING (November 2024). All systems developed within the Western Pacific or adjacent East Asian maritime basins. The initial condition for each forecast is set at 00:00 UTC on its respective formation date. Ground-truth and the corresponding forecasts from all comparison models are obtained from the publicly accessible CMA Tropical Cyclone Data Archive (TCData)\footnote{\url{https://tcdata.typhoon.org.cn}}.

Following previous AI‑based approaches~\cite{panguweather, magnusson2021tropical}, we identify the tropical‑cyclone eye as the location of the local minimum in mean sea‑level pressure (MSL). The tracking algorithm employs a multi‑constraint sequential method that combines MSL minimization with physical consistency checks to ensure robust and meteorologically plausible trajectory estimation.

The tracking procedure consists of the following steps: (1) initialize the current position with the given coordinates; (2) extract the latitude–longitude grid and relevant atmospheric variables (MSL, U10, V10); (3) locate the cyclone center by finding the minimum MSL within a 278 km search radius; (4) compute the central pressure, maximum wind speed, and displacement distance; (5) evaluate termination criteria against thresholds for pressure (101 200 Pa), wind speed (10.2$m/s$), and displacement (400 km); (6) if all criteria are satisfied, update the current position and append the point to the trajectory.

To ensure the robustness of the tracking results, we conduct a comprehensive sensitivity analysis on the key threshold parameters. The choice of the 278 km search radius is based on the typical diameter of tropical cyclones, while the pressure (101 200 Pa) and wind‑speed (10.2 $m/s$) thresholds follow the World Meteorological Organization definition of a tropical depression. We systematically vary each threshold within a plausible range (pressure: 101 000–102 000Pa; wind speed: 9.0–11.5$m/s$; search radius: 200–400 km) and evaluate the resulting track errors. 

% Across all tested configurations, our model consistently outperforms the baseline methods, with the average improvement in mean distance error remaining above 25\%. This demonstrates that the reported performance gain is not an artifact of specific parameter choices but reflects the inherent superiority of the proposed approach.

For typhoon track evaluation, we employ the Mean Distance Error (MDE), which averages the great‑circle (Haversine) distance between predicted and observed cyclone centers over all forecast steps:
\begin{align}
\text{MDE} &= \frac{1}{N} \sum_{k=1}^{N} d\bigl({P}_{\text{pred}}^{(k)},\, {P}_{\text{obs}}^{(k)}\bigr),\\[4pt]
d({P}_1, {P}_2) &= 2R \cdot \arcsin\!\bigl(\sqrt{a}\,\bigr), \\[4pt]
a &= \sin^2\!\Bigl(\frac{\Delta\phi}{2}\Bigr) + \cos\phi_1 \cdot \cos\phi_2 \cdot \sin^2\!\Bigl(\frac{\Delta\lambda}{2}\Bigr),
\end{align}
where \(R = 6371\ \text{km}\) is the Earth’s mean radius, \({P}_1 = (\phi_1, \lambda_1)\) and \({P}_2 = (\phi_2, \lambda_2)\) are the latitude–longitude pairs of two points, and \(\Delta\phi = \phi_2 - \phi_1\), \(\Delta\lambda = \lambda_2 - \lambda_1\).

\subsection{Image Classification on Imagenet-1K}
\label{sec:top1_accuracy}

Top-1 accuracy represents the most stringent evaluation metric for image classification models on the ImageNet-1K dataset. It measures the proportion of test images for which the model's single highest-confidence prediction matches the ground-truth label. Formally, given a validation set $\mathcal{D} = \{(\mathbf{x}_i, y_i)\}_{i=1}^N$ where $\mathbf{x}_i \in \mathbb{R}^{H \times W \times 3}$ denotes the input image, $y_i \in \{1, 2, \dots, K\}$ is the true class label ($K=1000$ for ImageNet-1K), and $N=50,000$ is the size of the validation set, the top-1 accuracy is computed as:

\begin{equation}
\text{Accuracy}_{\text{top-1}} = \frac{1}{N} \sum_{i=1}^{N} \mathbb{I}\left( \arg\max_{k \in \{1,\dots,K\}} f_k(\mathbf{x}_i; \theta) = y_i \right)
\label{eq:top1_accuracy}
\end{equation}

\noindent where:
\begin{itemize}
    \item $f_k(\mathbf{x}_i; \theta)$ denotes the predicted probability for class $k$ given input $\mathbf{x}_i$ and model parameters $\theta$
    \item $\arg\max_{k} f_k(\mathbf{x}_i; \theta)$ identifies the class index with the highest prediction probability
    \item $\mathbb{I}(\cdot)$ is the indicator function that returns 1 when the condition is true and 0 otherwise
\end{itemize}

The evaluation protocol follows strict standardization: all images are center-cropped to $224 \times 224$ pixels after resizing the shorter edge to 256 pixels. No test-time augmentation is applied during top-1 accuracy computation, ensuring consistent comparison across different architectures. This metric provides a direct measure of a model's ability to make precise single-label predictions under real-world deployment scenarios where only one classification decision is permitted per input.

\subsection{Semantic Segmentation on ADE20K}
\label{sec:miou_ade20k}

The mean Intersection over Union (mIoU) serves as the primary evaluation metric for semantic segmentation performance on the ADE20K dataset. This metric quantifies the spatial overlap accuracy between predicted segmentation masks and ground-truth annotations across all semantic classes. For a segmentation model $f: \mathbb{R}^{H \times W \times 3} \rightarrow \{1, 2, \dots, C\}^{H \times W}$ where $C=150$ represents the number of semantic classes in ADE20K validation set, the mIoU is computed through the following formal procedure.

Given a validation image $\mathbf{x}$ with ground-truth segmentation mask $\mathbf{y} \in \{1, \dots, C\}^{H \times W}$ and predicted segmentation $\hat{\mathbf{y}} = f(\mathbf{x})$, the Intersection over Union for class $c$ is defined as:

\begin{equation}
\text{IoU}_c = \frac{|\{p : y_p = c \land \hat{y}_p = c\}|}{|\{p : y_p = c \lor \hat{y}_p = c\}| + \epsilon}
\label{eq:class_iou}
\end{equation}

\noindent where:
\begin{itemize}
    \item $p$ indexes spatial positions in the segmentation map
    \item $y_p$ and $\hat{y}_p$ denote the ground-truth and predicted class labels at position $p$
    \item $\epsilon = 10^{-6}$ is a small constant to prevent division by zero when both sets are empty
    \item The numerator computes the true positive pixels for class $c$
    \item The denominator represents the union of ground-truth and predicted pixels for class $c$
\end{itemize}

The mean IoU across all $C$ semantic classes is then calculated as:

\begin{equation}
\text{mIoU} = \frac{1}{C} \sum_{c=1}^{C} \text{IoU}_c
\label{eq:miou}
\end{equation}

For the ADE20K benchmark, evaluation follows strict protocol specifications: predictions are resized to match the original image resolution before computing mIoU, and the metric is computed exclusively on the official 2,000-image validation set. Notably, the background class (class 0) is excluded from evaluation, focusing solely on the 150 foreground semantic categories. This implementation adheres to the standard evaluation code provided by the MIT Scene Parsing Benchmark, ensuring consistent and comparable results across different segmentation architectures. The mIoU metric provides a robust measure of a model's ability to accurately delineate object boundaries and correctly classify regions across diverse indoor and outdoor scene categories.

\begin{table}[t]
\centering
\small
\caption{Comprehensive performance comparison between Plain Multi-scale Module and Multi-Convs Layer. All metrics are measured on identical hardware and input conditions with NVIDIA A100 GPU.}
\label{tab:conv_performance}
% \begin{threeparttable}
\begin{tabular}{l| cc | c}
\toprule
\multirow{2}{*}{ \textbf{Metric} } & {\textbf{Plain Multi-scale Module}} & {\textbf{Multi-Convs Layer}} &  \\
\cmidrule(lr){2-4}  % \cmidrule(l){4-5}
&  \multicolumn{2}{c}{\textbf{Output Statistics}} & {\textbf{Difference}} \\
\midrule
% \rowcolor{highlight}
\textbf{Range (min/max)} & {-154.778732/155.702896} & {-154.780792/155.698761} &  \\
Mean & 0.001675 & 0.001684 & 0.000009 \\
Std Dev & 32.543854 & 32.543770 & 0.000084 \\
% Max Difference & \multicolumn{2}{c}{0.042381} & \multicolumn{2}{c}{0.000028} \\
\midrule
\multicolumn{4}{c}{\textbf{Weight Gradients} } \\
\midrule
Weight 1 / Layer 1 & 2.657446 & 2.657445 & 0.0000001 \\
% & & & \multicolumn{2}{c}{Rel Diff: 0.000017} \\
Weight 2 / Layer 2 & 6.396124 & 6.396096 & 0.000028 \\
% & & & \multicolumn{2}{c}{Rel Diff: 0.000008} \\
Weight 3 / Layer 3 & 8.638949 & 8.638924 &  0.000025 \\
% & & & \multicolumn{2}{c}{Rel Diff: 0.000006} \\
\midrule
% \rowcolor{highlight}
 & \multicolumn{2}{c}{\textbf{ Time (s)}} & \textbf{Speedup Factor} \\
\cmidrule(lr){1-4} % \cmidrule(l){4-5}
\textbf{Performance} & 0.227616  & 0.039987 &   5.69$\times$ \\
% Custom Layer & 0.039987 &  0.227616   & \multicolumn{2}{c}{\multirow{2}{*}{5.69$\times$}} \\
% Separated Conv & & 0.227616 & \multicolumn{2}{c}{} \\
\bottomrule
\end{tabular}
\end{table}

\begin{table*}[t]
    % \vspace{-5pt}
    % \vskip 0.13in
    \centering
    \small
    \caption{Ablation study of the loss function on weather forecasting tasks. A lower RMSE (denormalized, $\downarrow$) indicates better performance. \textbf{All experiments are conducted on a 6‑year subset of ERA5 data spanning 2015–2020.} Consequently, the resulting error values are notably worse than those reported in earlier experiments, which used different evaluation periods. }
    
    % Ablation study of loss function on weather forecasting task. A small RMSE (denormalized, $\downarrow$)  indicate better performance. The experiments are all tested on a 5-year ERA5 dataset from 2015 to 2020. Thus, the experimental results seem much lower than the previous experiments.
     \makebox[\textwidth]{
\resizebox{\linewidth}{!}{
    % \begin{small}
    %     \begin{sc}
            \renewcommand{\multirowsetup}{\centering}
            \begin{tabular}{l|ccc|ccc|ccc|l}
                \toprule
                \multirow{2}{*}{Loss Function} & \multicolumn{3}{c|}{6-hour (RMSE $\downarrow$)}  & \multicolumn{3}{c|}{4-day (RMSE $\downarrow$)}   & \multicolumn{3}{c|}{10-day (RMSE $\downarrow$)}   & \multirow{2}{*}{Equation} \\
                \cmidrule(lr){2-10}
               & Z500   & T2M    & U10  &  Z500   & T2M    & U10  &  Z500   & T2M    & U10  &   \\

                \midrule   

                L2Loss  & 34.4   &  0.731  & 0.651  &   -   &  -   & -  &   -   &  -   & -  &  \(\mathcal{E} = (\mathbf{\hat{X}}^{t+1} - \mathbf{X}^{t+1})^2\)  \\  
                \cmidrule(lr){2-11} 
                Variable-weighted Loss  & 31.8   & 0.704  & 0.652    &   -   &  -   & -  &   -   &  -   & -  &  $\frac{\mathcal{E} }{e^\text{w}} + \text{w}$  \\  
                \cmidrule(lr){2-11} 
                Lat-weighted Loss   &  31.5  &  0.707   &   0.659    &   -   &  -   & -  &   -   &  -   & -  &  $\textbf{L}\odot \mathcal{E}$   \\   \cmidrule(lr){2-11}
                
                Lat- and Variable- weighted Loss  &  32.5   &  0.713   &  0.643    &   217.96   &  1.46   & 2.05  &   789.68   &  3.21   & 4.41  &  $\textbf{L}\odot \Bigl(\frac{\mathcal{E} }{e^\text{w}} + \text{w} \Bigr)$   \\   
                \cmidrule(lr){2-11} 
                Hybrid Loss ($\lambda_{init}=0.2$) &  31.1   &  0.796   &  0.651   &   215.82   &  1.31   & 2.01   &   764.51   &  2.82   &  4.27  &  \multirow{4}{*}{$ \textbf{L}\odot \mathcal{E} + \lambda * \Bigl(\frac{\mathcal{E} }{e^\text{w}} + \text{w} \Bigr), \lambda$ is trainable}   \\   \cmidrule(lr){2-10}
                Hybrid Loss ($\lambda_{init}=0.3$) & 32.1   &  0.707   &  0.649   &   221.57   &  1.30   &  2.08  &   802.67   &  2.81   &   4.41      \\    \cmidrule(lr){2-10}
                Hybrid Loss ($\lambda_{init}=0.4$) & 35.7   &  0.731    &   0.654  & 227.74   &  1.36   & 2.18  &   833.6   &  2.92   & 4.61      \\   \cmidrule(lr){2-11} 
                Hybrid Loss (Sinusoidual Weighting) &   31.6   &   0.696  &  0.641   &  -   &  -   & -  &   -   &  -   & -  &  $\frac{1}{2}\bigl(1 + \sin(\theta)\bigr) \bigl( \mathbf{L} \odot \mathcal{E} \bigr) \;+\; \frac{1}{2}\bigl(1 - \sin(\theta)\bigr) \Bigl( \frac{\mathcal{E}}{ e^{\mathbf{w}} } + \mathbf{w} \Bigr)$  \\ 
                \cmidrule(lr){2-11} 
                
                Hybrid Loss (Ours) &   30.3   &   0.684  &  0.626   & 212.34   &  1.28   & 1.99  &   756.67   &  2.72   & 4.22  &   $\frac{1}{2}\bigl(1 - \sin(\theta)\bigr) \bigl( \mathbf{L} \odot \mathcal{E} \bigr) \;+\; \frac{1}{2}\bigl(1 + \sin(\theta)\bigr) \Bigl( \frac{\mathcal{E}}{ e^{\mathbf{w}} } + \mathbf{w} \Bigr)$  \\

                \bottomrule
            \end{tabular}
 %        \end{sc}
	% \end{small}
    }}
    % \vspace{-0.1cm}
    
    % \vspace{-5mm}
    \label{tab:ablation_loss}
\end{table*}

\begin{figure*}[t]
  \centering
   \includegraphics[width=0.7\linewidth]{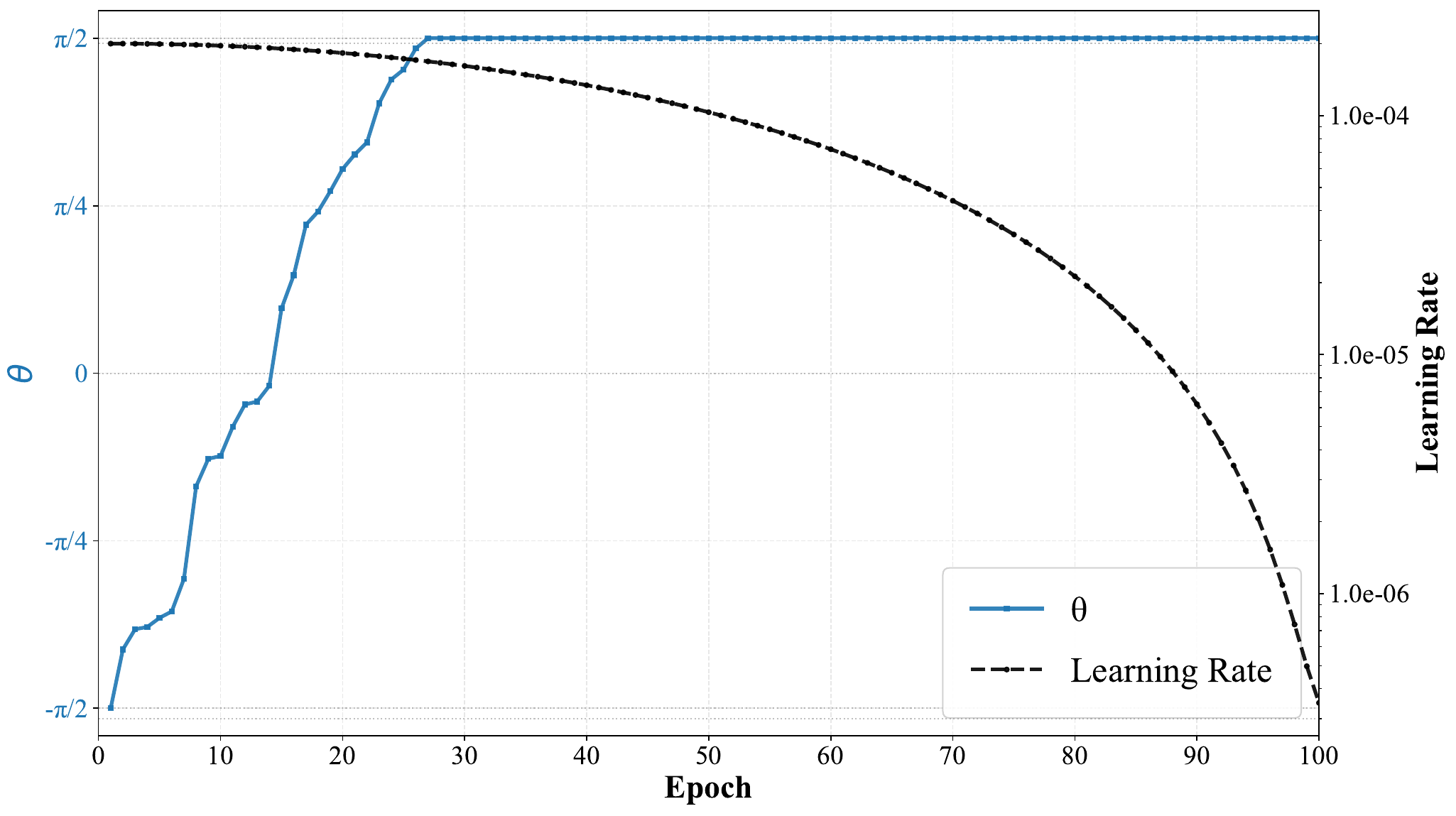}
   \caption{ Training curves of learning rate and $\theta$ in 100 epochs.
   } 
   % \footnotemark
  % \footnotetext{ \url{https://sites.research.google/gr/weatherbench/deterministic-scores} }         % 定义脚注文本
   \label{fig:lr_theta}
   % \vspace{-0.2cm}
\end{figure*}

\begin{table*}[t]
  \centering
  \small
  \caption{Comparative Analysis of Training Times and Hardware Specifications for Deep Learning Models. $\ast$ denotes that the reported time includes only the training duration, excluding finetuning. GPU cost and Latency are all computed by us with one A100(40G) GPU.  }
%        \makebox[\textwidth]{
% \resizebox{1.0\linewidth}{!}{
    \begin{tabular}{lcclccc}
    \toprule
    \multirow{2}{*}{\textbf{Model}} & \textbf{Params} & \textbf{MACs} & \multicolumn{1}{c}{\textbf{GPUs}} & \textbf{Training Time}  &  \textbf{GPU Cost}  &  \textbf{Latency}  \\
     & \textbf{(M)} & \textbf{(G)} &   &    &  \textbf{(Inference, M)}  &  \textbf{(ms)}  \\
    
    \midrule
        \multicolumn{7}{c}{$721\times 1440$ (0.25\degree)} \\
    \midrule
    Fengwu~\cite{fengwu} & 153.49 & 132.83 & 32 A100 & 17 days & 1363.1  & 227 \\
    FourCastNet~\cite{fourcastnet} & 79.6 & 111.62 & 64 A100 & 16 hours & 868.0  &  164.0 \\
    Graphcast~\cite{graphcast}  & 28.95 & 1639.26 & 32 TPUv4 & 4 weeks  & $>$40G & -   \\
    Pangu-Weather~\cite{panguweather} & 23.83 & 142.39 & 192 V100 & 64 days & 997.1 & 275  \\
    Ours   &  208.70  &  124.89  &  16 A100  &  9 days  &  962.2  &  177.6  \\
    \midrule
        \multicolumn{7}{c}{$128\times 256$ (1.4\degree)} \\
    \midrule
    VA-MoE~\cite{vamoe} & 665.37 & - & 16 A100 & 12 days$^\ast$  &  688.3  &  269  \\
    OneForecast ~\cite{oneforecast} &  24.76  &  509.27  &  16 A100  &  8 days$^\ast$ & 225.3   &  663  \\
    STCast~\cite{stcast} &  654.82  &  436.12  &  16 A100  &  5 days$^\ast$  & 546.7   &  115 \\
    Ours  &  157.9  &  102.54  &  16 A100  &  60 hours$^\ast$   &  703.3   &  98.3 \\
    \bottomrule
    \end{tabular}%
    % }}
  \label{tab:times}%
\end{table*}%

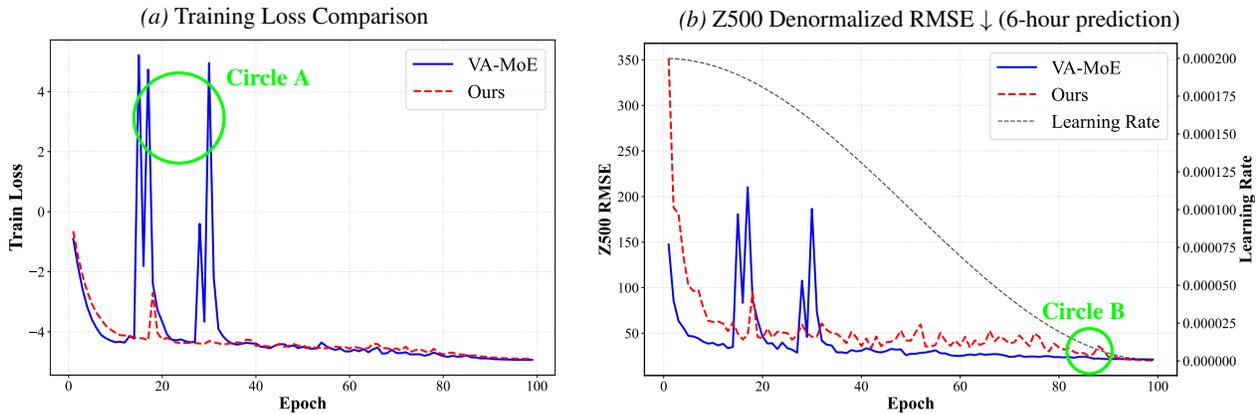
\begin{figure*}[t] % 使用figure*环境确保图片出现在页面顶部
    \centering
    \begin{minipage}{\columnwidth} % 限制为单栏宽度
        \centering

        \begin{subfigure}{0.455\textwidth}
            \centering
            \captionsetup{justification=centering}
            \caption{Training Loss Comparison }
            % 使用TikZ在图片上叠加绿色圆圈和文字
            \begin{tikzpicture}[remember picture]
                \node[anchor=south west, inner sep=0] (image) at (-2.6,-5) {
                    \includegraphics[width=0.95\linewidth]{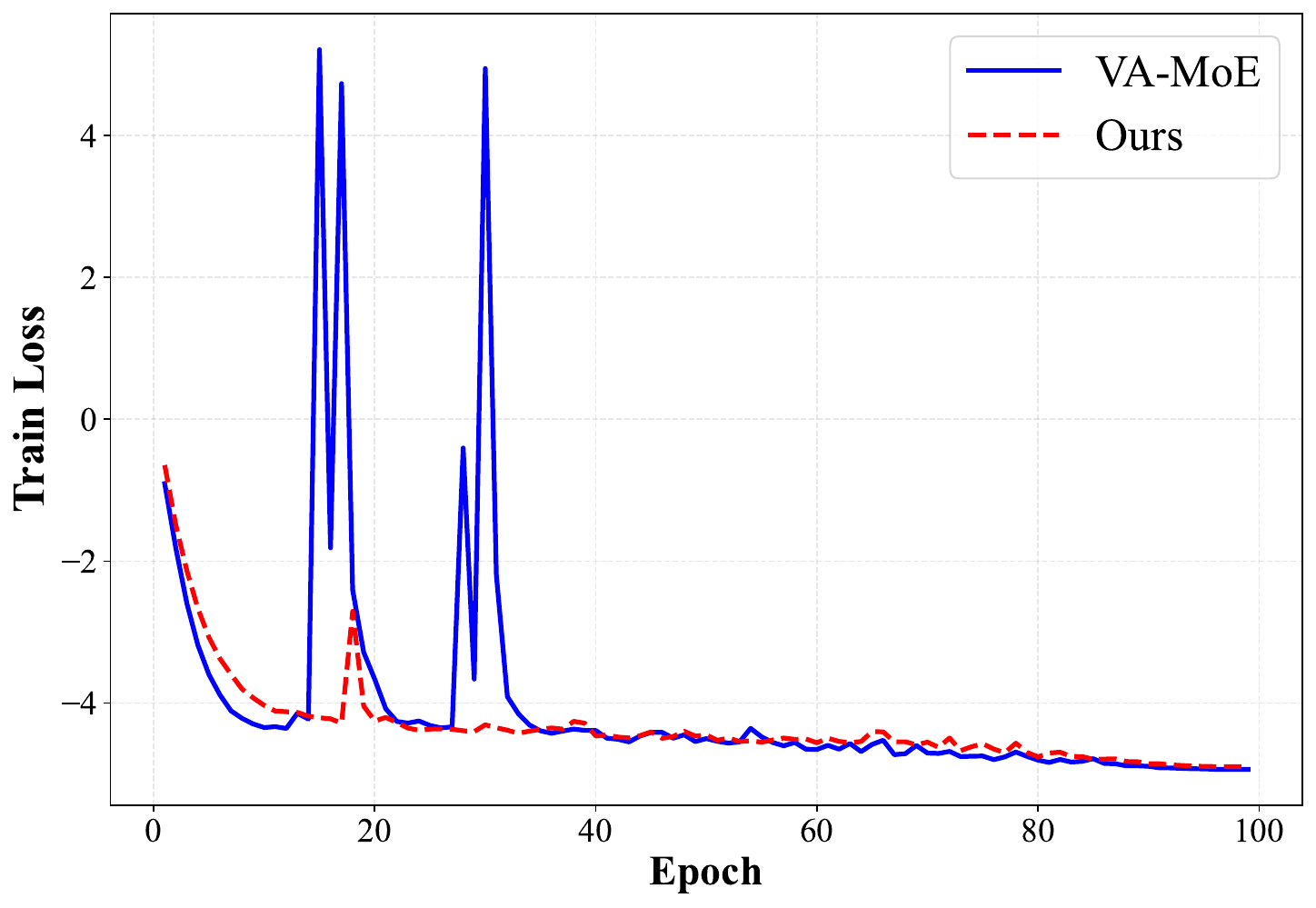}
                };
                % 获取图片尺寸
                \pgfgetlastxy{\imagewidth}{\imageheight}
                % 定义相对坐标 (30%宽度, 70%高度)
                \coordinate (circle_pos) at (0.1*\imagewidth, 0.2*\imageheight);
                % 画绿色圆圈 (半径0.8cm)
                \draw[green, thick, line width=1.2pt] (circle_pos) circle[radius=0.6cm];
                % 添加绿色文字标注 (在圆圈右上方)
                \node[green, font=\bfseries\footnotesize, 
                      above right=0.3cm and 0.5cm of circle_pos] (label) {Circle A};
                % 可选：添加指向线
                % \draw[green, thick, ->] (label.south west) to[out=225, in=45] ($(circle_pos)+(0.3,0.3)$);
            \end{tikzpicture}
            \label{fig:loss_a}
        \end{subfigure}%
        \hfill
        \begin{subfigure}{0.545\textwidth}
            \centering
            \captionsetup{justification=centering}
            \caption{Z500 Denormalized RMSE $\downarrow$ (6-hour prediction)}
            % 使用TikZ在图片上叠加绿色圆圈和文字
            \begin{tikzpicture}[remember picture]
                \node[anchor=south west, inner sep=0] (image) at (-33,-1) {
                    \includegraphics[width=0.95\linewidth]{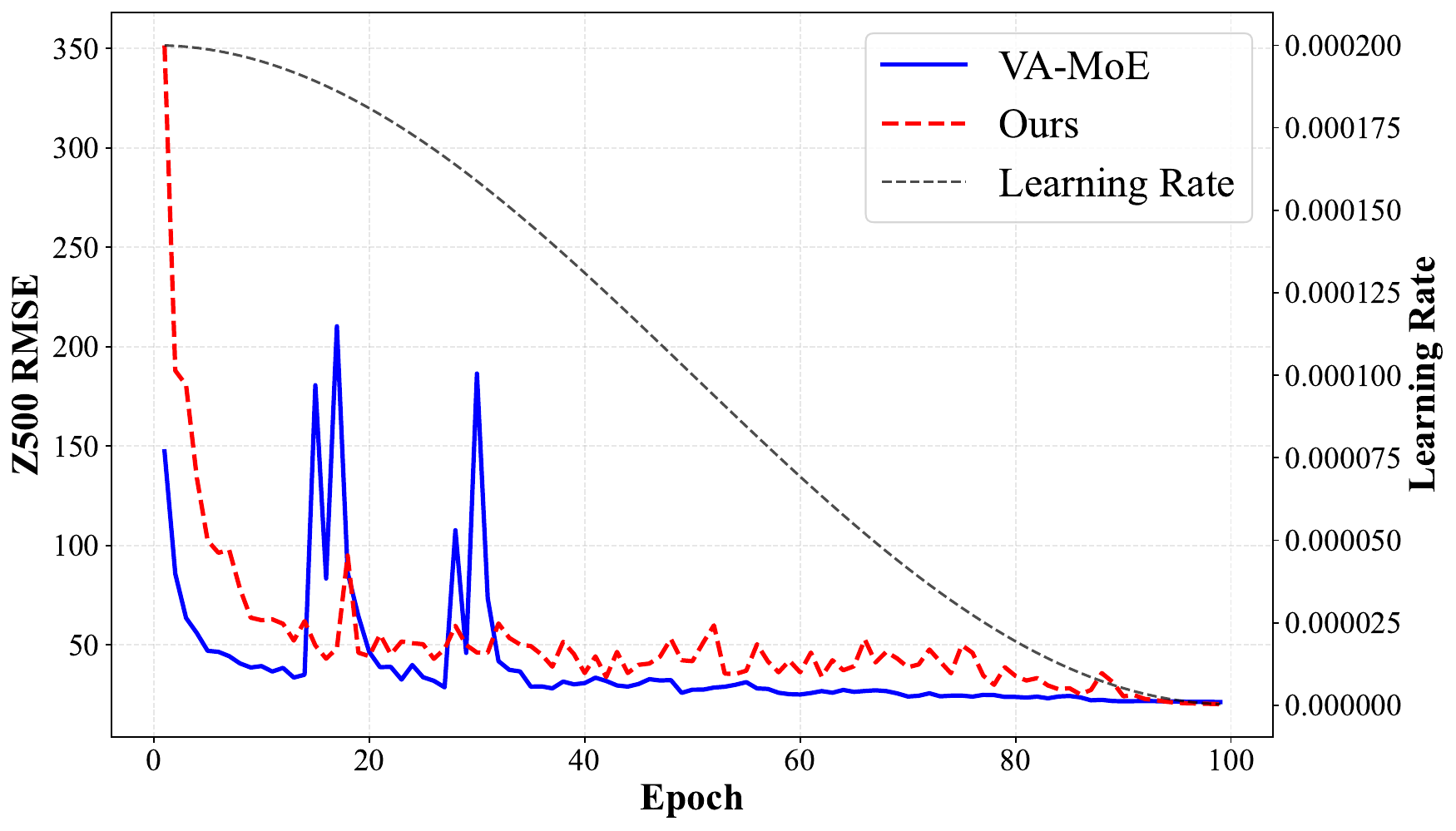}
                };
                % 获取图片尺寸
                \pgfgetlastxy{\imagewidth}{\imageheight}
                % 定义相对坐标 (65%宽度, 40%高度)
                \coordinate (circle_pos) at (0.8*\imagewidth, 0.1*\imageheight);
                % 画绿色圆圈 (半径0.6cm)
                \draw[green, thick, line width=1.2pt] (circle_pos) circle[radius=0.3cm];
                % 添加绿色文字标注 (在圆圈左下方)
                \node[green, font=\bfseries\footnotesize, 
                      above left=0.3cm and -0.6cm of circle_pos] (label) {Circle B};
                % 可选：添加指向线
                % \draw[green, thick, ->] (label.north east) to[out=45, in=225] ($(circle_pos)+(-0.3,-0.3)$);
            \end{tikzpicture}
            \label{fig:loss_b}
        \end{subfigure}
        
        % 整体大标题
        \caption{Comparison between the training processes of VA-MoE~\cite{vamoe} and the proposed method. 
        (a) Training loss comparison: \textcolor{red}{Ours} and \textcolor{blue}{VA-MoE}.
        (b) Short-term forecast performance: evolution of Z500 RMSE during training for \textcolor{red}{Ours} and \textcolor{blue}{VA-MoE}. 
        }
        
        \label{fig:losses}
    \end{minipage}
\end{figure*}

\begin{figure*}[t]
  \centering
   \includegraphics[width=1.0\linewidth]{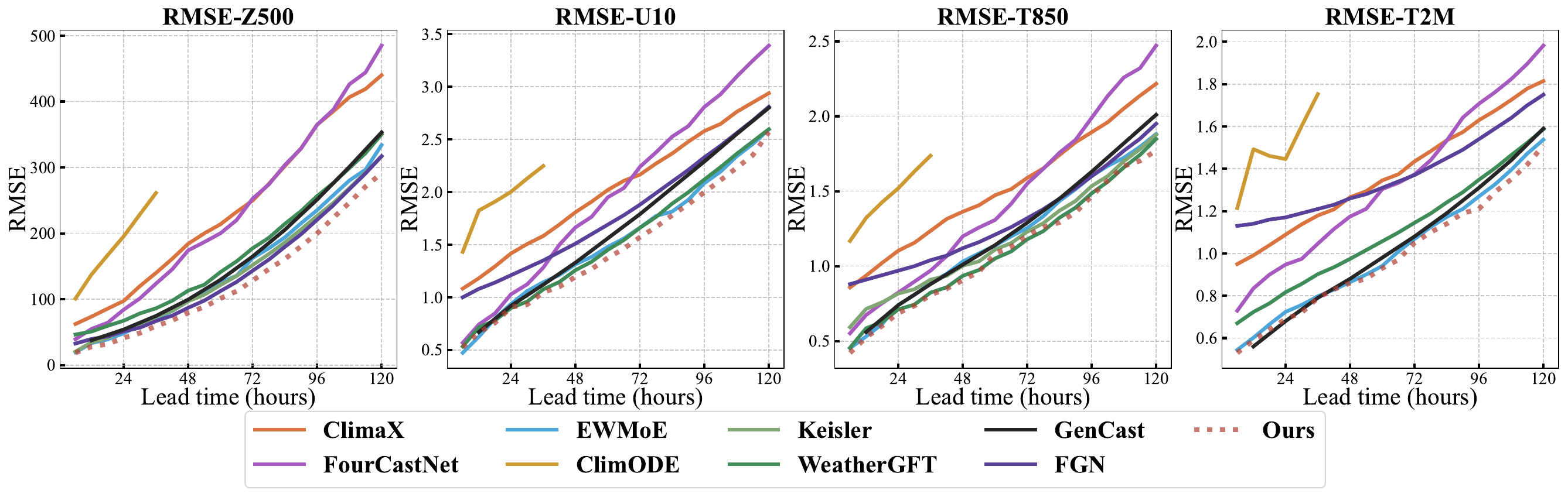}
   \caption{ 
   120-hour comparative analysis of \textbf{RMSE} \boldmath{$\downarrow$} across 10 data-driven models for four variables, including Z500, T850, T2M, and U10. Results are collected from EWMoE~\cite{ewmoe}, WeatherGFT~\cite{weathergft} and WeatherBench~\cite{weatherbench} in \url{https://sites.research.google/gr/weatherbench/deterministic-scores}. 
   } 
   % \footnotemark
  % \footnotetext{ \url{https://sites.research.google/gr/weatherbench/deterministic-scores} }         % 定义脚注文本
   \label{fig:sota}
   % \vspace{-0.2cm}
\end{figure*}

\begin{table}[t]
\centering
\small
\caption{Efficiency and Accuracy Comparison. \textbf{(a)} Module-level efficiency comparison (average of 50 inferences). \textbf{(b)} End-to-end forecasting performance (Denormalized RMSE on 1.4$^{\circ}$ ERA5, 6-year quick training) alongside overall latency and MACs.}
\label{tab:efficiency_accuracy}

\vspace{0.2cm}
\textbf{(a) Single Module Efficiency Comparison} \\
\vspace{0.1cm}
\begin{tabular}{lccc}
\toprule
\textbf{Single Module} & \textbf{Params (M)} & \textbf{Latency (ms)} & \textbf{GPU Memory (MB)} \\
\midrule
RepVGG~\cite{ding2021repvgg} & 9.17 & 14.3 & 324 \\
RepLKNet~\cite{ding2022scaling} & 9.17 & 14.5 & 318 \\
Plain Multi-Scale Module & 9.17 & 31.2 & 301 \\
Multi-Convs Layer (Ours) & 9.17 & 8.9 & 195 \\
\bottomrule
\end{tabular}

\vspace{0.4cm}
\textbf{(b) End-to-end Denormalized RMSE and Efficiency} \\
\vspace{0.1cm}
% \resizebox{\textwidth}{!}{% 如果表格过宽，可以取消注释并使用 resizebox
\begin{tabular}{lccccccccccc}
\toprule
\multirow{2}{*}{\textbf{Method}} & \multicolumn{3}{c}{\textbf{6-hour}} & \multicolumn{3}{c}{\textbf{4-day}} & \multicolumn{3}{c}{\textbf{10-day}} & \multirow{2}{*}{\begin{tabular}[c]{@{}c@{}}\textbf{Latency}\\ \textbf{(ms)}\end{tabular}} & \multirow{2}{*}{\begin{tabular}[c]{@{}c@{}}\textbf{MACs}\\ \textbf{(G)}\end{tabular}} \\
\cmidrule(lr){2-4} \cmidrule(lr){5-7} \cmidrule(lr){8-10}
 & Z500 & T2M & U10 & Z500 & T2M & U10 & Z500 & T2M & U10 & & \\
\midrule
RepVGG~\cite{ding2021repvgg} & 32.7 & 0.715 & 0.664 & 319.37 & 2.81 & 2.49 & 871.08 & 4.32 & 4.66 & 110.6 & 119.3 \\
RepLKNet~\cite{ding2022scaling} & 32.3 & 0.719 & 0.671 & 229.72 & 1.37 & 2.20 & 835.87 & 2.93 & 4.64 & 113.7 & 122.1 \\
% DSMS Conv &	33.9 &	0.693 &	0.733 &	234.56 &	1.33 &	2.22 &	826.86 &	2.92 &	4.61 &	141.5 &	127.9    \\
EMFormer (Ours) & 30.3 & 0.684 & 0.626 & 212.34 & 1.28 & 1.99 & 764.51 & 2.72 & 4.27 & 98.3 & 102.5 \\
\bottomrule
\end{tabular}
% } % 对应 resizebox 的结束括号
\end{table}

\begin{table*}[t]
    % \vspace{-5pt}
    % \vskip 0.13in
    \centering
    \small
    \caption{Ablation study on the accumulative context finetuning strategy with VA-MoE~\cite{vamoe}. All results are trained with 1.4\degree ERA5. The model is trained by us followed official code. }
    
%      \makebox[\textwidth]{
% \resizebox{0.9\linewidth}{!}{
    % \begin{small}
    %     \begin{sc}
            \renewcommand{\multirowsetup}{\centering}
            \begin{tabular}{l@{\hspace{0.5cm}}|cc|cc|cc|cc|cc}
                \toprule
                \multirow{4}{*}{VA-MoE} & \multicolumn{10}{c}{Metric}  \\
                \cmidrule(lr){2-11}
                &  \multicolumn{2}{c}{6-hour} & \multicolumn{2}{c}{1-day} & \multicolumn{2}{c}{4-day} & \multicolumn{2}{c}{7-day} & \multicolumn{2}{c}{10-day}   \\
                \cmidrule(lr){2-11}
               & RMSE& ACC & RMSE& ACC & RMSE& ACC & RMSE& ACC & RMSE& ACC \\

                \midrule

                % Ours\_v1 & 0.0814 & - & 0.1363 & - & 0.2797 & - & 0.4575 & - & 0.5997 & -  \\
                
                % Oneforecast\cite{oneforecast} & 0.0549 & 0.9943  & 0.1231    & 0.9737      & 0.2732  & 0.8825  &  0.4468  & 0.6888  & 0.5918  & 0.4457   \\
                
                w/o finetuning & 0.0613 & 0.9938 & 0.1248 & 0.9715 & 0.2699 & 0.8737 & 0.4436 & 0.6899 & 0.5855 & 0.4481  \\

                w/ Finetuning & 0.0630 & 0.9925 & 0.1205 & 0.9734 & 0.2603 & 0.8849 & 0.4293 & 0.6966 & 0.5673 & 0.4701  \\

                w/ Accumulative finetuning & 0.0608 & 0.9941 & 0.1161 & 0.9761 & 0.2451 & 0.8896 & 0.4101 & 0.7194 & 0.5128 & 0.5351 \\

                \bottomrule
            \end{tabular}
 %        \end{sc}
	% \end{small}
    % }}
    % \vspace{-0.1cm}
    
    % \vspace{-5mm}
    \label{tab:ablation_vamoe}
\end{table*}

\begin{table*}[t]
    % \vspace{-5pt}
    % \vskip 0.13in
    \centering
    \small
    \caption{Ablation study on the balancing hyperparamter $\lambda$ in the memory module of accumulative context finetuning. All results are gained with 1.4\degree ERA5 on 10-step accumulative context finetuning with 5 cache length.}
    
    \renewcommand{\multirowsetup}{\centering}
    \begin{tabular}{l@{\hspace{0.5cm}}|cc|cc|cc|cc|cc}
        \toprule
        \multirow{4}{*}{$\lambda$} & \multicolumn{10}{c}{Metric}  \\
        \cmidrule(lr){2-11}
        &  \multicolumn{2}{c}{6-hour} & \multicolumn{2}{c}{1-day} & \multicolumn{2}{c}{4-day} & \multicolumn{2}{c}{7-day} & \multicolumn{2}{c}{10-day}   \\
        \cmidrule(lr){2-11}
       & RMSE& ACC & RMSE& ACC & RMSE& ACC & RMSE& ACC & RMSE& ACC \\

        \midrule
        
        0.7   & 0.0599 & 0.9948 & 0.1139 & 0.9774    & 0.2444 & 0.8924 & 0.4099 & 0.7190 & 0.5121 & 0.5253  \\

        0.8   & 0.0598 & 0.9949 & 0.1140 & 0.9775    & 0.2441 & 0.8932 & 0.4087 & 0.7211  & 0.5109 & 0.5316  \\

        0.9   & 0.0599 & 0.9949 & 0.1139 & 0.9775    & 0.2439 & 0.8936 & 0.4072 & 0.7223 & 0.5094 & 0.5389 \\

        1.0   & 0.0599 & 0.9949 & 0.1139 & 0.9775    & 0.2462 & 0.8901 & 0.4139 & 0.7173 & 0.5204 & 0.5218 \\

        \bottomrule
    \end{tabular}
    
    % \vspace{-5mm}
    \label{tab:ablation_lambda}
\end{table*}

\begin{table*}[t]
    % \vspace{-5pt}
    % \vskip 0.13in
    \centering
    \small
    \caption{Ablation study on the cache length in the accumulative context finetuning with 10 steps. All results are gained with 1.4\degree ERA5.}
    
    \renewcommand{\multirowsetup}{\centering}
    \begin{tabular}{l@{\hspace{0.5cm}}|cc|cc|cc|cc|cc}
        \toprule
        \multirow{4}{*}{Cache length} & \multicolumn{10}{c}{Metric}  \\
        \cmidrule(lr){2-11}
        &  \multicolumn{2}{c}{6-hour} & \multicolumn{2}{c}{1-day} & \multicolumn{2}{c}{4-day} & \multicolumn{2}{c}{7-day} & \multicolumn{2}{c}{10-day}   \\
        \cmidrule(lr){2-11}
       & RMSE& ACC & RMSE& ACC & RMSE& ACC & RMSE& ACC & RMSE& ACC \\

        \midrule

        % Ours\_v1 & 0.0814 & - & 0.1363 & - & 0.2797 & - & 0.4575 & - & 0.5997 & -  \\
        
        % Hourglass  & 0.0811 & 0.9872 & 0.1347 & 0.9664 & 0.2696 & 0.8811 & 0.4407 & 0.6878 & 0.5819 & 0.4514  \\
        % 0 & 0.0626 & 0.9931 & 0.1229 & 0.9739 & 0.2673 & 0.8765 & 0.4407 & 0.6878 & 0.5819 & 0.4514  \\
        
        % 5 & 0.0624 & 0.9945 & 0.1221 & 0.9743 & 0.2483 & 0.8871 & 0.4235 & 0.7081 & 0.5302 & 0.5098  \\
        
        % 6 & 0.0615 & 0.9947 & 0.1201 & 0.9750 & 0.2481 & 0.8875 & 0.4202 & 0.7101 & 0.5236 & 0.5168  \\

        % 7 & 0.0610 & 0.9947 & 0.1168 & 0.9766 & 0.2468 & 0.8906 & 0.4156 & 0.7168 & 0.5194 & 0.5229  \\
        
        3 & 0.0599 & 0.9948 & 0.1186 & 0.9748 & 0.2511 & 0.8834 & 0.4293 & 0.6991 & 0.5461 & 0.4875  \\

        4 & 0.0598 & 0.9949 & 0.1149 & 0.9760 & 0.2451 & 0.8921 & 0.4101 & 0.7201  & 0.5147 & 0.5219  \\

        5 & 0.0599 & 0.9949 & 0.1139 & 0.9775 & 0.2439 & 0.8936 & 0.4072 & 0.7223 & 0.5094 & 0.5389 \\

        6 & 0.0598 & 0.9950 & 0.1135 & 0.9778 & 0.2431 & 0.8941 & 0.4067 & 0.7229 & 0.5081 & 0.5399 \\

        \bottomrule
    \end{tabular}
    
    % \vspace{-5mm}
    \label{tab:ablation_cachelength}
\end{table*}

\begin{table*}[t]
    % \vspace{-5pt}
    % \vskip 0.13in
    \centering
    \small
    \caption{Ablation study on the number of steps in the finetuning stage with cache length 5. All results are gained with 1.4\degree ERA5.}
    
%      \makebox[\textwidth]{
% \resizebox{0.9\linewidth}{!}{
    % \begin{small}
    %     \begin{sc}
            \renewcommand{\multirowsetup}{\centering}
            \begin{tabular}{l@{\hspace{0.5cm}}|cc|cc|cc|cc|cc}
                \toprule
                \multirow{4}{*}{\#Steps} & \multicolumn{10}{c}{Metric}  \\
                \cmidrule(lr){2-11}
                &  \multicolumn{2}{c}{6-hour} & \multicolumn{2}{c}{1-day} & \multicolumn{2}{c}{4-day} & \multicolumn{2}{c}{7-day} & \multicolumn{2}{c}{10-day}   \\
                \cmidrule(lr){2-11}
               & RMSE& ACC & RMSE& ACC & RMSE& ACC & RMSE& ACC & RMSE& ACC \\

                \midrule

                % Ours\_v1 & 0.0814 & - & 0.1363 & - & 0.2797 & - & 0.4575 & - & 0.5997 & -  \\
                
                % Hourglass  & 0.0811 & 0.9872 & 0.1347 & 0.9664 & 0.2696 & 0.8811 & 0.4407 & 0.6878 & 0.5819 & 0.4514  \\
                0 & 0.0626 & 0.9931 & 0.1219 & 0.9749 & 0.2673 & 0.8845 & 0.4327 & 0.6978 & 0.5719 & 0.4614  \\
                
                5 & 0.0624 & 0.9945 & 0.1210 & 0.9752 & 0.2483 & 0.8871 & 0.4235 & 0.7081 & 0.5302 & 0.5098  \\
                
                6 & 0.0615 & 0.9947 & 0.1201 & 0.9756 & 0.2481 & 0.8875 & 0.4202 & 0.7101 & 0.5236 & 0.5168  \\

                7 & 0.0610 & 0.9947 & 0.1168 & 0.9766 & 0.2468 & 0.8906 & 0.4156 & 0.7168 & 0.5194 & 0.5229  \\
                
                8 & 0.0610 & 0.9948 & 0.1164 & 0.9767 & 0.2445 & 0.8926 & 0.4103 & 0.7197 & 0.5185 & 0.5227  \\

                9 & 0.0603 & 0.9949 & 0.1152 & 0.9771 & 0.2441 & 0.8928 & 0.4088 & 0.7221  & 0.5109 & 0.5314  \\

                10 & 0.0599 & 0.9949 & 0.1139 & 0.9775 & 0.2439 & 0.8936 & 0.4072 & 0.7223 & 0.5094 & 0.5389 \\

                \bottomrule
            \end{tabular}
 %        \end{sc}
	% \end{small}
    % }}
    % \vspace{-0.1cm}
    
    % \vspace{-5mm}
    \label{tab:ablation_steps}
\end{table*}

\begin{table*}[t]
    % \vspace{-5pt}
    % \vskip 0.13in
    \centering
    % \normalsize
    \caption{Ablation study on the windows size of the windows-attention. All results are gained with 1.4\degree ERA5.}
    
     \makebox[\textwidth]{
\resizebox{\linewidth}{!}{
    % \begin{small}
    %     \begin{sc}
            \renewcommand{\multirowsetup}{\centering}
            \begin{tabular}{l|cc|cc|cc|cc|cc}
                \toprule
                \multirow{4}{*}{Windows size} & \multicolumn{10}{c}{Metric}  \\
                \cmidrule(lr){2-11}
                &  \multicolumn{2}{c}{6-hour} & \multicolumn{2}{c}{1-day} & \multicolumn{2}{c}{4-day} & \multicolumn{2}{c}{7-day} & \multicolumn{2}{c}{10-day}   \\
                \cmidrule(lr){2-11}
               & RMSE& ACC & RMSE& ACC & RMSE& ACC & RMSE& ACC & RMSE& ACC \\

                \midrule

                % Hourglass, Hourglass+RWKV, Hourglass+windows

                Only $4\times 4$ & 0.0641 & 0.9922 & 0.1228 & 0.9735 & 0.2679 & 0.8819 & 0.4395 & 0.6885 & 0.5878 & 0.4465  \\

                Hybrid Sizes ($4\times 4, 8\times 2, 2\times 8$)  & 0.0626 & 0.9931 & 0.1219 & 0.9749 & 0.2673 & 0.8845 & 0.4327 & 0.6978 & 0.5719 & 0.4614 \\
                
                % Hourglass  & 0.0811 & 0.9872 & 0.1347 & 0.9664 & 0.2696 & 0.8811 & 0.4407 & 0.6878 & 0.5819 & 0.4514  \\

                % Ours  & 0.0629 & \textbf{0.9969} & \textbf{0.1139} & \textbf{0.9775} & \textbf{0.2539} & \textbf{0.8936} & \textbf{0.4072} & \textbf{0.7223} & \textbf{0.5394} & \textbf{0.4989}  \\

                \bottomrule
            \end{tabular}
 %        \end{sc}
	% \end{small}
    }}
    % \vspace{-0.1cm}
    
    % \vspace{-5mm}
    \label{tab:ablation_windows}
\end{table*}

\section{Additional Experiments}
\label{sec:add_result}

\subsection{Experiments of Proposition}
\label{subsec:add_proofresult}

% In this section, we provide the experiments to verify our theorems.

% \textbf{Experiment on Theorem 1.} We provide the performance comparison between plain multi-scale module and our multi-convs layer in \cref{tab:conv_performance}. With the same input feature and initialized parameters, the difference of output statistics and weight gradients are all lower than \(10^{-4}\), which indicates that the forward and backward function of those two modules are the same. While the forward and backward time of our multi-convs layer is quicker than the plain multi-scale module by 5.69 times.

In this section, we present experimental results to validate our theoretical claims.

\textbf{Experiment on Proposition 1.} The performance comparison between the plain multi‑scale module and our proposed multi‑convs layer is summarized in \cref{tab:conv_performance}. Given identical input features and initialization, the differences in output statistics and weight gradients between the two modules remain below \(10^{-4}\), confirming the functional equivalence of their forward and backward passes in \cref{subsec:proof2}. At the same time, the forward and backward runtimes of the multi‑convs layer are reduced by a factor of 5.69× relative to the plain multi‑scale module.

\textbf{Experiment on Proposition 2.} To validate the proof of Proposition 2 given in \cref{subsec:proof2}, we visualize the evolution of \(\theta\) alongside the learning rate during training (\cref{fig:lr_theta}). The curve shows that \(\theta\) starts from \(-\pi/2\) and stabilizes near \(\pi/2\) after approximately 30 epochs.

To further assess the effectiveness of the proposed loss, we conduct an ablation study on a 6‑year subset of ERA5 data, comparing eight different loss formulations (\cref{tab:ablation_loss}). The results indicate that both the variable‑weighted loss and the latitude‑weighted loss individually outperform the plain L2 loss. Interestingly, simply combining the two weighting strategies without adaptive balancing yields worse performance than either alone. In contrast, the hybrid loss with sine‑based balancing achieves the best results, improving Z500, T2M, T850, and U10 by 0.8, 0.02, 0.024, and 0.017, respectively, over the best single‑weighted variant. These experiments confirm that our hybrid loss effectively guides the model to balance variable‑wise and geographic information, leading to superior forecast accuracy.

\subsection{Training Time}
\label{subsec:times}

As shown in \cref{tab:times}, we compare the number of parameters, multiply–accumulate operations (MACs), GPU usage, training duration, GPU cost in inference, and Latency across eight baseline models. While our model has more parameters and MACs than GNN‑based methods such as GraphCast and OneForecast, its overall computational cost remains substantially lower than that of large‑scale predecessors, notably FengWu, Pangu‑Weather, and GraphCast. These results show that our approach attains better performance while retaining competitive efficiency.

\textbf{How does the model achieve better performance at lower cost?} We examine two training metrics in \cref{fig:losses} and highlight two observations. (1) \cref{fig:loss_a} compares training losses between VA-MoE~\cite{vamoe} and our model. Although VA-MoE converges more rapidly in early epochs, its loss curve displays sharp, irregular spikes (marked in \textcolor{green}{Circle A}), whereas our loss decreases smoothly throughout training. (2) \cref{fig:loss_b} compares Z500 RMSE during training. For the first 90 epochs (under a higher learning rate), our model lags behind VA-MoE. However, in the final 10 epochs, our model exhibits a sharp drop in RMSE when a much lower learning rate is applied (see \textcolor{green}{Circle B}). As noted in \cref{subsec:loss}, lower learning rates favor the optimization of atmospheric variables. Our training schedule aligns with this principle, thereby harmonizing effectively with the characteristics of atmospheric data.

\subsection{Additional Results}
\label{subsec:add_result}

% \textbf{Additional weather forecasting analysis.} As illustrated in \cref{fig:sota}, we compare Ours with several state-of-the-art models across forecasting horizons ranging from 6 to 120 hours, evaluated on four key atmospheric variables. Experimental results show that Ours performs comparably to WeatherGFT~\cite{weathergft} and Stormer~\cite{stormer} in predicting T850 and U10, while outperforming all baselines in Z500 and T2M. These findings highlight the effectiveness of Ours and its integrated multi-scale architecture in weather forecasting, demonstrating its ability to capture atmospheric patterns with different regions across diverse meteorological variables.

\textbf{Additional weather forecasting analysis.} As shown in \cref{fig:sota}, we compare the proposed model with several baselines over forecast horizons from 6 to 120 hours on four key atmospheric variables. Results indicate that our model performs on par with WeatherGFT~\cite{weathergft} and Keisler~\cite{keisler} in predicting T850 and U10, while surpassing all compared methods in Z500 and T2M. These outcomes underscore the effectiveness of the proposed approach and its integrated multi‑scale architecture, demonstrating a consistent ability to capture atmospheric patterns across different spatial regions and diverse meteorological variables.

\textbf{Additional typhoon track prediction.} In addition to the aggregate errors reported in \cref{tab:tc}, we provide more detailed error visualizations for individual lead times in \cref{fig:tcs_part1,fig:tcs_part2}. These plots show the Mean Distance Error (MDE, in kilometers) across nine forecasting systems at each forecast step. The results indicate that our model achieves competitive accuracy in the short term while delivering superior performance at longer lead times, particularly for Typhoons Bebinca, Kong‑rey, and Yagi (see \cref{fig:tcs_bebinca,fig:tcs_kongrey,fig:tcs_yagi}). Such detailed comparisons further confirm the model’s strong capability in extreme‑weather forecasting.

To illustrate the trajectory forecasts concretely, \cref{fig:tctrack} visualizes the predicted tracks of Typhoon Ampil (\cref{fig:tctrack_ewiniar}) and Typhoon Yinxing (\cref{fig:tctrack_yinxing}) alongside those of nine baseline methods. Additional visualizations in \cref{fig:tcs_initial,fig:tcs_predict,fig:tcs_real} further dissect the evolution of Typhoon Yinxing by contrasting the initial conditions, model predictions, and ground‑truth observations. Collectively, these results substantiate the robustness of the proposed approach for extreme‑event assessment and demonstrate its clear advantage in accurately predicting tropical‑cyclone tracks over extended forecast horizons.

% Beyond the typhoon track visualizations in the main paper, we provide a quantitative evaluation using Mean Distance Error (MDE, in kilometers) across five forecasting systems: ECMWF~\cite{molteni1996ecmwf}, FengWu~\cite{fengwu}, Pangu‑Weather~\cite{panguweather}, STCast~\cite{stcast}, and the proposed method. Results in Figures~\ref{fig:tcs_ewiniar} and~\ref{fig:tcs_yinxing} show that our model achieves competitive short‑term accuracy while delivering superior long‑term performance. For Typhoon Yinxing and Typhoon Ewiniar, our method attains the lowest mean errors of 80.6 km and 115.4 km, respectively. 

\subsection{Additional Ablation Study}
\label{subsec:add_ablation}

\textbf{Efficiency and Accuracy Comparison.} \cref{tab:efficiency_accuracy} comprehensively evaluates the efficiency and accuracy of our proposed approach against existing structural re-parameterization methods, such as RepVGG~\cite{ding2021repvgg} and RepLKNet~\cite{ding2022scaling}. While conventional methods implement multi-branch structures at the framework (\textit{e.g.}, PyTorch) level, meaning the training phase still incurs the full multi-branch computational cost and speedups are restricted primarily to inference, our approach pushes this optimization to the CUDA level. By rewriting the low-level execution path, the Multi-Convs Layer efficiently executes multi-scale convolutions while preserving the original representation, enabling acceleration in both training and inference.

As demonstrated in Table \ref{tab:efficiency_accuracy}a, under an identical parameter budget (9.17M), the Multi-Convs Layer achieves the lowest latency (8.9 ms) and GPU memory footprint (195 MB), significantly outperforming the plain multi-scale module and Rep-based baselines. Furthermore, \cref{tab:efficiency_accuracy}b demonstrates that this module-level efficiency directly translates into superior end-to-end forecasting performance. When integrated into the overall architecture, EMFormer not only records the lowest inference latency (98.3 ms) and computational cost (102.5G MACs), but also consistently yields the best denormalized RMSE scores across all tested forecast horizons (6-hour, 4-day, and 10-day) for key meteorological variables including Z500, T2M, and U10.

% \textbf{KV finetuning on VA-MoE.} \cref{tab:ablation_vamoe} provides the comparison between non-finetuning and kv finetuning strategies on VA-MoE~\cite{vamoe}. Experimental results show that the kv finetuning strategy fully improves the performance of VA-MoE, with 0.0046, 0.0159, 0.0295, and 0.087 improvement in ACC across 1-day, 4-day, 7-day, and 10-day predictions. From the experimental results, we could draw the conclusion that the kv finetuning strategy is not just make sense in our model, but also works well in other method, like VA-MoE.

\textbf{Finetuning strategy on VA‑MoE.} \cref{tab:ablation_vamoe} presents an ablation study comparing the VA‑MoE model~\cite{vamoe} without finetuning and with the proposed accumulative context finetuning strategy. Results show that accumulative finetuning consistently improves forecast accuracy, increasing ACC by 0.0046, 0.0159, 0.0295, and 0.087 for 1‑day, 4‑day, 7‑day, and 10‑day predictions, respectively. These gains demonstrate that the accumulative context finetuning strategy is not only effective in our architecture, but also generalizes well to other frameworks such as VA‑MoE.

% \textbf{Balancing hyperparameter $\lambda$ in kv-cache finetuning. } \cref{tab:ablation_lambda} shows the comparison of different balancing hyperparameter $\lambda$ on the kv-cache finetuning stage. From the experiments, we observe that there is limited impact on the balancing parameter in the short-term forecasts. Even in the long-term forecasts, only when $\lambda=1.0$ the performance has a significant decline. From the experimental results, we draw the conclusion that a complete inclination towards the current token setting will lead to a deterioration in the performance. Thus, we choose $\lambda=0.9$ as the best value to balance the current tokens and the historical tokens.

\textbf{Balancing hyperparameter $\lambda$ in accumulative context finetuning. } \cref{tab:ablation_lambda} compares different values of the balancing hyperparameter $\lambda$ during the accumulative context finetuning stage. Experimental results indicate that the choice of $\lambda$ has only a minor effect on short‑term forecast accuracy. Even for long‑term predictions, a noticeable performance decline occurs only when $\lambda = 1.0$, \textit{i.e.}, when the updated scores rely exclusively on the current token without blending historical information. These observations suggest that completely disregarding historical token information leads to degraded forecast consistency. Therefore, we select $\lambda = 0.9$ as an effective trade‑off, balancing the contributions of current and historical tokens in the cache‑update process.

% \textbf{Cache length in kv-cache finetuning. } \cref{tab:ablation_cachelength} shows the comparison of the cache length in the kv-cache finetuning stage. Just as intuitively, increasing cache length from 3 to 6 will make the prediction results of the model better. The difference between cache length 5 and cache length 6 has limited discrepancy with almost 0.0005 and 0.0013 RMSE improvements in 7-day and 10-day prediction. However, the overall computational load will multiply as the cache length increases. Thus, we choose cache length equal to 5 as the hyperparameter to prune the kv caches.

\textbf{Cache length in accumulative context finetuning. } \cref{tab:ablation_cachelength} presents an ablation study on the cache length during accumulative context finetuning. As expected, increasing the cache length from 3 to 6 yields improved forecast performance. However, the gain between cache lengths 5 and 6 is marginal, with RMSE improvements of only 0.0005 and 0.0013 for 7‑day and 10‑day predictions, respectively. Meanwhile, computational cost grows nearly linearly with cache length. We therefore select a cache length of 5 as a practical trade‑off between predictive accuracy and computational overhead, and adopt this value for pruning the KV values in our experiments.

% \textbf{Number of steps in finetuning. } \cref{tab:ablation_steps} shows an ablation study on the number of steps in the KV-cache finetuning stage with several steps from 5 to 10. There is no doubt that the performance of forecasts are increased with the increase of steps. After the number of step achieving 9, the performance of weather forecasting almost achieve the stable stage and do not increase again. Considering the computation cost increasing with the steps, we pick 10 as the finetuning steps in this work.

\textbf{Number of steps in finetuning. } \cref{tab:ablation_steps} presents an ablation study on the number of steps used in the accumulative context finetuning stage, with step counts ranging from 5 to 10. As expected, forecast performance improves as the number of steps increases. After reaching 9 steps, however, the weather forecasting accuracy stabilizes and shows no further gains. Given that computational cost grows with the number of steps, we select 10 steps as a practical compromise between performance and efficiency for the experiments in this work.

% \textbf{Window sizes in window attention. } \cref{tab:ablation_windows} presents an ablation study on the window size, including $4\times4$ and hybrid sizes. From the previous experiments, we find that the model with a single window size may causes the prediction results with several grids, which is inconsistent with the smooth distribution of the atmospheric variables on the global face. The hybrid sizes has a better performance than a single size, and avoids the grids effect on the prediction results.

\textbf{Window sizes in window attention. } \cref{tab:ablation_windows} presents an ablation study comparing a fixed \(4 \times 4\) window size with a hybrid configuration. Previous experiments indicate that using a single window size can introduce grid‑like artifacts in the predictions, which is inconsistent with the smooth spatial distribution of atmospheric variables on a global scale. In contrast, hybrid window sizes yield better performance and effectively mitigate these gridding artifacts.

\begin{figure*}[t] % 顶部位置
    \centering
    \begin{minipage}{\textwidth}
        \centering
        % 第一行：2张图片
        \begin{subfigure}{0.5\textwidth}
            \centering
            \captionsetup{justification=centering}
            \caption{Typhoon AMPIL (2024.08)}
            \includegraphics[width=1.0\linewidth]{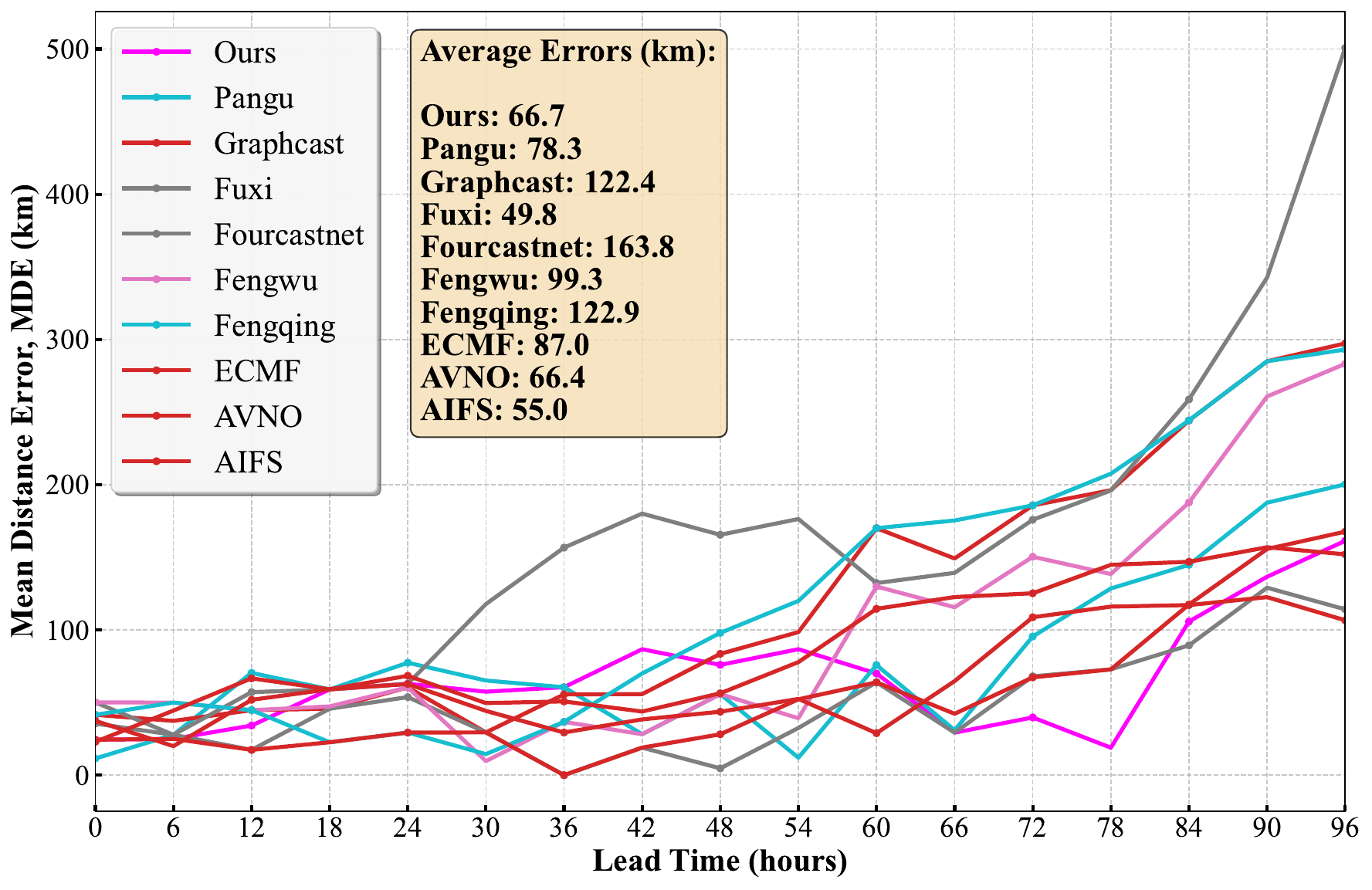}
            \label{fig:tcs_ampil}
        \end{subfigure}%
        \hfill
        \begin{subfigure}{0.5\textwidth}
            \centering
            \captionsetup{justification=centering}
            \caption{Typhoon BEBINCA (2024.09)}
            \includegraphics[width=1.0\linewidth]{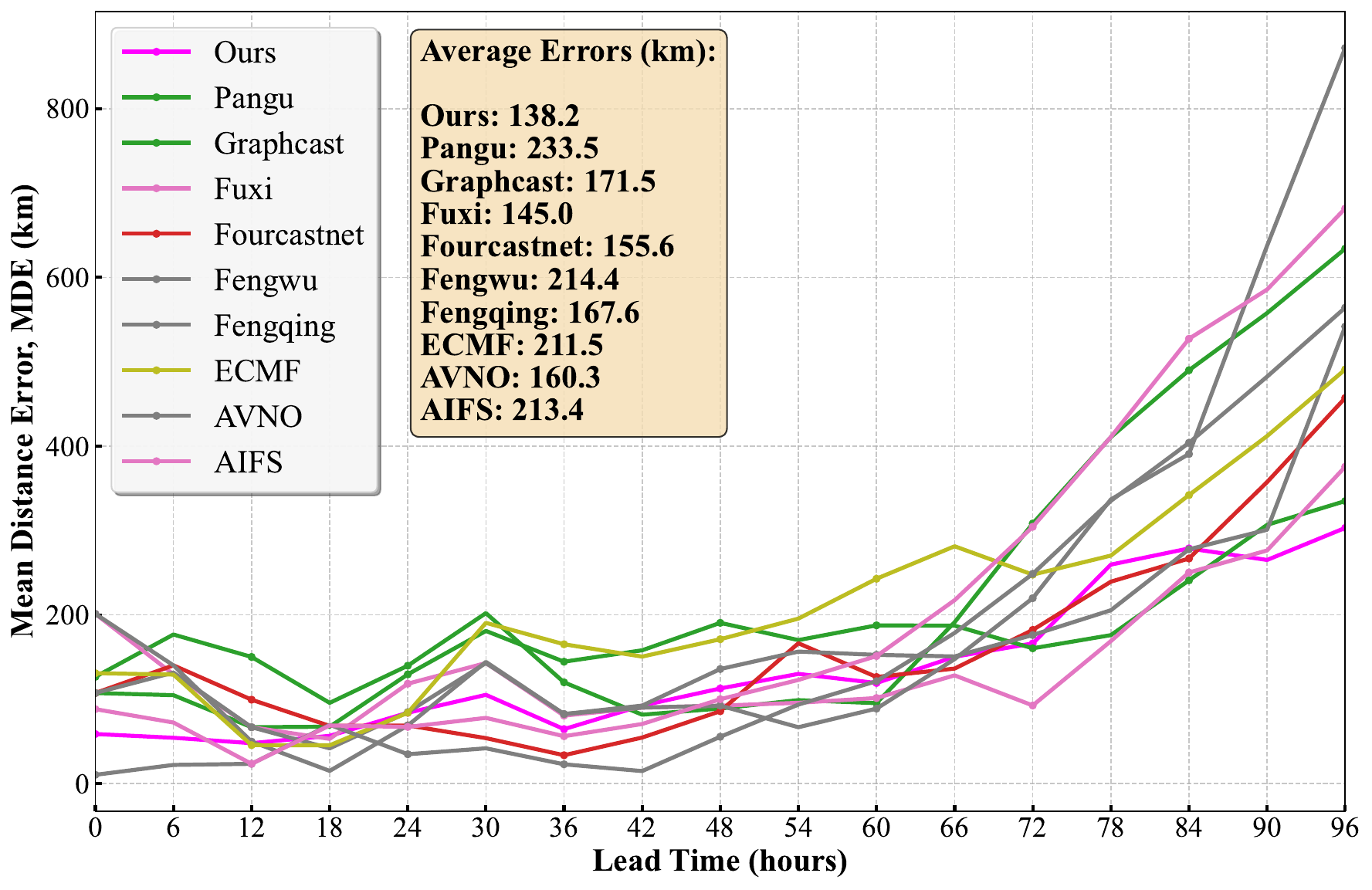}
            \label{fig:tcs_bebinca}
        \end{subfigure}
        
        \vspace{3mm}
        % 第二行：2张图片
        \begin{subfigure}{0.5\textwidth}
            \centering
            \captionsetup{justification=centering}
            \caption{Typhoon EWINIAR (2024.05)}
            \includegraphics[width=1.0\linewidth]{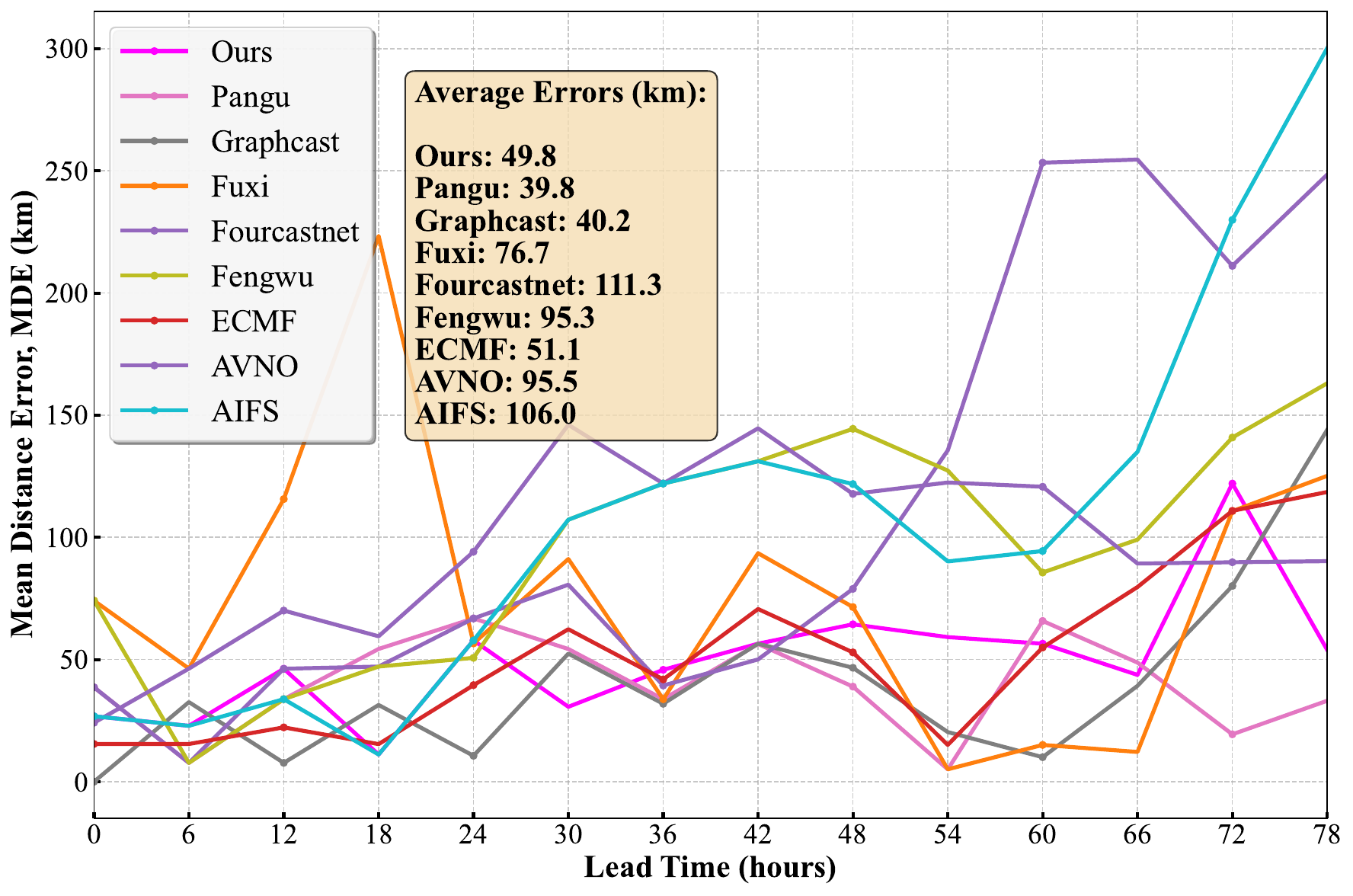}
            \label{fig:tcs_ewiniar}
        \end{subfigure}%
        \hfill
        \begin{subfigure}{0.5\textwidth}
            \centering
            \captionsetup{justification=centering}
            \caption{Typhoon GAEMI (2024.07)}
            \includegraphics[width=1.0\linewidth]{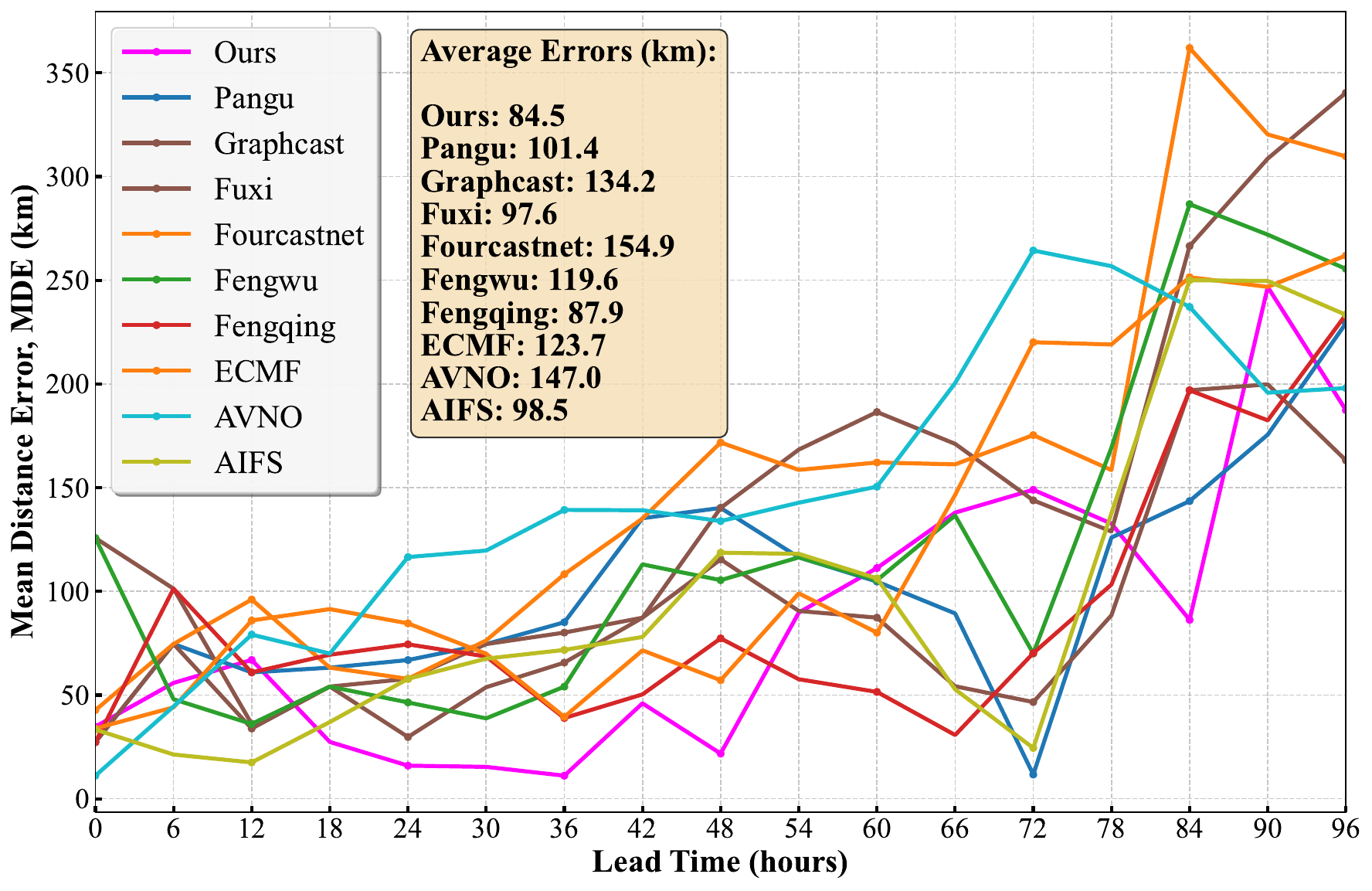}
            \label{fig:tcs_gaemi}
        \end{subfigure}
        
        \vspace{3mm}
        % 第三行：2张图片
        \begin{subfigure}{0.5\textwidth}
            \centering
            \captionsetup{justification=centering}
            \caption{Typhoon KONG-REY (2024.10)}
            \includegraphics[width=1.0\linewidth]{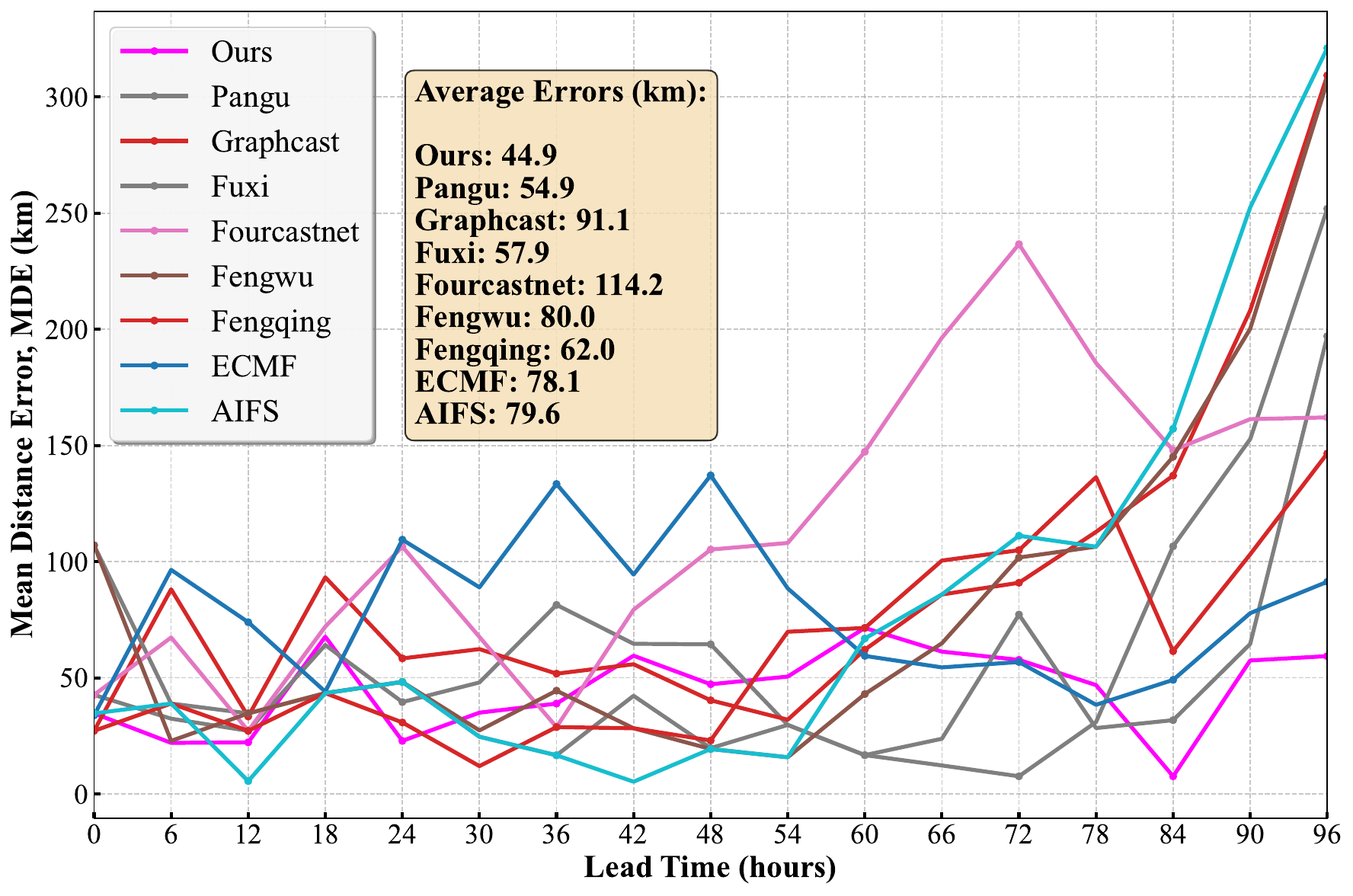}
            \label{fig:tcs_kongrey}
        \end{subfigure}%
        \hfill
        \begin{subfigure}{0.5\textwidth}
            \centering
            \captionsetup{justification=centering}
            \caption{Typhoon KRATHON (2024.09)}
            \includegraphics[width=1.0\linewidth]{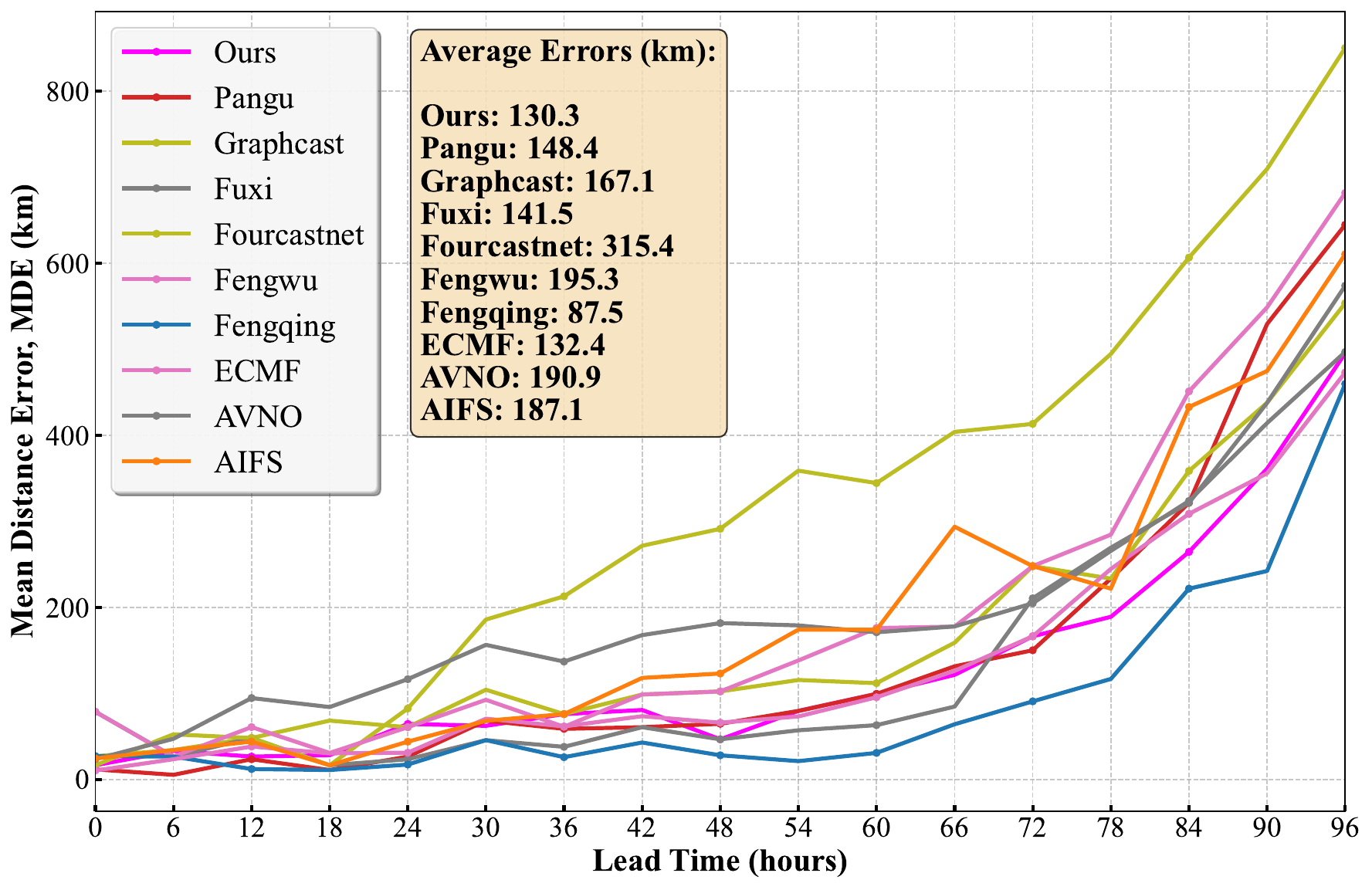}
            \label{fig:tcs_krathon}
        \end{subfigure}
        
        \vspace{2mm}
        
    \end{minipage}
    \caption{Typhoon Track Assessment (Part 1 of 2). Comparative analysis of Mean Distance Error (MDE, in kilometers $\downarrow$) for six typhoons. (a)-(f) show the prediction accuracy across different lead times.}
    \label{fig:tcs_part1}
\end{figure*}

% 第二部分：下半部分图片，使用[!t]确保紧跟在上半部分之后
\begin{figure*}[!t] % 强制顶部位置，紧跟上一部分
    \centering
    \ContinuedFloat
    \begin{minipage}{\textwidth}
        \centering
        % 第四行：2张图片
        \begin{subfigure}{0.5\textwidth}
            \centering
            \captionsetup{justification=centering}
            \caption{Typhoon MAN-YI (2024.11)}
            \includegraphics[width=1.0\linewidth]{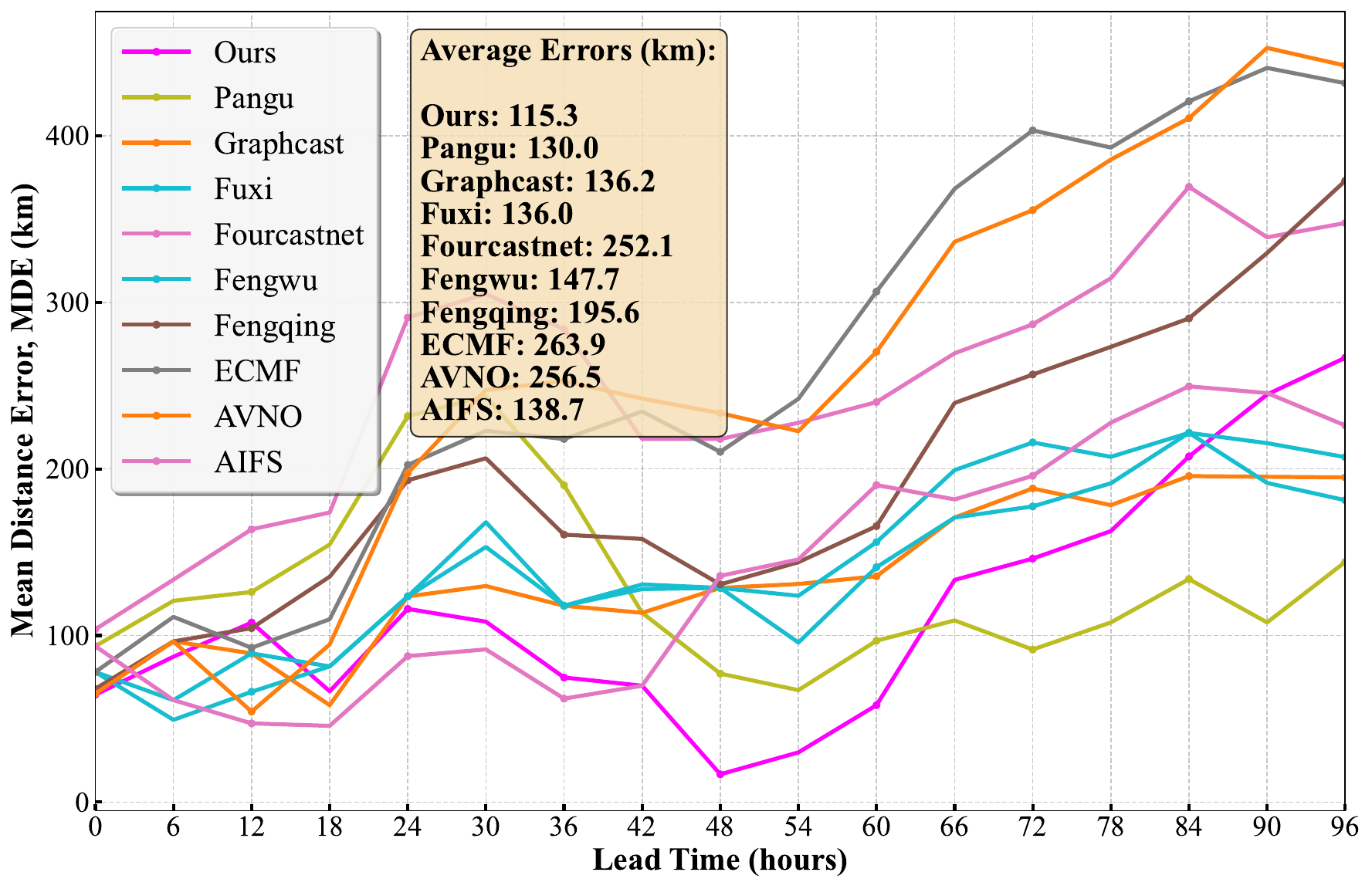}
            \label{fig:tcs_manyi}
        \end{subfigure}%
        \hfill
        \begin{subfigure}{0.5\textwidth}
            \centering
            \captionsetup{justification=centering}
            \caption{Typhoon SHANSHAN (2024.08)}
            \includegraphics[width=1.0\linewidth]{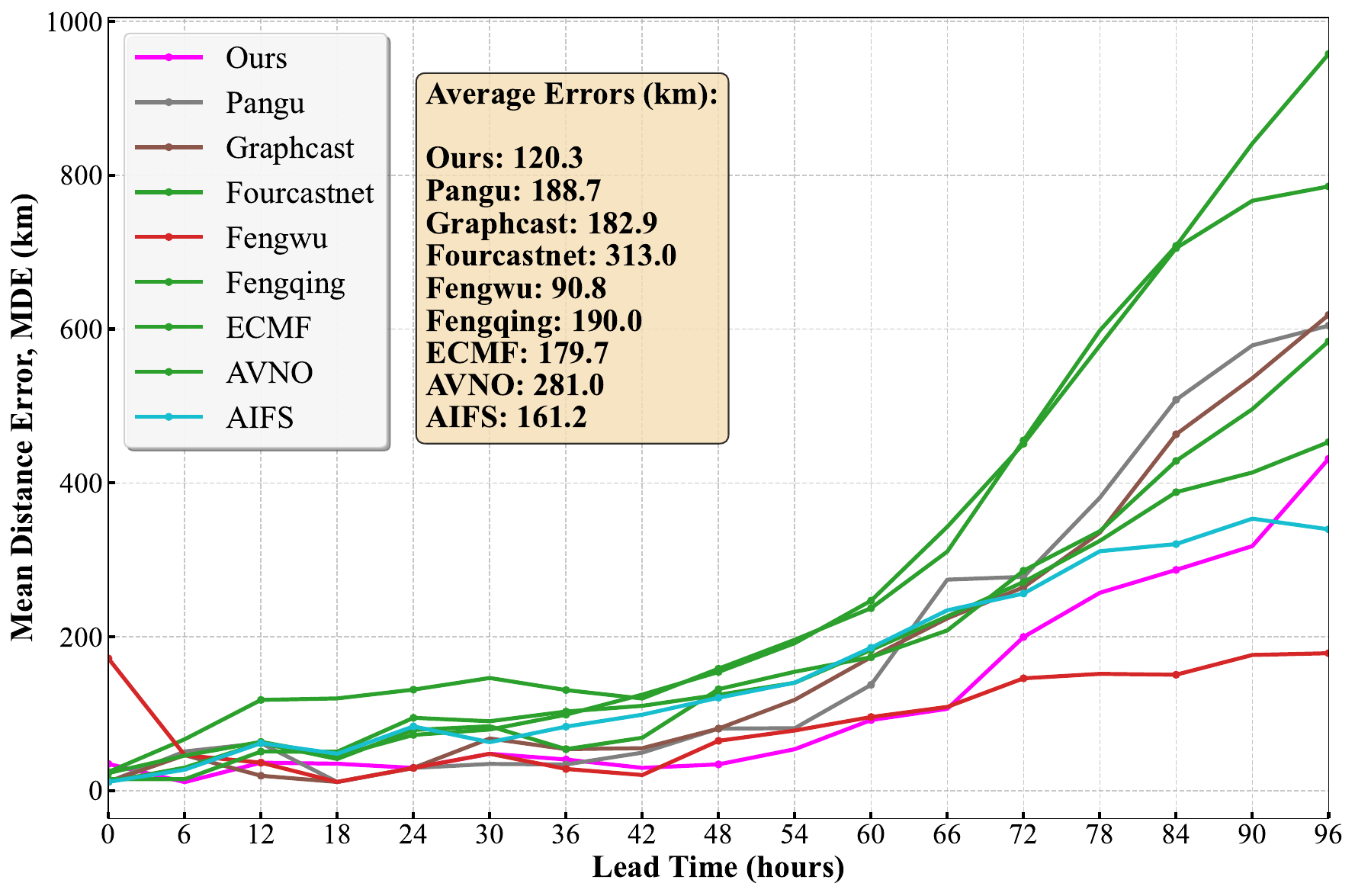}
            \label{fig:tcs_shanshan}
        \end{subfigure}
        
        \vspace{3mm}
        % 第五行：2张图片
        \begin{subfigure}{0.5\textwidth}
            \centering
            \captionsetup{justification=centering}
            \caption{Typhoon YAGI (2024.09)}
            \includegraphics[width=1.0\linewidth]{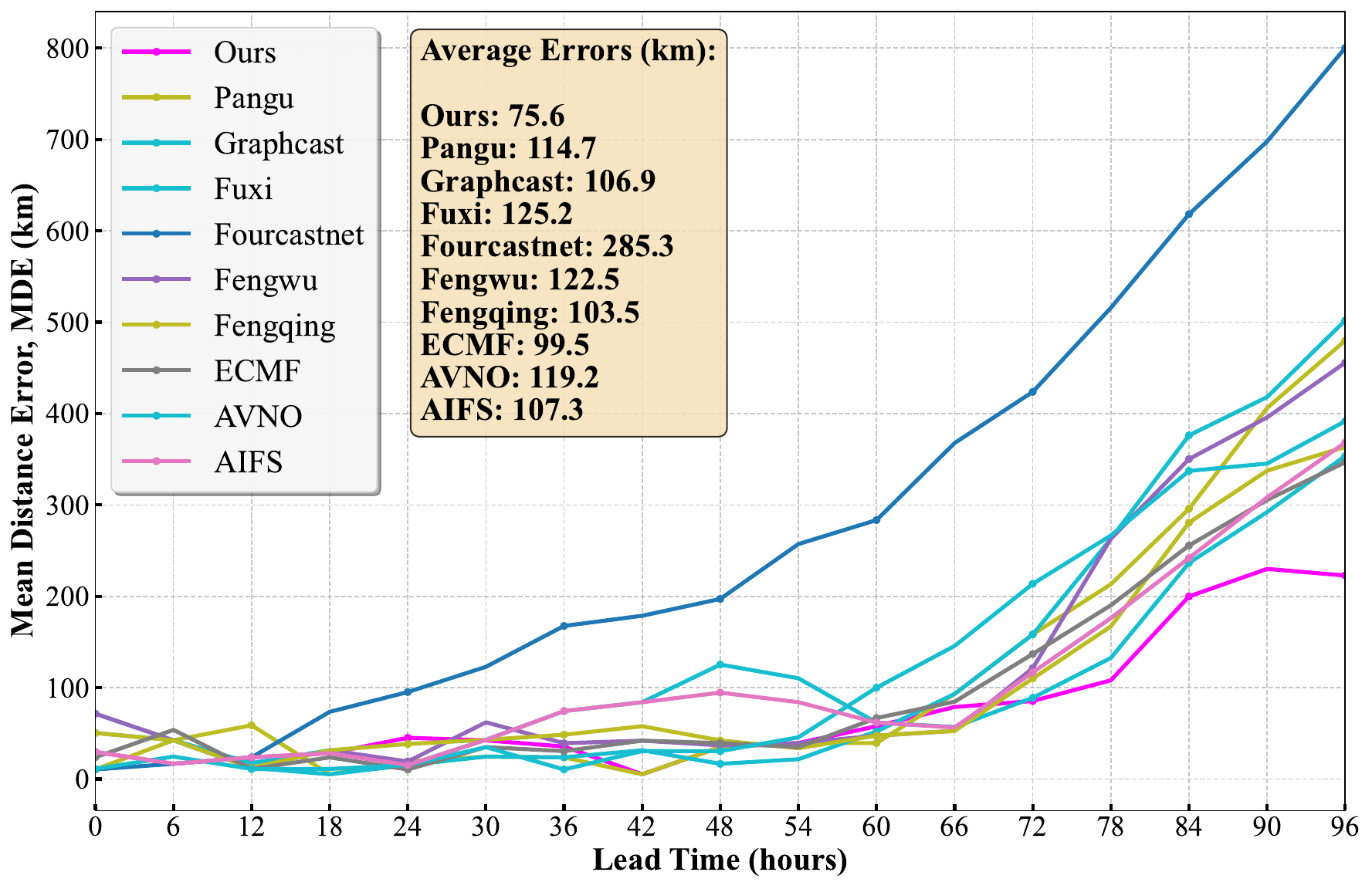}
            \label{fig:tcs_yagi}
        \end{subfigure}%
        \hfill
        \begin{subfigure}{0.5\textwidth}
            \centering
            \captionsetup{justification=centering}
            \caption{Typhoon YINXING (2024.11)}
            \includegraphics[width=1.0\linewidth]{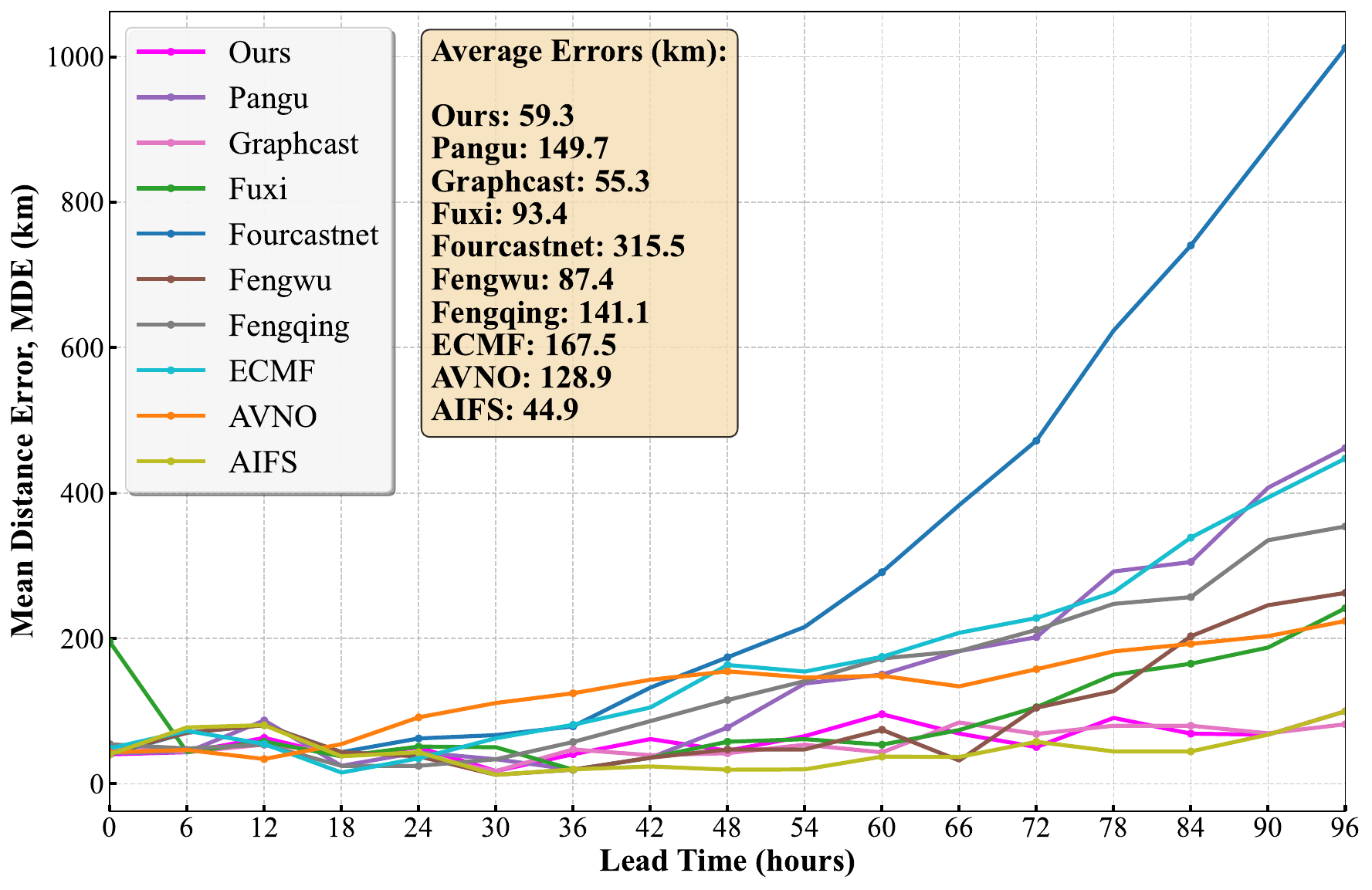}
            \label{fig:tcs_yinxing}
        \end{subfigure}
        
        \vspace{2mm}
    \end{minipage}
    \caption{(Continued) Typhoon Track Assessment (Part 2 of 2). (g)-(j) continue the comparative analysis for the remaining four typhoons. The complete assessment demonstrates consistent performance patterns across different weather conditions and seasons.}
    \label{fig:tcs_part2}
\end{figure*}

\begin{figure*}[t] % 使用figure*环境确保图片出现在页面顶部
    \centering
    \begin{minipage}{\columnwidth} % 限制为单栏宽度
        % 第一行：3张图片
        \centering
        \begin{subfigure}{0.56\textwidth}
            \centering
            \captionsetup{justification=centering}
            \includegraphics[width=1.0\linewidth]{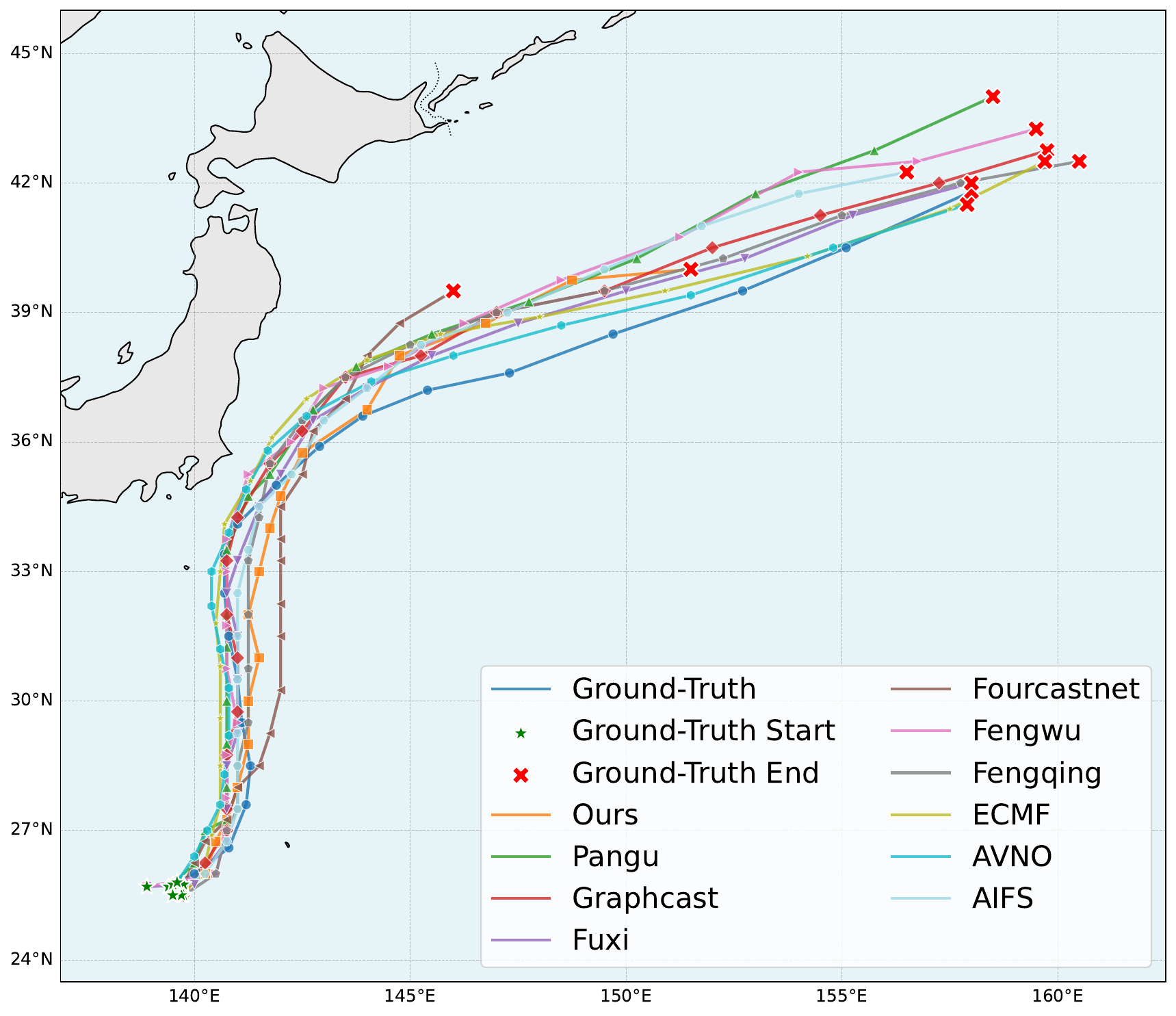}
            \caption{Typhoon AMPIL (2024.08)}
            \label{fig:tctrack_ewiniar}
        \end{subfigure}%
        % \hspace{0.3\textwidth} % 居中占位
        \hfill
        \begin{subfigure}{0.42\textwidth}
            \centering
            \captionsetup{justification=centering}
            \includegraphics[width=1.0\linewidth]{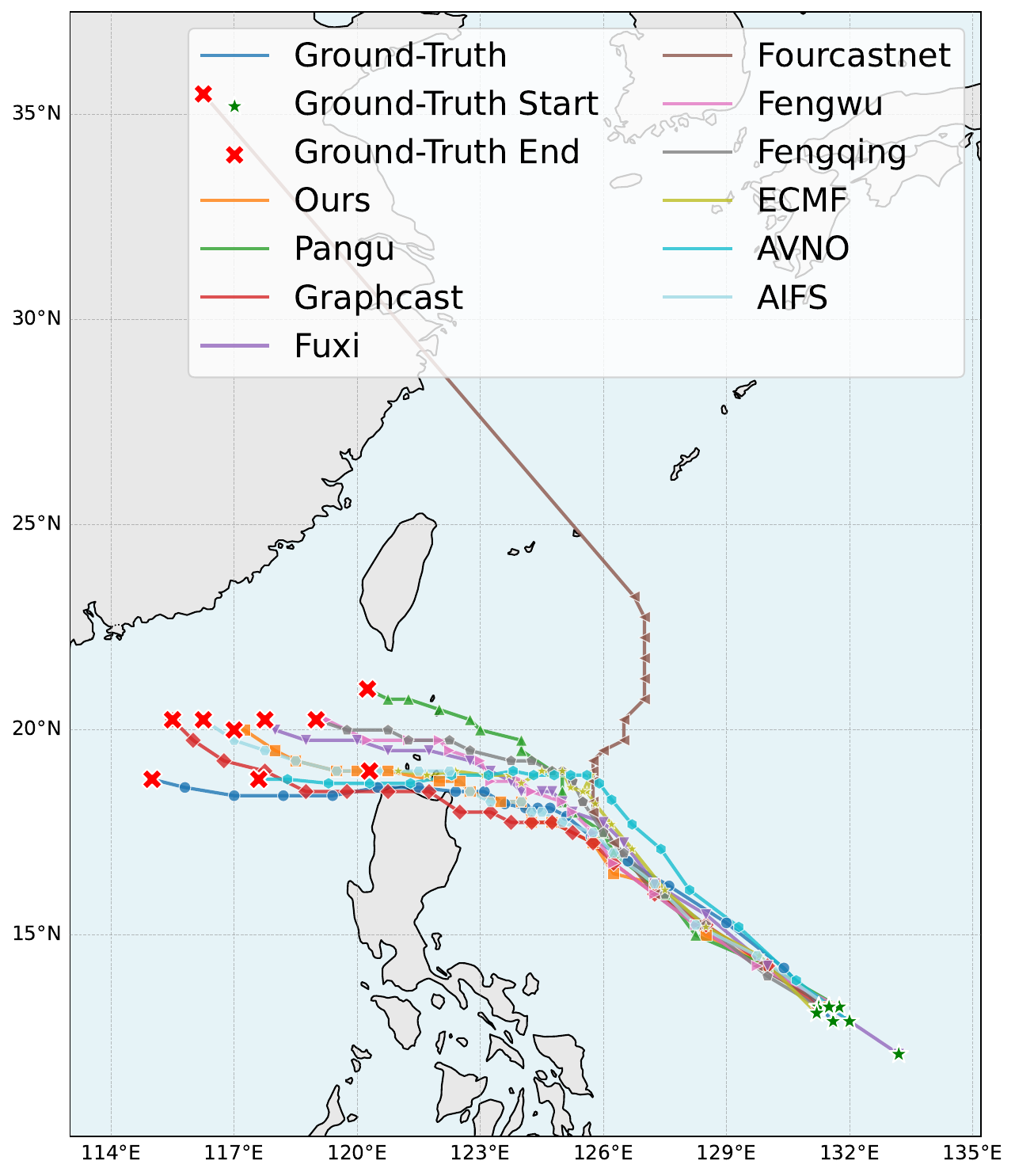}
            \caption{Typhoon Yinxing (2024.11)}
            \label{fig:tctrack_yinxing}
        \end{subfigure}

        \begin{subfigure}{0.33\textwidth}
            \centering
            \captionsetup{justification=centering}
            \includegraphics[width=1.0\linewidth]{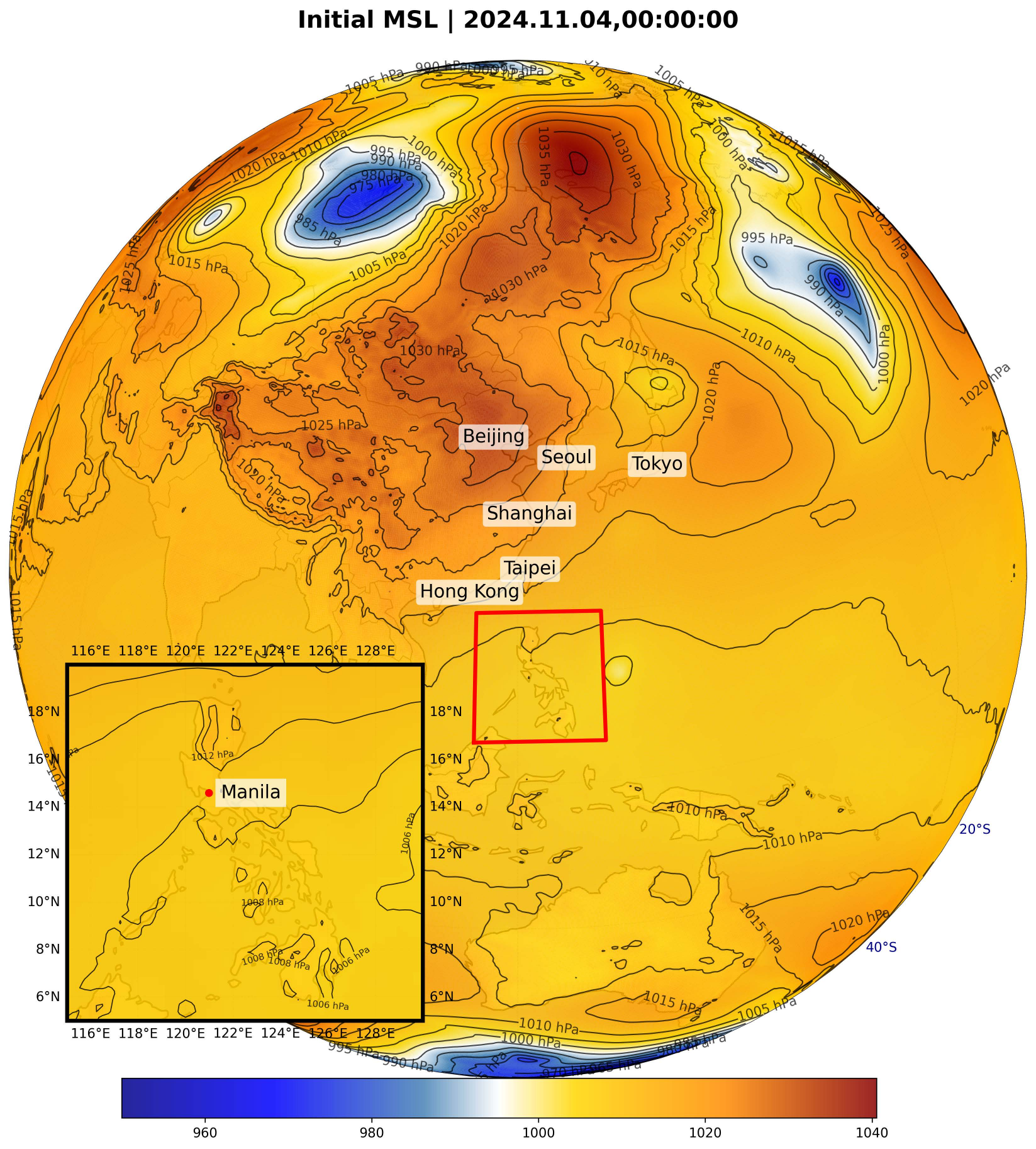}
            \caption{Initial MSL\\2024.11.04, T+0}
            \label{fig:tcs_initial}
        \end{subfigure}%
        \hfill
        \begin{subfigure}{0.33\textwidth}
            \centering
            \captionsetup{justification=centering}
            \includegraphics[width=1.0\linewidth]{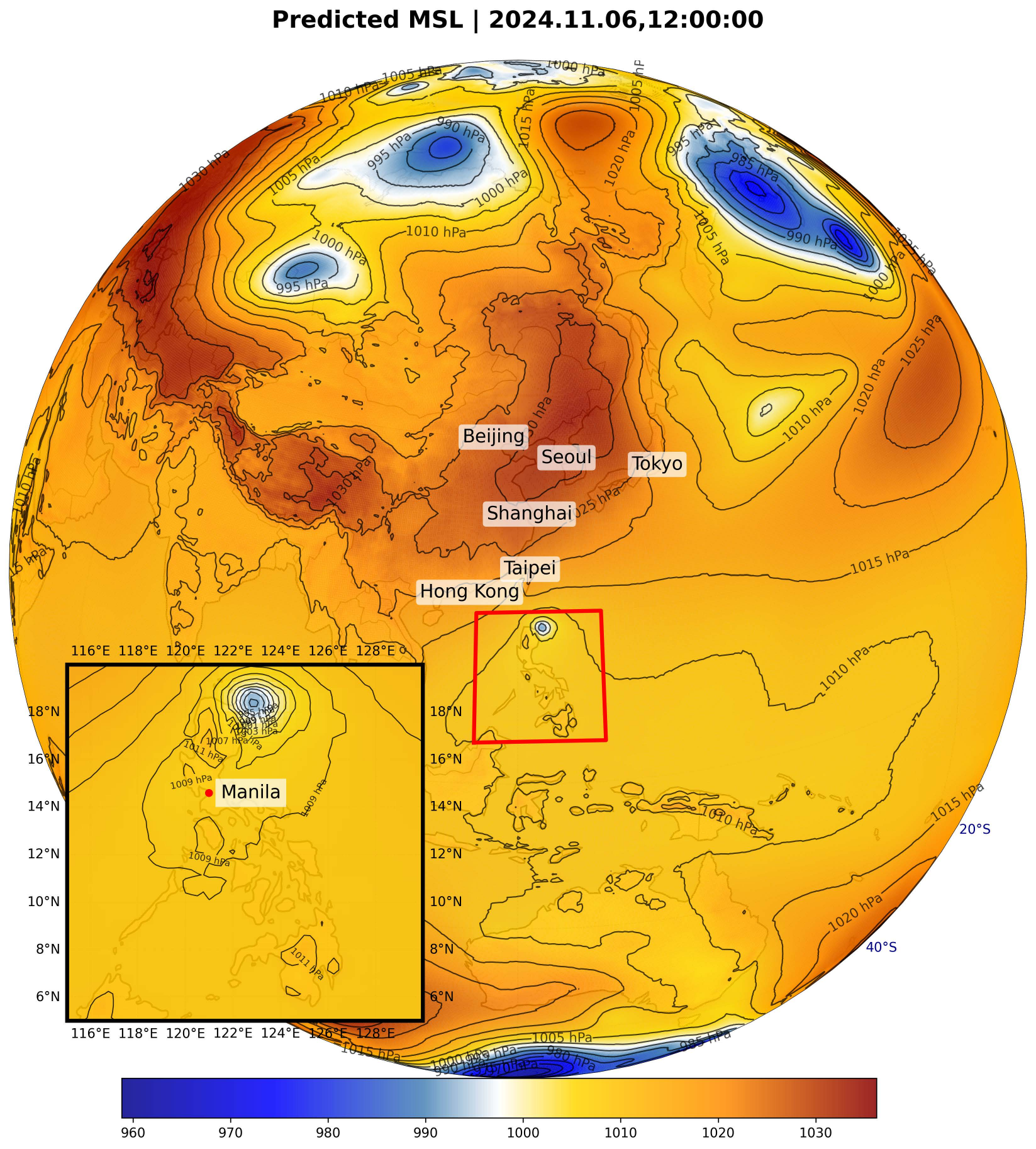}
            \caption{Predicted MSL\\2024.11.06, T+12}
            \label{fig:tcs_predict}
        \end{subfigure}%
        \hfill
        \begin{subfigure}{0.33\textwidth}
            \centering
            \captionsetup{justification=centering}
            \includegraphics[width=1.0\linewidth]{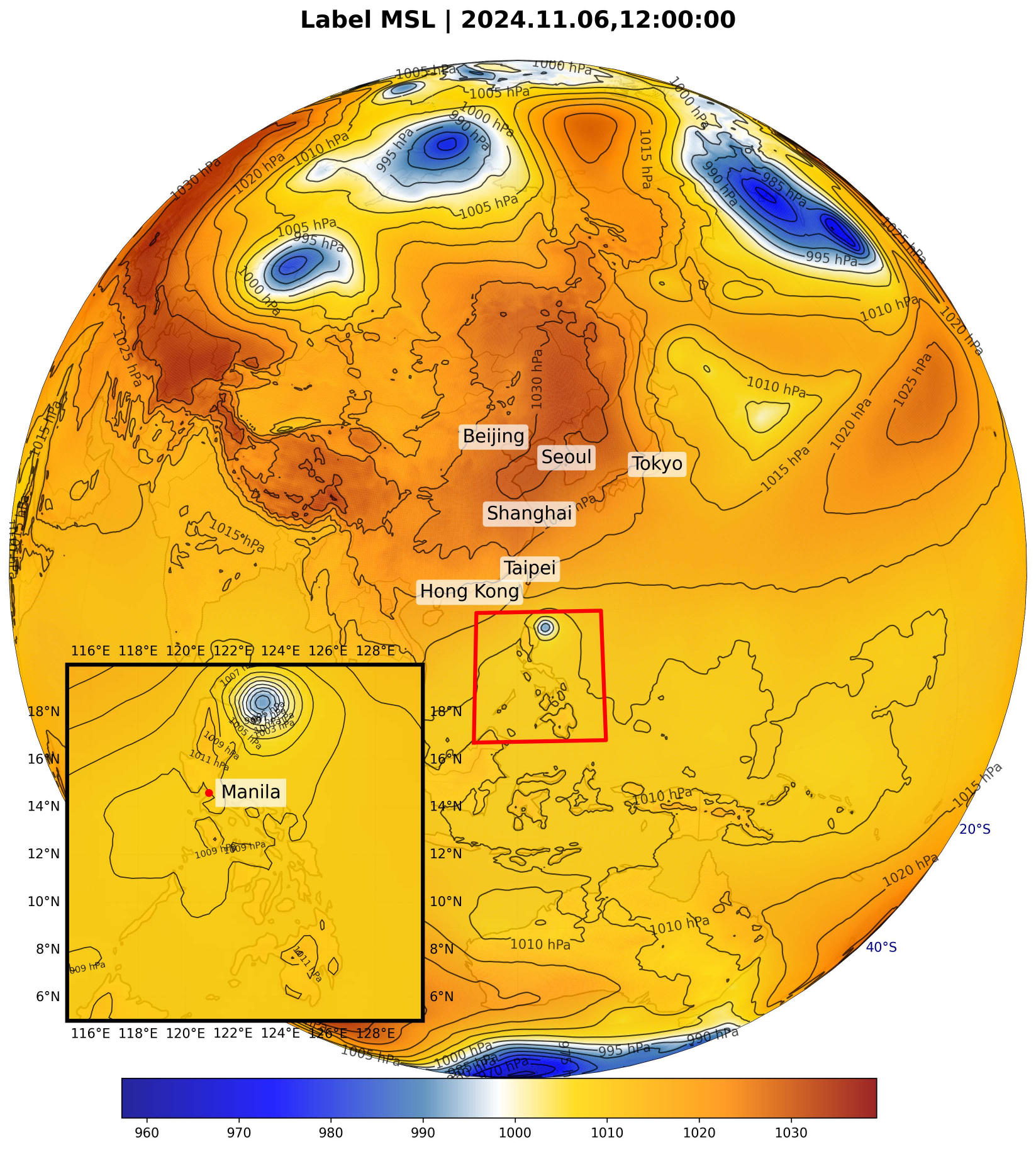}
            \caption{Real MSL\\2024.11.06, T+12}
            \label{fig:tcs_real}
        \end{subfigure}

        \vspace{-1mm}
        % 整体大标题
        \caption{ Typhoon Track Assessment. (a) and (b) present a five-day comparative analysis of Mean Distance Error (MDE, in kilometers $\downarrow$) for Typhoons AMPIL and Yinxing, respectively. (c), (d), and (e) visualize the evolution of Mean Sea-Level Pressure (MSL) across three temporal stages: initial conditions, 60-hour predictions, and corresponding ground-truth observations. }
        
        \vspace{-3mm}
        \label{fig:tctrack}
    \end{minipage}
\end{figure*}

% \noindent
% % \footnotesize
% \dirtree{%
% .1 project\_root/.
% .2 code/.
% .3 config/.
% .4 EMFormer.yaml.
% .4 GraphCast.yaml.
% .4 OneForecast.yaml.
% .3 networks/.
% .4 EMFormer.py.
% .4 \_\_init\_\_.py.
% .4 l2\_loss.py.
% .4 test\_triple\_conv\_cuda.py.
% .3 utils/.
% .4 YParams.py.
% .4 data\_loader\_npyfiles.py.
% .4 fileio.py.
% .4 logging\_utils.py.
% .4 metrics.py.
% .4 weighted\_acc\_rmse.py.
% .3 inference.py.
% .3 test.py.
% .3 test.sh.
% .2 logs/.
% .3 typhoon/.
% .4 AMPIL\_trajectories.txt.
% .4 BEBINCA\_trajectories.txt.
% .4 Ewiniar\_trajectories.txt.
% .4 GAEMI\_trajectories.txt.
% .4 KONG-REY\_trajectories.txt.
% .4 KRATHON\_trajectories.txt.
% .4 MAN-YI\_trajectories.txt.
% .4 SHANSHAN\_trajectories.txt.
% .4 YAGI\_trajectories.txt.
% .4 Yinxing\_trajectories.txt.
% .3 vision\_task/.
% .4 base.log.
% .4 small.log.
% .4 tiny.log.
% .3 weather/.
% .4 emformer.log.
% .4 graphcast.log.
% .4 oneforecast.log.
% .4 vamoe.log.
% }
% % \normalsize

\subsection{Visualization}
\label{subsec:add_vis}

% We provide the visualization about semantic segmentation on ADE20K, comparing ground-truth annotation with the predicted results from our small and base models in \cref{fig:vis_seg}.

We present visualizations of semantic segmentation results on ADE20K, comparing ground‑truth annotations with the predictions generated by EMFormer small and base models in \cref{fig:vis_seg}.

We also provide more visualization about global weather forecasting of 6-hour, 1-day, 2-day, 3-day, 5-day, 7-day, and 10-day in \cref{fig:global0}, \cref{fig:global3}, \cref{fig:global7}, \cref{fig:global11}, \cref{fig:global19}, \cref{fig:global27}, and \cref{fig:global39}.

\section{Supplementary Materials}
\label{supple}

To further ensure the transparency and reproducibility of our experimental results, we have included supplementary materials containing: partial test codes for the implemented models, complete training logs for both weather prediction and image classification tasks, and the raw latitude-longitude coordinates used for visualizing typhoon trajectories.

\begin{figure*}[t]
    \centering
    \begin{overpic}[width=0.7\linewidth]{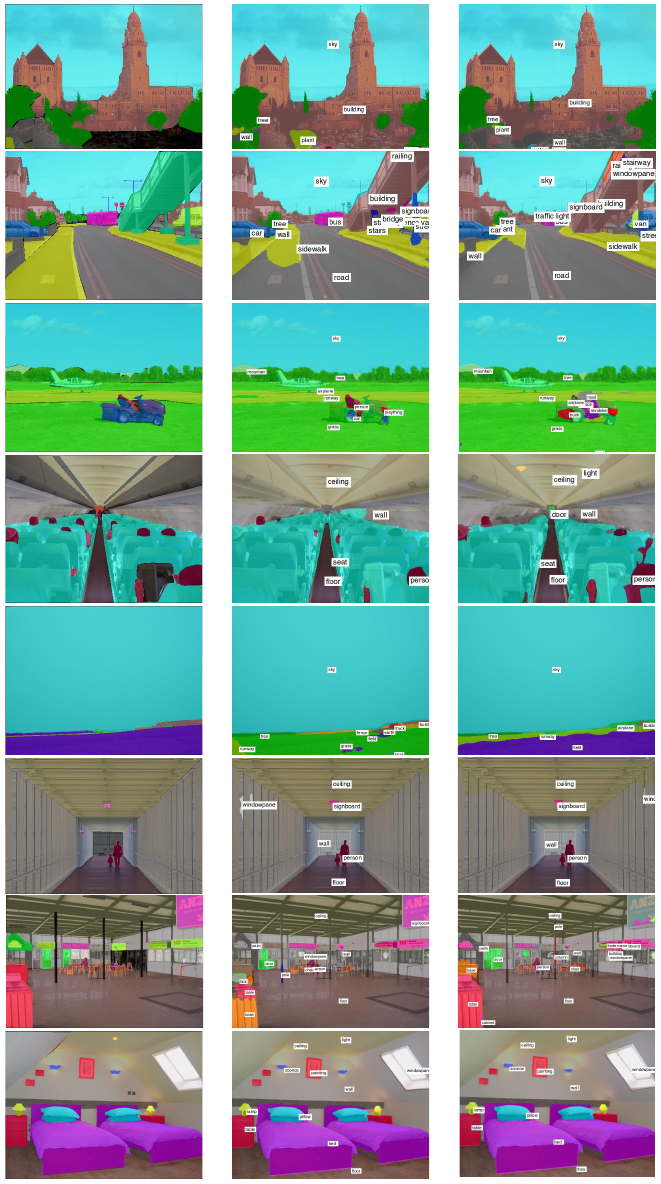} 

    \put(5, 100.3 ){\small Ground Truth}
    \put(22.6, 100.3 ){\small EMFormer-Small}
    \put(42, 100.3 ){\small EMFormer-Base}

    \end{overpic}
    \caption{Visualization of segmentation results on ADE20K, comparing ground‑truth annotations with predictions from our small and base models. }
    \label{fig:vis_seg}
\end{figure*}

\begin{figure*}[t]
    \centering
    \begin{overpic}[width=\linewidth]{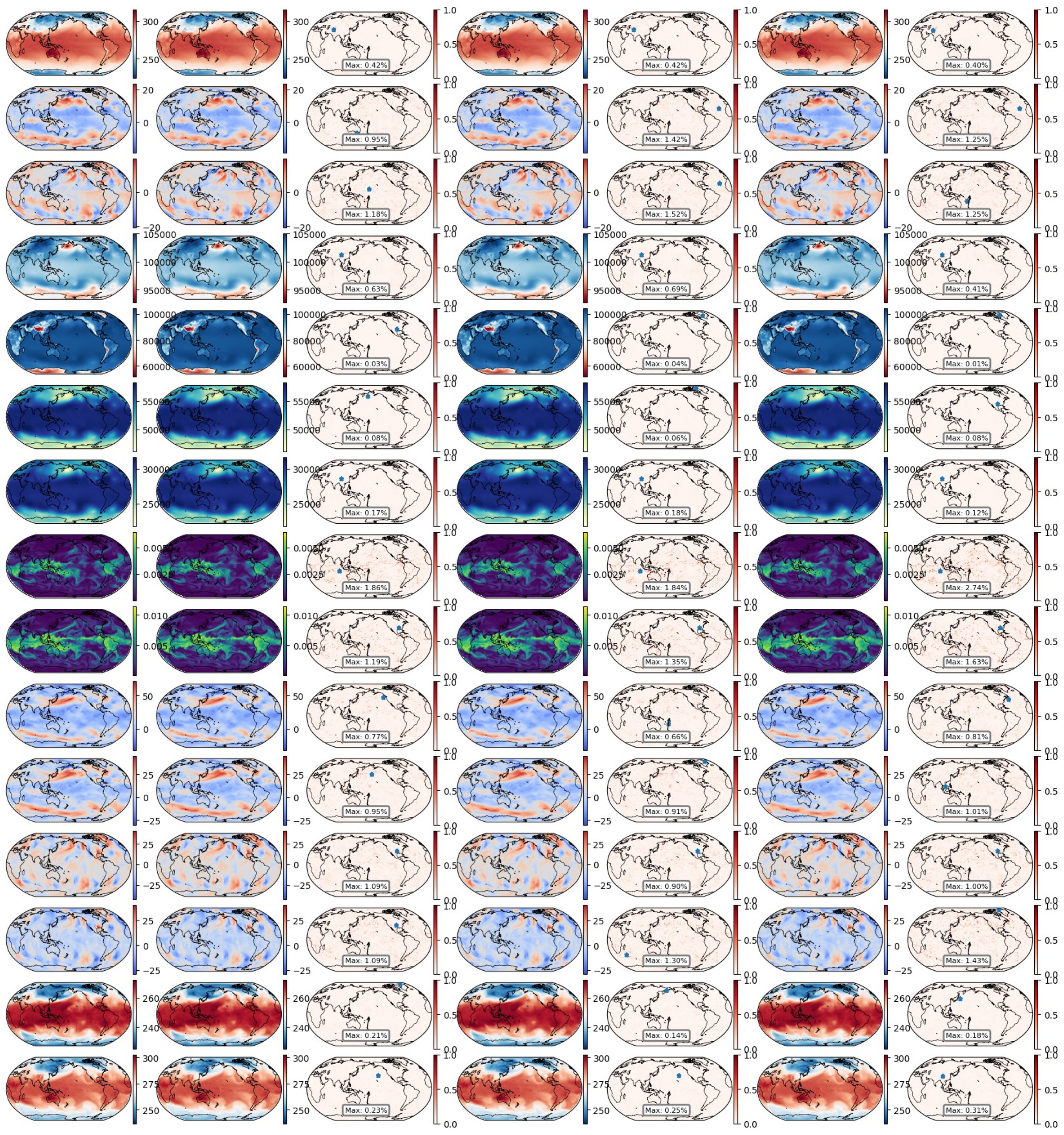} 

    \put(5, 99.3){\scriptsize GT}
    \put(16.4, 99.3){\scriptsize Graphcast}
    \put(27.8, 99.3){\scriptsize Error@Graphcast}
    \put(42.4, 99.3){\scriptsize Oneforecast}
    \put(53.3, 99.3){\scriptsize Error@Oneforecast}
    \put(70.5, 99.3){\scriptsize Ours}
    \put(81.7, 99.3){\scriptsize Error@Ours}

    \put(-0.5, 68.4){\scriptsize \rotatebox{90}{ SP }}
    \put(-0.5, 74.5){\scriptsize \rotatebox{90}{ MSL }}
    \put(-0.5, 81.3){\scriptsize \rotatebox{90}{ V10 }}
    \put(-0.5, 88.1){\scriptsize \rotatebox{90}{ U10 }}
    \put(-0.5, 94.9){\scriptsize \rotatebox{90}{ T2M }}

    \put(-0.5, 34.8){\scriptsize \rotatebox{90}{ U500 }}
    \put(-0.5, 41.2){\scriptsize \rotatebox{90}{ Q700 }}
    \put(-0.5, 47.9){\scriptsize \rotatebox{90}{ Q500 }}
    \put(-0.5, 54.5){\scriptsize \rotatebox{90}{ Z700 }}
    \put(-0.5, 61){\scriptsize \rotatebox{90}{ Z500 }}
    
    \put(-0.5, 2){\scriptsize \rotatebox{90}{ T850 }}
    \put(-0.5, 9){\scriptsize \rotatebox{90}{ T500 }}
    \put(-0.5, 15.7){\scriptsize \rotatebox{90}{ V700 }}
    \put(-0.5, 22.3){\scriptsize \rotatebox{90}{ V500 }}
    \put(-0.5, 28.9){\scriptsize \rotatebox{90}{ U700 }}
    
    \end{overpic}
    \caption{6-hour forecast results of global weather among different models. }
    \label{fig:global0}
\end{figure*}

% \begin{figure*}[t]
%   \centering
%    \begin{overpic}[width=\linewidth]{visualizes/fig_step1.pdf}
%         \put(5.5,44.2){\scriptsize GT}     \put(17,44.2){\scriptsize Graphcast}     \put(29,44.2){\scriptsize Error@Graphcast}     \put(45,44.2){\scriptsize Oneforecast}     \put(57,44.2){\scriptsize Error@Oneforecast}     \put(75.5,44.2){\scriptsize Ours}     \put(87.5,44.2){\scriptsize Error@Ours}          \put(-0.5, 3){\scriptsize \rotatebox{90}{ SP }}     \put(-0.5, 12){\scriptsize \rotatebox{90}{ MSL }}     \put(-0.5, 20){\scriptsize \rotatebox{90}{ V10 }}     \put(-0.5, 29){\scriptsize \rotatebox{90}{ U10 }}     \put(-0.5, 38){\scriptsize \rotatebox{90}{ T2M }}          
        
%     \end{overpic}     
%     \caption{0.5-day forecast results of global weather among different models.} 
%    \label{fig:global1}
%    \vspace{-0.2cm}
% \end{figure*}

\begin{figure*}[t]
  \centering
   \begin{overpic}[width=\linewidth]{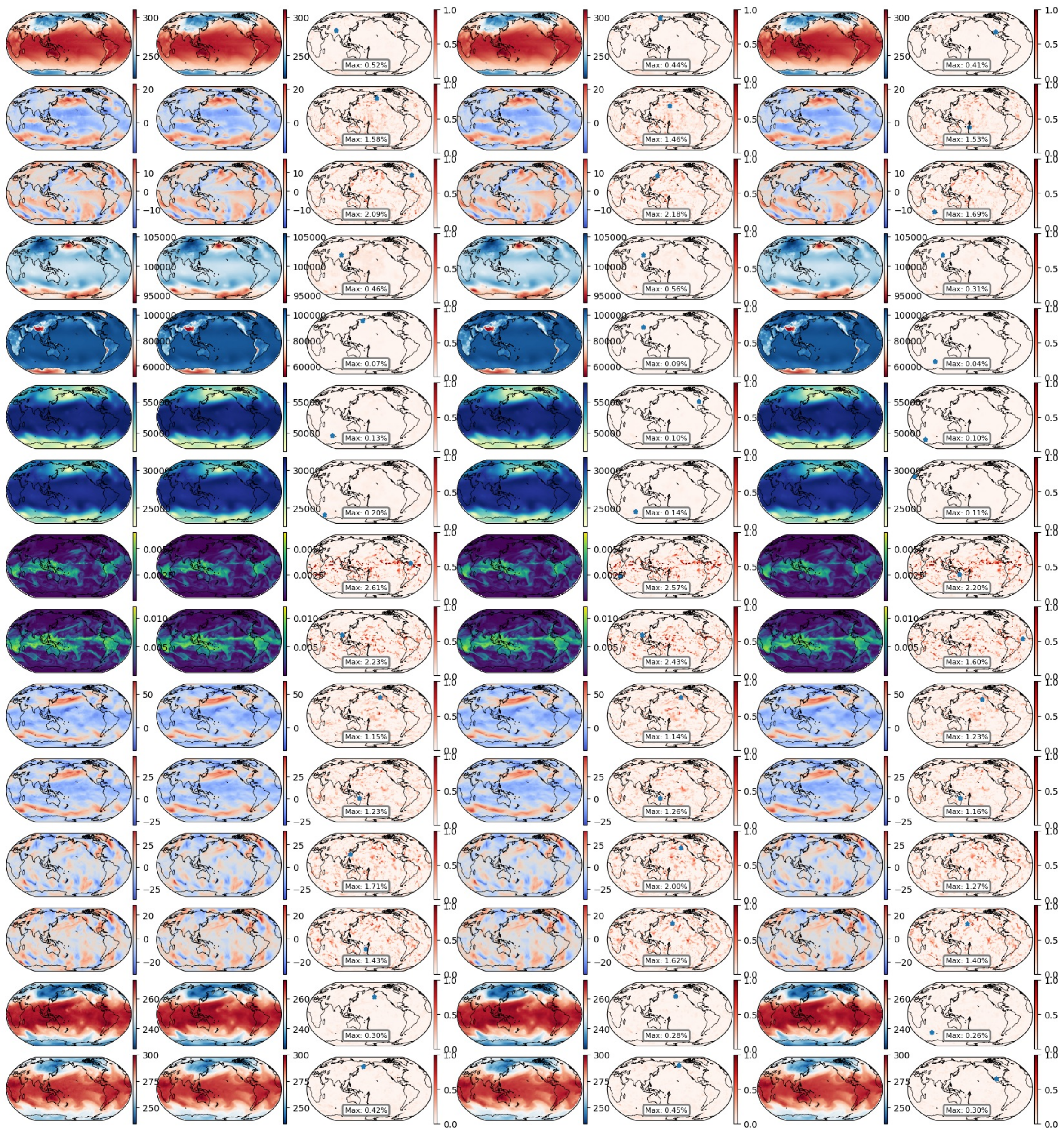}
        \put(5, 99.3){\scriptsize GT}
    \put(16.4, 99.3){\scriptsize Graphcast}
    \put(27.8, 99.3){\scriptsize Error@Graphcast}
    \put(42.4, 99.3){\scriptsize Oneforecast}
    \put(53.3, 99.3){\scriptsize Error@Oneforecast}
    \put(70.5, 99.3){\scriptsize Ours}
    \put(81.7, 99.3){\scriptsize Error@Ours}

    \put(-0.5, 68.4){\scriptsize \rotatebox{90}{ SP }}
    \put(-0.5, 74.5){\scriptsize \rotatebox{90}{ MSL }}
    \put(-0.5, 81.3){\scriptsize \rotatebox{90}{ V10 }}
    \put(-0.5, 88.1){\scriptsize \rotatebox{90}{ U10 }}
    \put(-0.5, 94.9){\scriptsize \rotatebox{90}{ T2M }}

    \put(-0.5, 34.8){\scriptsize \rotatebox{90}{ U500 }}
    \put(-0.5, 41.2){\scriptsize \rotatebox{90}{ Q700 }}
    \put(-0.5, 47.9){\scriptsize \rotatebox{90}{ Q500 }}
    \put(-0.5, 54.5){\scriptsize \rotatebox{90}{ Z700 }}
    \put(-0.5, 61){\scriptsize \rotatebox{90}{ Z500 }}
    
    \put(-0.5, 2){\scriptsize \rotatebox{90}{ T850 }}
    \put(-0.5, 9){\scriptsize \rotatebox{90}{ T500 }}
    \put(-0.5, 15.7){\scriptsize \rotatebox{90}{ V700 }}
    \put(-0.5, 22.3){\scriptsize \rotatebox{90}{ V500 }}
    \put(-0.5, 28.9){\scriptsize \rotatebox{90}{ U700 }}        
        \end{overpic}     \caption{1-day forecast results of global weather among different models.} 
   \label{fig:global3}
   \vspace{-0.2cm}
\end{figure*}

% \begin{figure*}[t]
%   \centering
%    \begin{overpic}[width=\linewidth]{visualizes/fig_step5.pdf}
%         \put(5.5,44.2){\scriptsize GT}     \put(17,44.2){\scriptsize Graphcast}     \put(29,44.2){\scriptsize Error@Graphcast}     \put(45,44.2){\scriptsize Oneforecast}     \put(57,44.2){\scriptsize Error@Oneforecast}     \put(75.5,44.2){\scriptsize Ours}     \put(87.5,44.2){\scriptsize Error@Ours}          \put(-0.5, 3){\scriptsize \rotatebox{90}{ SP }}     \put(-0.5, 12){\scriptsize \rotatebox{90}{ MSL }}     \put(-0.5, 20){\scriptsize \rotatebox{90}{ V10 }}     \put(-0.5, 29){\scriptsize \rotatebox{90}{ U10 }}     \put(-0.5, 38){\scriptsize \rotatebox{90}{ T2M }}          \end{overpic}     \caption{1.5-day forecast results of global weather among different models.} 
%    \label{fig:global5}
%    \vspace{-0.2cm}
% \end{figure*}

\begin{figure*}[t]
  \centering
   \begin{overpic}[width=\linewidth]{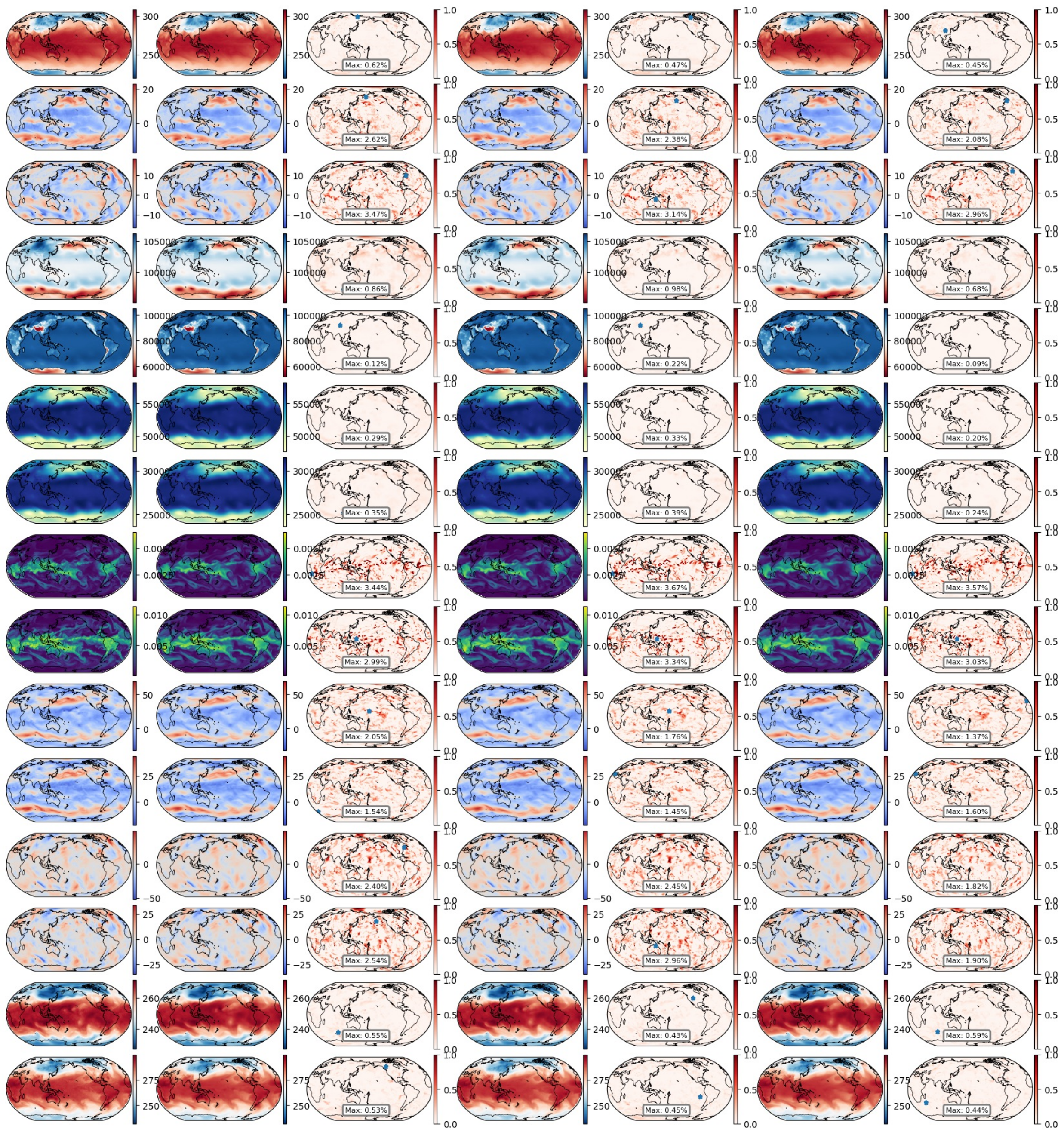}
        \put(5, 99.3){\scriptsize GT}
    \put(16.4, 99.3){\scriptsize Graphcast}
    \put(27.8, 99.3){\scriptsize Error@Graphcast}
    \put(42.4, 99.3){\scriptsize Oneforecast}
    \put(53.3, 99.3){\scriptsize Error@Oneforecast}
    \put(70.5, 99.3){\scriptsize Ours}
    \put(81.7, 99.3){\scriptsize Error@Ours}

    \put(-0.5, 68.4){\scriptsize \rotatebox{90}{ SP }}
    \put(-0.5, 74.5){\scriptsize \rotatebox{90}{ MSL }}
    \put(-0.5, 81.3){\scriptsize \rotatebox{90}{ V10 }}
    \put(-0.5, 88.1){\scriptsize \rotatebox{90}{ U10 }}
    \put(-0.5, 94.9){\scriptsize \rotatebox{90}{ T2M }}

    \put(-0.5, 34.8){\scriptsize \rotatebox{90}{ U500 }}
    \put(-0.5, 41.2){\scriptsize \rotatebox{90}{ Q700 }}
    \put(-0.5, 47.9){\scriptsize \rotatebox{90}{ Q500 }}
    \put(-0.5, 54.5){\scriptsize \rotatebox{90}{ Z700 }}
    \put(-0.5, 61){\scriptsize \rotatebox{90}{ Z500 }}
    
    \put(-0.5, 2){\scriptsize \rotatebox{90}{ T850 }}
    \put(-0.5, 9){\scriptsize \rotatebox{90}{ T500 }}
    \put(-0.5, 15.7){\scriptsize \rotatebox{90}{ V700 }}
    \put(-0.5, 22.3){\scriptsize \rotatebox{90}{ V500 }}
    \put(-0.5, 28.9){\scriptsize \rotatebox{90}{ U700 }}         
        \end{overpic}     \caption{2-day forecast results of global weather among different models.} 
   \label{fig:global7}
   \vspace{-0.2cm}
\end{figure*}

% \begin{figure*}[t]
%   \centering
%    \begin{overpic}[width=\linewidth]{visualizes/fig_step9.pdf}
%         \put(5.5,44.2){\scriptsize GT}     \put(17,44.2){\scriptsize Graphcast}     \put(29,44.2){\scriptsize Error@Graphcast}     \put(45,44.2){\scriptsize Oneforecast}     \put(57,44.2){\scriptsize Error@Oneforecast}     \put(75.5,44.2){\scriptsize Ours}     \put(87.5,44.2){\scriptsize Error@Ours}          \put(-0.5, 3){\scriptsize \rotatebox{90}{ SP }}     \put(-0.5, 12){\scriptsize \rotatebox{90}{ MSL }}     \put(-0.5, 20){\scriptsize \rotatebox{90}{ V10 }}     \put(-0.5, 29){\scriptsize \rotatebox{90}{ U10 }}     \put(-0.5, 38){\scriptsize \rotatebox{90}{ T2M }}          \end{overpic}     \caption{2.5-day forecast results of global weather among different models.} 
%    \label{fig:global9}
%    \vspace{-0.2cm}
% \end{figure*}

\begin{figure*}[t]
  \centering
   \begin{overpic}[width=\linewidth]{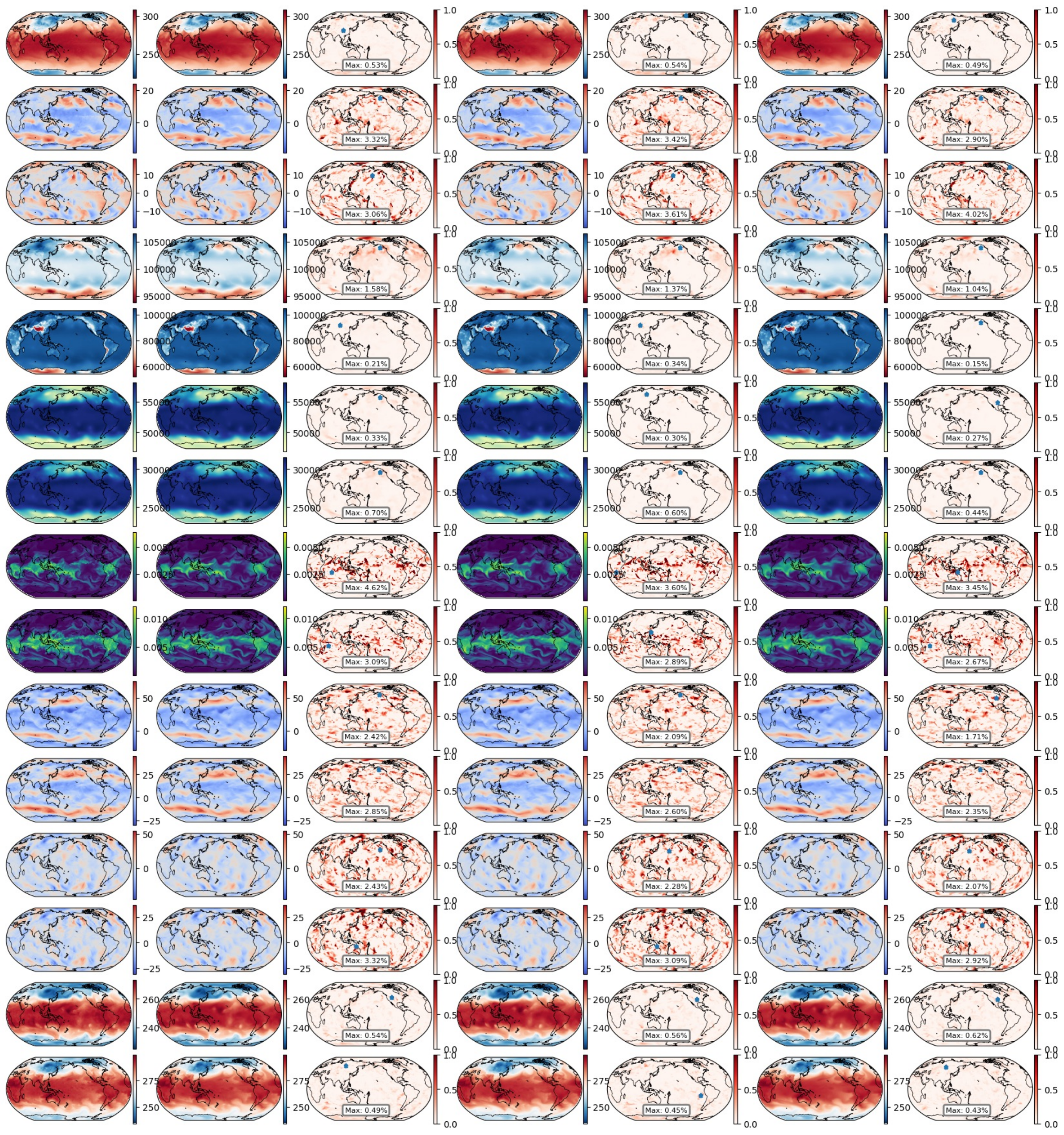}
        \put(5, 99.3){\scriptsize GT}
    \put(16.4, 99.3){\scriptsize Graphcast}
    \put(27.8, 99.3){\scriptsize Error@Graphcast}
    \put(42.4, 99.3){\scriptsize Oneforecast}
    \put(53.3, 99.3){\scriptsize Error@Oneforecast}
    \put(70.5, 99.3){\scriptsize Ours}
    \put(81.7, 99.3){\scriptsize Error@Ours}

    \put(-0.5, 68.4){\scriptsize \rotatebox{90}{ SP }}
    \put(-0.5, 74.5){\scriptsize \rotatebox{90}{ MSL }}
    \put(-0.5, 81.3){\scriptsize \rotatebox{90}{ V10 }}
    \put(-0.5, 88.1){\scriptsize \rotatebox{90}{ U10 }}
    \put(-0.5, 94.9){\scriptsize \rotatebox{90}{ T2M }}

    \put(-0.5, 34.8){\scriptsize \rotatebox{90}{ U500 }}
    \put(-0.5, 41.2){\scriptsize \rotatebox{90}{ Q700 }}
    \put(-0.5, 47.9){\scriptsize \rotatebox{90}{ Q500 }}
    \put(-0.5, 54.5){\scriptsize \rotatebox{90}{ Z700 }}
    \put(-0.5, 61){\scriptsize \rotatebox{90}{ Z500 }}
    
    \put(-0.5, 2){\scriptsize \rotatebox{90}{ T850 }}
    \put(-0.5, 9){\scriptsize \rotatebox{90}{ T500 }}
    \put(-0.5, 15.7){\scriptsize \rotatebox{90}{ V700 }}
    \put(-0.5, 22.3){\scriptsize \rotatebox{90}{ V500 }}
    \put(-0.5, 28.9){\scriptsize \rotatebox{90}{ U700 }}        
        \end{overpic}     \caption{3-day forecast results of global weather among different models.} 
   \label{fig:global11}
   \vspace{-0.2cm}
\end{figure*}

\begin{figure*}[t]
  \centering
   \begin{overpic}[width=\linewidth]{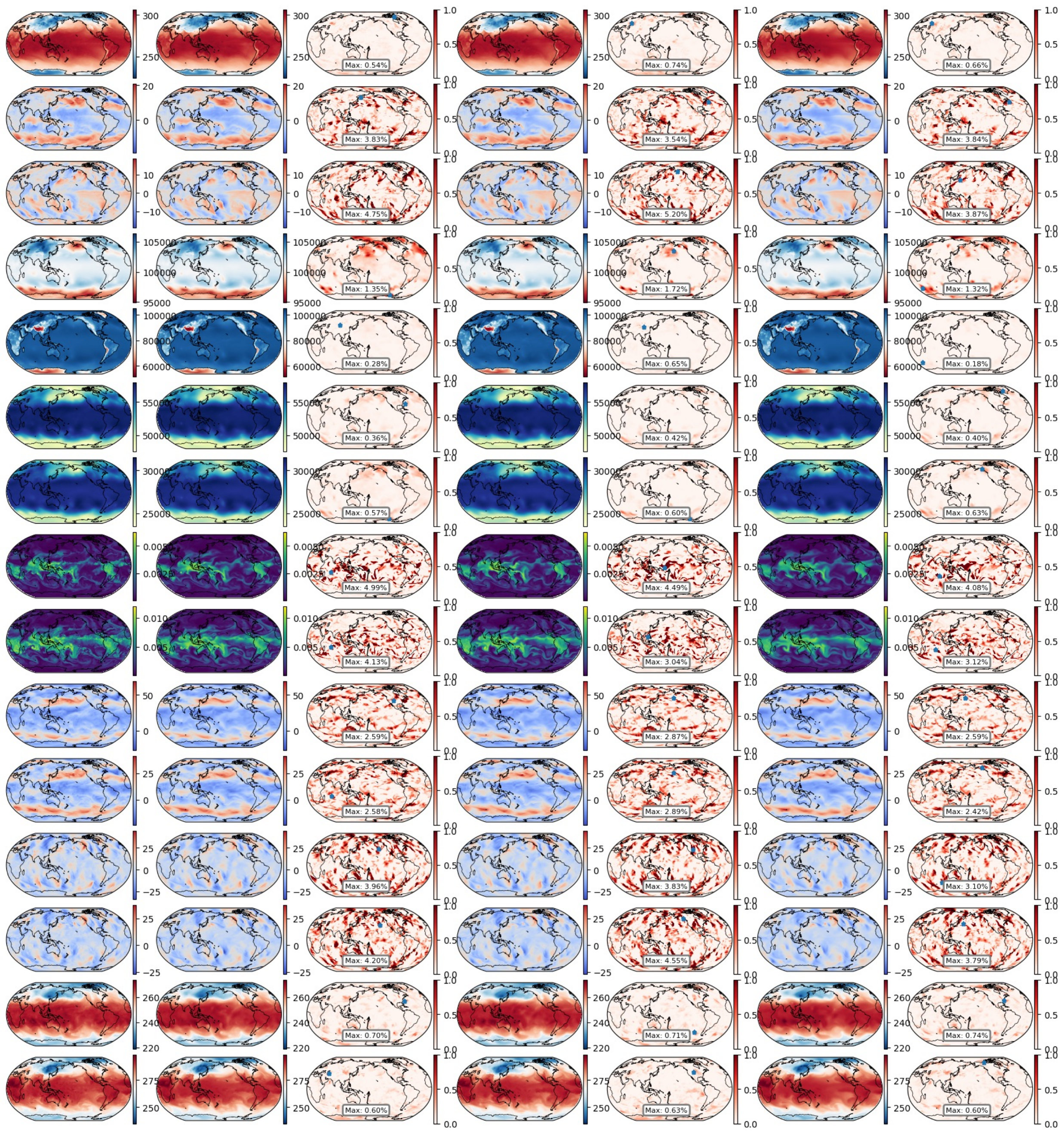}
        \put(5, 99.3){\scriptsize GT}
    \put(16.4, 99.3){\scriptsize Graphcast}
    \put(27.8, 99.3){\scriptsize Error@Graphcast}
    \put(42.4, 99.3){\scriptsize Oneforecast}
    \put(53.3, 99.3){\scriptsize Error@Oneforecast}
    \put(70.5, 99.3){\scriptsize Ours}
    \put(81.7, 99.3){\scriptsize Error@Ours}

    \put(-0.5, 68.4){\scriptsize \rotatebox{90}{ SP }}
    \put(-0.5, 74.5){\scriptsize \rotatebox{90}{ MSL }}
    \put(-0.5, 81.3){\scriptsize \rotatebox{90}{ V10 }}
    \put(-0.5, 88.1){\scriptsize \rotatebox{90}{ U10 }}
    \put(-0.5, 94.9){\scriptsize \rotatebox{90}{ T2M }}

    \put(-0.5, 34.8){\scriptsize \rotatebox{90}{ U500 }}
    \put(-0.5, 41.2){\scriptsize \rotatebox{90}{ Q700 }}
    \put(-0.5, 47.9){\scriptsize \rotatebox{90}{ Q500 }}
    \put(-0.5, 54.5){\scriptsize \rotatebox{90}{ Z700 }}
    \put(-0.5, 61){\scriptsize \rotatebox{90}{ Z500 }}
    
    \put(-0.5, 2){\scriptsize \rotatebox{90}{ T850 }}
    \put(-0.5, 9){\scriptsize \rotatebox{90}{ T500 }}
    \put(-0.5, 15.7){\scriptsize \rotatebox{90}{ V700 }}
    \put(-0.5, 22.3){\scriptsize \rotatebox{90}{ V500 }}
    \put(-0.5, 28.9){\scriptsize \rotatebox{90}{ U700 }}         
        \end{overpic}     \caption{5-day forecast results of global weather among different models.} 
   \label{fig:global19}
   \vspace{-0.2cm}
\end{figure*}

\begin{figure*}[t]
  \centering
   \begin{overpic}[width=\linewidth]{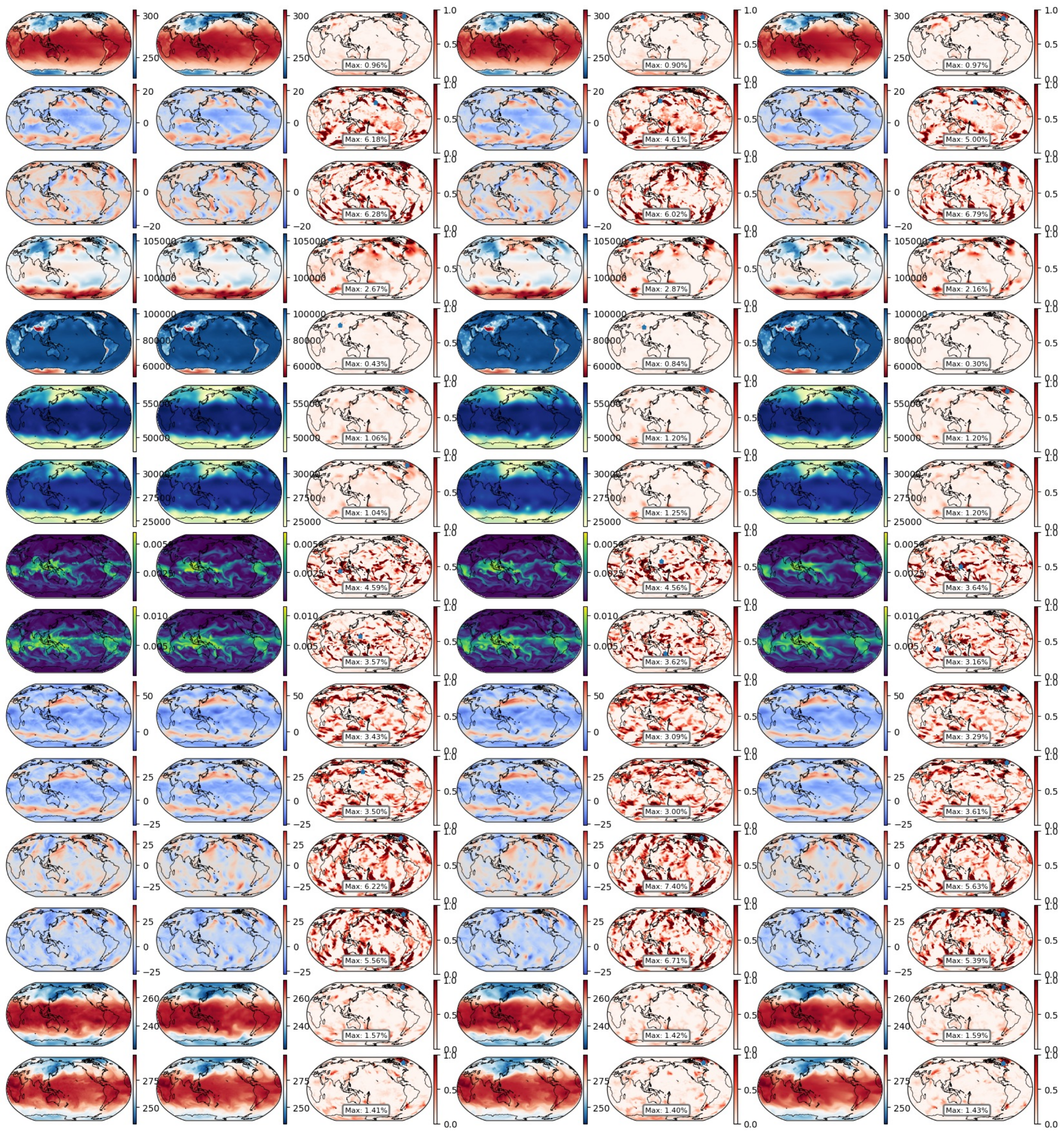}
        \put(5, 99.3){\scriptsize GT}
    \put(16.4, 99.3){\scriptsize Graphcast}
    \put(27.8, 99.3){\scriptsize Error@Graphcast}
    \put(42.4, 99.3){\scriptsize Oneforecast}
    \put(53.3, 99.3){\scriptsize Error@Oneforecast}
    \put(70.5, 99.3){\scriptsize Ours}
    \put(81.7, 99.3){\scriptsize Error@Ours}

    \put(-0.5, 68.4){\scriptsize \rotatebox{90}{ SP }}
    \put(-0.5, 74.5){\scriptsize \rotatebox{90}{ MSL }}
    \put(-0.5, 81.3){\scriptsize \rotatebox{90}{ V10 }}
    \put(-0.5, 88.1){\scriptsize \rotatebox{90}{ U10 }}
    \put(-0.5, 94.9){\scriptsize \rotatebox{90}{ T2M }}

    \put(-0.5, 34.8){\scriptsize \rotatebox{90}{ U500 }}
    \put(-0.5, 41.2){\scriptsize \rotatebox{90}{ Q700 }}
    \put(-0.5, 47.9){\scriptsize \rotatebox{90}{ Q500 }}
    \put(-0.5, 54.5){\scriptsize \rotatebox{90}{ Z700 }}
    \put(-0.5, 61){\scriptsize \rotatebox{90}{ Z500 }}
    
    \put(-0.5, 2){\scriptsize \rotatebox{90}{ T850 }}
    \put(-0.5, 9){\scriptsize \rotatebox{90}{ T500 }}
    \put(-0.5, 15.7){\scriptsize \rotatebox{90}{ V700 }}
    \put(-0.5, 22.3){\scriptsize \rotatebox{90}{ V500 }}
    \put(-0.5, 28.9){\scriptsize \rotatebox{90}{ U700 }}       
        \end{overpic}     
        \caption{7-day forecast results of global weather among different models.} 
   \label{fig:global27}
   \vspace{-0.2cm}
\end{figure*}

\begin{figure*}[t]
  \centering
   \begin{overpic}[width=\linewidth]{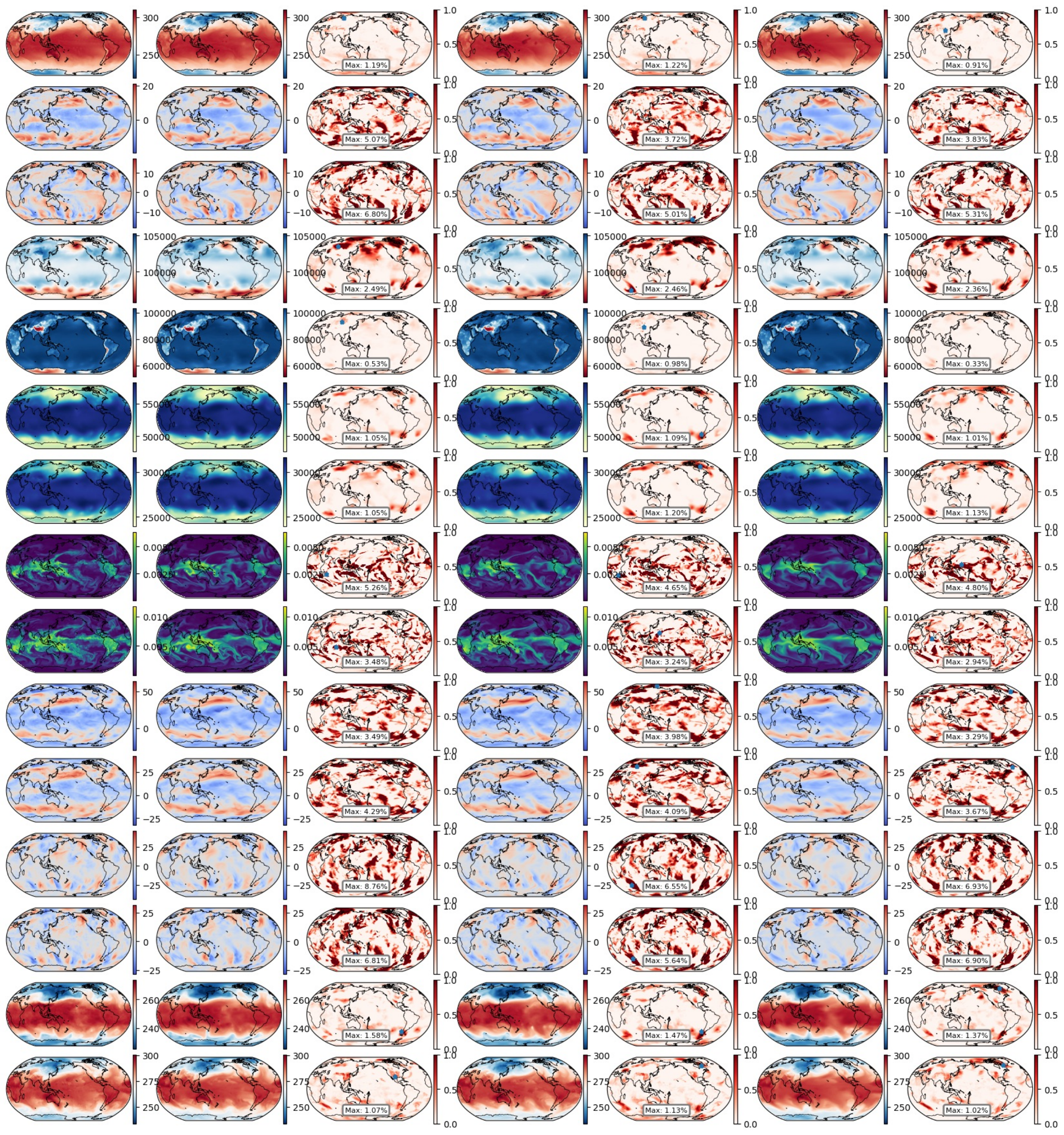}
        \put(5, 99.3){\scriptsize GT}
    \put(16.4, 99.3){\scriptsize Graphcast}
    \put(27.8, 99.3){\scriptsize Error@Graphcast}
    \put(42.4, 99.3){\scriptsize Oneforecast}
    \put(53.3, 99.3){\scriptsize Error@Oneforecast}
    \put(70.5, 99.3){\scriptsize Ours}
    \put(81.7, 99.3){\scriptsize Error@Ours}

    \put(-0.5, 68.4){\scriptsize \rotatebox{90}{ SP }}
    \put(-0.5, 74.5){\scriptsize \rotatebox{90}{ MSL }}
    \put(-0.5, 81.3){\scriptsize \rotatebox{90}{ V10 }}
    \put(-0.5, 88.1){\scriptsize \rotatebox{90}{ U10 }}
    \put(-0.5, 94.9){\scriptsize \rotatebox{90}{ T2M }}

    \put(-0.5, 34.8){\scriptsize \rotatebox{90}{ U500 }}
    \put(-0.5, 41.2){\scriptsize \rotatebox{90}{ Q700 }}
    \put(-0.5, 47.9){\scriptsize \rotatebox{90}{ Q500 }}
    \put(-0.5, 54.5){\scriptsize \rotatebox{90}{ Z700 }}
    \put(-0.5, 61){\scriptsize \rotatebox{90}{ Z500 }}
    
    \put(-0.5, 2){\scriptsize \rotatebox{90}{ T850 }}
    \put(-0.5, 9){\scriptsize \rotatebox{90}{ T500 }}
    \put(-0.5, 15.7){\scriptsize \rotatebox{90}{ V700 }}
    \put(-0.5, 22.3){\scriptsize \rotatebox{90}{ V500 }}
    \put(-0.5, 28.9){\scriptsize \rotatebox{90}{ U700 }}         
    \end{overpic}     
    \caption{10-day forecast results of global weather among different models.} 
   \label{fig:global39}
   \vspace{-0.2cm}
\end{figure*}

\end{document}